%% file: main.tex
\documentclass[twoside, 11pt]{article}
\pdfoutput=1

\include{package_imports}

\title{Kernel-based Translations of Convolutional Networks}
\author{Corinne Jones, Vincent Roulet, Zaid Harchaoui \vspace{0.2cm}\\ University of Washington \\ \{cjones6, vroulet, zaid\}@uw.edu}

\include{shortcuts}

\begin{document}
\maketitle

\begin{abstract}
\input{sections/abstract}
\end{abstract}

\section{Introduction}
\input{sections/sec1}

\section{Related work}
\input{sections/sec2}

\section{Convolutional networks: kernel and neural formulations}
\label{sec:translation}
\input{sections/sec3_v2}

\section{Supervised training of CKNs}
\label{sec:training}
\input{sections/sec4}

\input{sections/optim_ckn}

\section{Experiments}
\input{sections/sec5}

\section{Conclusion}
\input{sections/sec6}

\input{sections/acknowledgements}
\bibliography{main}
\bibliographystyle{abbrvnat_modified}

\onecolumn
\section*{Appendix}
\appendix
\input{sections/appendix0-intro}
\input{sections/appendix2-convnet_math_details}

\input{sections/appendix3-lenet_incomplete_scheme}
\input{sections/appendix4-ckn_translations}

\input{sections/appendix5-gradient_proofs}

\input{sections/appendix7-stochastic_gradient_manifolds_constrained.tex}
\input{sections/appendix8-ultimate_layer_reversal.tex}

\input{sections/appendix10-additional_results_training}
\end{document}

%% file: package_imports.tex
\usepackage{times}
\usepackage{mdframed}
\usepackage[letterpaper, margin=1in]{geometry}
\usepackage{microtype}
\usepackage{graphicx}
\usepackage{booktabs} 

\usepackage[sort]{natbib}
\usepackage[utf8]{inputenc} %
\usepackage[T1]{fontenc}    %
\usepackage{url}            %
\usepackage{booktabs}       %
\usepackage{amsfonts}       %
\usepackage{nicefrac}       %
\usepackage{microtype}      %
\usepackage{algorithm}
\usepackage{algorithmic}
\usepackage{color}     
\usepackage{amsmath} 
\usepackage{mathtools}
\usepackage{bbm}
\usepackage{caption}
\usepackage{subcaption}
\usepackage{tabularx}
\usepackage{tablefootnote}
\usepackage{pdflscape}
\usepackage{afterpage}
\usepackage{capt-of}
\usepackage{standalone}
\usepackage{tikz}
\usepackage{color}
\usepackage{xcolor}
\usepackage{amsthm}
\usepackage{amssymb}
\usepackage{enumitem}
\usepackage{threeparttable}
\usepackage{thmtools}

\DeclareMathAlphabet{\mathpzc}{OT1}{pzc}{m}{it}

\definecolor{puorange}{rgb}{0.80,0.20,0}
\definecolor{bluegray}{rgb}{0.04,0,0.7}
\definecolor{greengray}{rgb}{0.05,0.50,0.15}
\definecolor{darkbrown}{rgb}{0.40,0.2,0.05}
\definecolor{darkcyan}{rgb}{0,0.4,1}
\definecolor{black}{rgb}{0,0,0}
\definecolor{grey}{rgb}{0.93,0.93,0.93}

%% file: shortcuts.tex
\newcommand*{\C}{C}

\newcommand*{\emat}{\mathcal{E}}
\newcommand*{\E}{E}

\newcommand*{\F}{F}
\newcommand*{\f}{f}

\newcommand*{\im}{F}
\newcommand*{\imk}{K}
\newcommand*{\impatch}{\boxplus}
\renewcommand*{\k}{k}

\newcommand*{\N}{N}
\newcommand*{\nbhd}{\mathcal{N}}
\newcommand*{\nfilt}{f}

\renewcommand*{\P}{P}
\newcommand*{\p}{p}
\newcommand*{\patch}{x}
\newcommand*{\patchk}{k}

\newcommand*{\s}{s}

\newcommand*{\V}{V}
\newcommand*{\W}{W}
\newcommand*{\w}{w}
\newcommand*{\Z}{Z}

\newcommand{\filter}{w}

\newcommand{\nbfilters}{f}

\newcommand{\gfun}{g}
\newcommand{\simpobj}{\hat f}
\newcommand{\sign}{\operatorname{sign}}

\let\vec\relax
\DeclareMathOperator*{\vec}{vec}

\DeclareMathOperator*{\argmin}{arg\,min}

\DeclareMathOperator*{\diag}{diag}

\newcommand*{\mbr}{\mathbb{R}}
\newcommand*{\mbs}{\mathbb{S}}

\newcommand*{\mch}{\mathcal{H}}
\newcommand*{\mcl}{\mathcal{L}}

\newcommand*{\mbone}{\mathbbm{1}}

\newcommand{\reals}{\mathbb{R}}
\newcommand{\grad}{\operatorname{grad}}
\newcommand{\Vect}{\operatorname{Vect}}
\newcommand{\Proj}{\operatorname{Proj}}
\newcommand{\Expect}{\mathbb{E}}
\newcommand{\stepsize}{\gamma}
\newcommand{\step}{\operatorname{ret}}

\newcommand{\ones}{\mathbf{1}}
\newcommand{\Tr}{\operatorname{\bf Tr}}
\newcommand{\id}{\operatorname{I}}

\renewcommand{\triangleq}{\coloneqq}

\renewcommand{\algorithmicrepeat}{\textbf{Do:}}

\newtheorem{example}{Example} 
\newtheorem{theorem}{Theorem}
\newtheorem{lemma}[theorem]{Lemma} 
\newtheorem{proposition}[theorem]{Proposition} 

\newtheorem{corollary}[theorem]{Corollary}

\declaretheoremstyle[
notefont=\bfseries, notebraces={}{},
bodyfont=\normalfont\itshape,
headformat=\NAME \NOTE
]{nopar}

\declaretheorem[style=nopar, name=Proposition]{proposition_unnumbered}

\allowdisplaybreaks

\renewcommand*{\b}{b}
\renewcommand*{\k}{k}
\renewcommand*{\L}{L}

%% file: sections/abstract.tex
Convolutional Neural Networks, as most artificial neural networks, are commonly viewed as methods different in essence from kernel-based methods. We provide a systematic translation of Convolutional Neural Networks (ConvNets) into their kernel-based counterparts, Convolutional Kernel Networks (CKNs), and demonstrate that this perception is unfounded both formally and empirically. We show that, given a Convolutional Neural Network, we can design a corresponding Convolutional Kernel Network, easily trainable using a new stochastic gradient algorithm based on an accurate gradient computation, that performs on par with its Convolutional Neural Network counterpart. 
We present experimental results supporting our claims on landmark ConvNet architectures comparing each ConvNet to its CKN counterpart over several parameter settings.

%% file: sections/sec1.tex
For many tasks, convolutional neural networks (ConvNets) are currently the most successful approach to learning a functional mapping from inputs to outputs. For example, this is true for image classification, where they can learn a mapping from an image to a visual object category.
The common description of a convolutional neural network decomposes the architecture into layers that implement particular parameterized functions \citep{goodfellow2016}. 

The first generation of ConvNets, which include the Neocognitron~\citep{fukushima1980} and the LeNet series~\citep{lecun1988,lecun1989,lecun1989b,lecun1995,lecun1998a} stack two main types of layers: convolutional layers and pooling layers. These two types of layers were motivated from the Hubel-Wiesel model of human visual perception \citep{hubel1962}. 
A convolutional layer decomposes into several units. Each unit is connected to local patches in the feature maps of the previous layer through a set of weights. A pooling layer computes a local statistic of a patch of units in one feature map.

 This operational description of ConvNets contrasts with the mathematical description of kernel-based methods. Kernel-based methods, such as support vector machines, were at one point the most popular array of approaches to learning functional mappings from input examples to output labels \citep{scholkopf2002,steinwart2008}. Kernels are positive-definite pairwise similarity measures that allow one to design and learn such mappings by defining them as linear functionals in a Hilbert space. Owing to the so-called reproducing property of kernels, these linear functionals can be learned from data. 
	
This apparent antagonism between the two families of approaches is, however, misleading and somewhat unproductive. We argue and demonstrate that, in fact, \emph{any convolutional neural network can potentially be translated into a convolutional kernel network}, a kernel-based method with an appropriate hierarchical compositional kernel. Indeed, the operational description of a ConvNet can be seen as the description of a data-dependent approximation of an appropriate kernel map.

The kernel viewpoint brings important insights. Despite the widespread use of ConvNets, relatively little is understood about them. We would, in general, like to be able to address questions such as the following: What kinds of activation functions should be used? How many filters should there be at each layer?  Why should we use spatial pooling? Through CKNs we can begin to understand the answers to these questions. Activation functions used in ConvNets are used to approximate a kernel, i.e., a similarity measure, between patches. The number of filters determines the quality of the approximation. Moreover, spatial pooling may be viewed as approximately taking into account the distance between patches when measuring the similarity between images.  
	
We lay out a systematic translation framework between ConvNets and CKNs and put it to practice with three landmark architectures on two problems. The three ConvNet architectures, %
LeNet-1, LeNet-5, and All-CNN-C, correspond to milestones in the development of ConvNets. 
We consider digit classification with LeNet-1 and LeNet-5 on MNIST \citep{lecun1995,lecun1998a} and image classification with All-CNN-C \citep{springenberg2014} on CIFAR-10 \citep{krizhevsky2009}. 
We present an efficient algorithm to train a convolutional kernel network end-to-end, based on a first stage of unsupervised training and a second stage of gradient-based supervised training. 

To our knowledge, this work presents the first systematic experimental comparison~\emph{on an equal standing} of the two approaches on real-world datasets. By equal standing we mean that the two architectures compared are analogous from a functional viewpoint and are trained similarly from an algorithmic viewpoint. 
This is also the first time kernel-based counterparts of convolutional nets are shown to perform on par with %
convolutional neural nets on several real datasets over a wide range of settings.
	
	In summary, we make the following contributions: 
	\begin{itemize}
\item Translating the LeNet-1, LeNet-5, and All-CNN-C ConvNet architectures into their Convolutional Kernel Net counterparts;
	\item Establishing a general gradient formula and algorithm to train a Convolutional Kernel Net in a supervised manner;
	\item Demonstrating that Convolutional Kernel Nets can achieve comparable %
	performance to their ConvNet counterparts.
	\end{itemize}
The CKN code for this project is publicly available in the software package \emph{yesweckn} at \url{https://github.com/cjones6/yesweckn}.

%% file: sections/sec2.tex
This paper builds upon two interwoven threads of research related to kernels: the connections between kernel-based methods and (convolutional) neural networks and the use of compositions of kernels %
for the design of feature representations.

The first thread dates back to~\citet{neal1996}, who showed that an infinite-dimensional single-layer neural network is equivalent to a Gaussian process. Building on this, \citet{williams1996} derived what the corresponding covariance functions were for two specific activation functions. Later,~\citet{cho2009} proposed what they termed the arc-cosine kernels, showing that they are equivalent to infinite-dimensional neural networks with specific activation functions (such as the ReLU) when the weights of the neural networks are independent and unit Gaussian. Moreover, they composed these kernels and obtained higher classification accuracy on a dataset like MNIST. However, all of these works dealt with neural networks, not convolutional neural networks. 

More recently, several works used kernel-based methods or approximations thereof as substitutes to fully-connected layers or other parts of convolutional neural networks to build hybrid ConvNet architectures~\citep{bruna2013,huang2014,dai2014,yang2015,oyallon2017,oyallon2018}. In contrast to these works, we are interested in purely kernel-based networks.

The second thread began with \citet[Section 13.3.1]{scholkopf2002}, who proposed a kernel over image patches for image classification. The kernel took into account the inner product of the pixels at 
every pair of pixel locations within a patch, in addition to the distance between the pixels. 
This served as a precursor to later work that composed kernels. A related idea, introduced by \citet{bouvrie2009}, entailed having a hierarchy of kernels. These kernels were defined to be normalized inner products of ``neural response functions'', which were derived by pooling the values of a kernel at the previous layer over a particular region of an image.
In a similar vein,~\citet{bo2011} first defined a kernel over sets of patches from two images based on the sum over all pairs of patches within the sets. Within the summation is a weighted product of a kernel over the patches and a kernel over the patch locations. They then proposed using this in a hierarchical manner, approximating the kernel at each layer by projecting onto a subspace.

Multi-layer convolutional kernels were introduced in~\citet{mairal2014}. In~\citet{mairal2014} and \citet{paulin2017}, kernel-based methods using such kernels were shown to achieve competitive performance for particular parameter settings on image classification and image retrieval tasks, respectively. The kernels considered were however different from the kernel counterparts of the ConvNets they competed with. 
Building upon this work, \citet{mairal2016} proposed an end-to-end training algorithm for convolutional kernel networks. Each of these aforementioned works relied on an approximation to the kernels based on either optimization or the Nystr\"{o}m method \citep{williams2000,bo2009}. Alternatively, kernel approximations using random Fourier features~\citep{rahimi2007} or variants thereof could also approximate such kernels for a variety of activation functions~\citep{daniely2016,daniely2017}, although at a slower rate. Finally,~\citet{bietti2019} studied the invariance properties of convolutional kernel networks from a theoretical function space point of view. 

Building off of~\citet{mairal2014,daniely2016,daniely2017,bietti2019} and \citet{paulin2017}, we put to practice the translation of a convolutional neural network into its convolutional kernel network counterpart for several landmark architectures. Translating a convolutional neural network to its convolutional kernel network counterpart requires a careful examination of the details of the architecture, going beyond broad strokes simplifications made in previous works.
When each translation is carefully performed, the resulting convolutional kernel network can compete with the convolutional neural network. We provide all the details of our translations in Appendix~\ref{app:ckn_desc}.
To effectively train convolutional kernel networks, we present a rigorous derivation of a general formula of the gradient of the objective with respect to the parameters. Finally, we propose a new stochastic gradient algorithm using an accurate gradient computation method to train convolutional kernel networks in a supervised manner. As a result, we demonstrate that convolutional kernel networks perform on par with their convolutional neural network counterparts.

%% file: sections/sec3_v2.tex
To refresh the reader on convolutional kernel networks (CKNs), we begin this section by providing a novel informal viewpoint. We then proceed to describe the correspondences between CKNs and ConvNets. Using these correspondences, we then demonstrate how to translate an example ConvNet architecture, LeNet-5~\citep{lecun1998a}, to a CKN.

\subsection{Convolutional kernel networks}
\label{sec:ckn_formulation}

Convolutional kernel networks provide a natural means of generating features for structured data such as signals and images.  Consider the aim of generating features for images such that ``similar'' images have ``similar'' features. %
CKNs approach this problem by developing a similarity measure between images.

\paragraph{Similarity measure on patches}
Let $K$ be a similarity measure between images and suppose $K$ can be written as an inner product 
\begin{equation*}
\imk(\im, \im') = \langle \varphi(\im), \varphi(\im')\rangle_{\mch_\imk}
\end{equation*}
in a space $\mch_\imk$ for some $\varphi$. Then we could take $\varphi(\im)$ to be the feature representation for image $\im$ and similarly for image $\im'$.
 The question therefore is how to choose $\imk$. A similarity measure $\imk$ that is applied pixel-wise will be ineffective. This is because there is no reason why we would expect two ``similar'' images to be similar pixel-wise. Therefore, we consider similarity measures applied to patches. Here a patch consists of the pixel values at a (typically contiguous) subset of pixel locations. Statistics on patches more closely align with how humans perceive images. Moreover, treating a patch as a unit takes into account the fact that images are often not locally rotationally invariant. 
 
Let $\impatch_i$ and ${\impatch_i}'$, $i=1,\dots, m$ be patches from images $\im$ and $\im'$, respectively, where for all $i$, $\impatch_i$ and ${\impatch_i}'$ are from the same positions. %
Then, given a similarity measure $\patchk$ on the patches, we could choose the similarity measure $\imk$ on the images to be given by
\begin{equation*}
\label{eq:image_kernel_patches}
\imk(\im, \im') = \sum_{i,j=1}^mw_{i,j}\patchk(\impatch^i, {\impatch^j}') 
\end{equation*}
for some weights $w_{ij}\geq0$.  The overall similarity measure $\imk$ accounts for the fact that images that are similar will not necessarily have patches in the same locations by comparing all pairs of patches between the two images. The weighting accounts for the fact that while similar patches may not occur in the same location, we would not expect them to be too far apart. 

Convolutional kernel networks build such similarity measures $\imk$ by using a \emph{positive definite kernel} as the similarity measure $\patchk$ between patches. A positive definite kernel implicitly maps the patches to an infinite-dimensional space (a reproducing kernel Hilbert space (RKHS)) $\mathcal{H}$ and computes their inner product in this space, i.e., $\patchk(\impatch, \impatch' ) = \langle \phi(\impatch), \phi(\impatch')\rangle_{\mathcal{H}}$, where $\phi$ is the mapping from patches to $\mathcal{H}$ induced by $\patchk$. %
As long as $\patchk$ is a kernel, $\imk$ is also a kernel and can therefore be written as $\imk(\im, \im') = \langle \varphi(\im), \varphi(\im')\rangle_{\mch_\imk}$, where $\mch_\imk$ is the RKHS associated with kernel $\imk$ and $\varphi$ is the mapping from patches to $\mch_K$ induced by $K$.

\paragraph{Network construction}
If two images are similar, we would expect their patches to be similar for several patch sizes. Multi-layer CKNs incorporate this via a hierarchy of kernels. A primary aim of using a hierarchical structure is to gain invariance in the feature representation of the images. Specifically, we seek to gain translation invariance and invariance to small deformations. 
We will now describe such a hierarchical structure.

Let $\phi_1$ be the canonical feature map of a kernel $\patchk_1$ defined on patches of a given size, i.e., 
\begin{equation*}
\patchk_1(\impatch, \impatch') =  \langle \phi_1(\impatch), \phi_1(\impatch')\rangle_{\mch_1},
\end{equation*}
where $\mch_1$ is the RKHS associated with kernel $\patchk_1$.
Then $\phi_1(\impatch)$ provides a feature representation of the patch $\impatch$. Applying $\phi_1$ to each patch of the given size in the image, we obtain a new representation of the image (See Figure~\ref{fig:exact_ckn}). %

If two images are similar, we would expect them to also be similar when comparing their representations obtained by applying $\phi_1$. Therefore, we may apply the same logic as before. Let $\patchk_2$ be a kernel on patches in the new space. Applying its canonical feature map $\phi_2$, we obtain another representation of the image. The features at each spatial location in this representation are derived from a larger portion of the original image than those in the previous representation (previous \emph{layer}). Specifically, denoting by $\impatch_\ell$ and $\impatch_{\ell}'$ image patches from image representations $\im_\ell$ and ${\im_\ell}'$ at layer $\ell$, we apply the canonical feature map $\phi_\ell$ of the kernel $\patchk_\ell$ given by
\begin{equation*}
\patchk_\ell(\impatch_\ell, \impatch_{\ell'}) =  \langle \phi_\ell(\impatch_\ell), \phi_\ell(\impatch_{\ell}')\rangle_{\mch_\ell}
\end{equation*}
to each patch in $\im_\ell$ and ${\im_\ell}'$. 

One way to increase the invariance is to include an averaging (i.e., \emph{pooling}) step after applying each feature map $\phi_\ell$. Specifically, denote by ${(\im_\ell)}_{jk}$ the feature representation of the image $\im$ at spatial location $(j,k)$ after applying feature map $\phi_\ell$. Letting $\nbhd_{jk}$ denote a spatial neighborhood of the point $(j, k)$, we compute 
\begin{equation*}
(\im_\ell)_{jk} \gets \sum_{(j',k')\in\nbhd_{jk}}  w_{j,j',k,k'} (\im_\ell)_{j'k'}
\end{equation*}
for all $j,k$ where $w_{j,j',k,k'}\geq0$ are pre-specified weights. For example, for average pooling,  $w_{j,j',k,k'} = 1/|\nbhd_{jk}|$.

After pooling we often subsample the locations in the image for computational purposes. Subsampling by an integer factor of $k$ entails retaining every $k$th feature representation in each row and then every $k$th feature representation in each column.
By subsampling after pooling we aim to remove redundant features. 
Building layers in the above manner by applying feature maps $\phi_\ell$, pooling, and subsampling, we obtain a convolutional kernel network.

\begin{figure}
\centering
\vspace*{-3cm}
\hspace*{-2.5cm}
\includegraphics[trim={3cm 6.7cm 5.5cm 1cm},clip,scale=0.6]{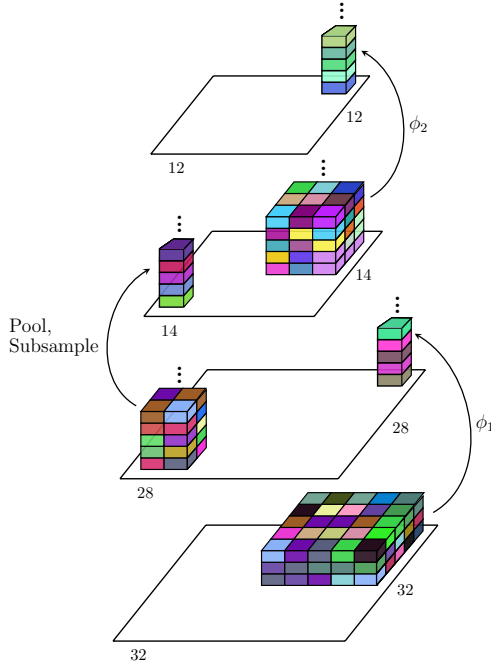}
\caption{\label{fig:exact_ckn} Example CKN architecture. The block sizes are the sizes of the filters at each layer.  The height of the blocks at the input layer is three to represent the input having three channels.  At every subsequent layer the feature representation is infinite dimensional. %
The arrows indicate how one block gets transformed into the block at the next layer. The numbers on the sides of the parallelograms indicate the spatial dimensions of the feature representations at each layer. }
\end{figure}

\begin{example}
Figure~\ref{fig:exact_ckn} depicts an example CKN. In the figure an initial RGB image of size $32\times32$ (represented by the bottom rhombus in the figure) gets transformed by applying feature map $\phi_1$ to patches of size $5\times5$. As $\phi_1$ is applied with stride $1\times1$ (i.e., it is applied to every possible contiguous $5\times5$ patch in the image), this results in a new feature representation (second rhombus) with spatial dimensions $28\times28$. Atop each spatial location sits an infinite-dimensional vector. 

At the second layer $2\times2$ pooling is applied to the infinite-dimensional vectors, followed by subsampling by a factor of 2. The pooling is performed on all contiguous $2\times2$ patches, which initially decreases the spatial dimensions to $27\times27$. Subsampling by a factor of two entails removing the features on top of every other spatial location, yielding an output with spatial dimensions $14\times14$. The output of this layer results in the next representation (third rhombus). 

Finally, the figure depicts the application of another feature map $\phi_2$ to patches of size $3\times3$ to obtain the feature representation at the final layer. As the stride is $1\times1$, this results in an output with spatial dimensions $12\times12$. 

Given such a network, we may then compute the similarity of two images by concatenating each image's feature representation at the final layer and then applying a linear kernel. 
While there are only two feature maps $\phi_\ell$, $\ell=1,2$, depicted in this figure, the process could continue for many more layers. 
\end{example}

\paragraph{Approximating the kernels}
While computing the overall kernel exactly is theoretically possible (assuming that the kernels at each layer only depend on inner products of features at previous layers), it is computationally unwieldy for even moderately-sized networks. 
To overcome the computational difficulties, we approximate the kernel $\patchk_\ell$ at each layer $\ell$ by finding a map $\psi_\ell$ such that 
$\patchk_\ell(\impatch_\ell,\impatch_\ell') \approx \langle \psi_\ell(\impatch_\ell), \psi_\ell(\impatch_\ell')\rangle_{\mbr^{\nfilt_\ell}}$
for some positive integer $\nfilt_\ell$. The $\psi_\ell$'s then replace the $\phi_\ell$'s at each layer, thereby providing feature representations of size $\nfilt_\ell$  of the patches at each layer $\ell$. 
There are many ways to choose $\psi_\ell$, including directly optimizing an approximation to the kernel, using random features, and projecting onto a
 subspace. %
 
We consider here the approximation resulting from the projection onto a subspace spanned by ``filters’’, usually referred to as the Nyström method \citep{williams2000,bo2009,mairal2016}. These filters may be initialized at random by sampling patches from the images. We shall show in Sections \ref{sec:grad_formula}-\ref{sec:intertwined_newton} how to differentiate through this approximation and learn the filters from data in a supervised manner.

 Consider a dot product kernel $\patchk$ with corresponding RKHS $\mch$ and canonical feature map $\phi$. Furthermore, let $\filter_1, \ldots, \filter_{\nbfilters} \in \reals^\s$ be a set of \emph{filters}. %
 Given a patch  $\impatch\in\mbr^\s$ of size $s$, the Nystr\"om approximation projects $\phi(\impatch)$  onto the subspace spanned by $\phi(\filter_1),\dots, \phi(\filter_\nbfilters)$ in $\mathcal{H}_\patchk$ by solving the 
kernel least squares problem
\begin{align*}
 \alpha^\star \in \argmin_{\alpha\in\mbr^\nbfilters} \left\Vert \phi(\impatch) - \sum_{i=1}^{\nbfilters}\alpha_{i} \phi(\filter_i) \right\Vert_{\mch}^2.
\end{align*}
Defining $\W=[\filter_1, \dots, \filter_\f]^T\in\mbr^{\f\times\s}$ and assuming that $\k$ is a dot product kernel, this results in the coefficients $\alpha^\star = \k(\W\W^T)^{-1}\k(\W\impatch)$, where the kernel $\k$ is understood to be applied element-wise.\footnote{A dot product kernel is a kernel of the form $k(x,y)=f(\langle x, y\rangle)$ for a function $f:\mbr\times\mbr\to\mbr$. For notational convenience for a dot product kernel $k$ we will write $k(t)$ rather than $k(x,y)$ where $t=\langle x, y\rangle$. For a matrix $A$ the element-wise application of $k$ to $A\in\mbr^{m\times n}$ results in $k(A)\coloneqq [k(A_{i,j})]_{i,j=1}^{m,n}$.}
Therefore, for two patches $\impatch$ and $\impatch'$ with corresponding optimal coefficients $\alpha$ and $\alpha'$, we have
\begin{align*}
\left\langle \phi(\impatch), \phi(\impatch')\right\rangle_\mch 
&\approx \left\langle \sum_{i=1}^{\nbfilters}\alpha_{i}^\star \phi(\filter_i), \sum_{i=1}^{\nbfilters}\alpha_{i}^{\star'} \phi(\filter_i)\right\rangle_\mch\\
&= \k(\W\impatch)^T\k(WW^T)^{-1}\k(\W\impatch')\\
&= \langle \patchk(\W\W^T)^{-1/2}\patchk(\W\impatch), \patchk(\W\W^T)^{-1/2}\patchk(\W\impatch')\rangle_{\mbr^\nfilt}.
\end{align*}
Hence, a finite-dimensional approximate feature representation of $\impatch$ is given by
\begin{align*}
\psi(\impatch) = (\patchk(\W\W^T))^{-1/2}\patchk(\W\impatch) \in \reals^\nbfilters.
\end{align*}
We will add a regularization term involving a small value $\epsilon>0$, as $\patchk(\W\W^T)$ may be poorly conditioned. 

Denote the input features to layer $\ell$ by $\F_{\ell-1}$, where the rows index the features and the columns index the spatial locations (which are flattened into one dimension). Let $\E_\ell$ be a function that extracts patches from $\F_{\ell-1}$.
We then write the features output by the Nyström method as 
\begin{align*}
 \F_{\ell}=\Psi_\ell(\F_{\ell-1}, \W_\ell) 
 =  (\patchk(\W\W^T)+\epsilon I)^{-1/2}\patchk(\W\E_\ell(\F_{\ell-1})).
 \end{align*}
 Here $\Psi_\ell$ denotes the function that applies the approximate feature map $\psi_\ell$ as derived above to the features at each spatial location.
It is important to note that the number of filters $\nfilt_\ell$ controls the quality of the approximation at layer $\ell$. Moreover, such a procedure results in the term $\W_\ell\E_\ell(\F_{\ell-1})$, in which the filters are \emph{convolved} with the images. %
This convolution is followed by a \emph{non-linearity} computed using the kernel $k$, resulting in the application of $\Psi_\ell$. 

\paragraph{Overall formulation}
The core hyperparameter in CKNs is the choice of kernel. For simplicity of the exposition we assume that the same kernel is used at each layer. Traditionally, CKNs use normalized kernels of the form
\begin{align*}
\patchk(\impatch, \impatch') = \|\impatch\|_2 \|\impatch'\|_2 \tilde \patchk(\impatch^\top \impatch'/ \|\impatch \|_2 \|\impatch'\|_2)
\end{align*}
where $\tilde k$ is a dot product kernel on the sphere. Examples of such kernels $\tilde k$ include the arc-cosine kernel of order 0 and the RBF kernel on the sphere. Here we allow for this formulation. Using dot product kernels on the sphere allows us to restrict the filters to lying on the sphere. Doing so adds a projection step in the optimization.

Let $\im_0$ be an input image. Denote by $\W_{\ell}$ the filters at layer $\ell$, $\E_\ell$ the function that extracts patches from $\F_\ell$ at layer $\ell$, and $\N_\ell$ the function normalizing the patches of $F_\ell$ at layer $\ell$. Furthermore, let $\P_\ell$ be the pooling and subsampling operator, represented by a matrix. (See Appendix~\ref{app:gradient_proofs} for precise definitions.) Then the representation at the next layer given by extracting patches, normalizing them, projecting onto a subspace, re-multiplying by the norms of the patches, pooling, and subsampling is given by
\begin{equation*}
F_{\ell} = \Psi_{\ell}(\im_{\ell-1}, \W_{\ell})\triangleq 
(\patchk(\W_\ell\W_\ell^T)+\epsilon \id)^{-1/2}\patchk(\W_\ell\E_\ell(\im_{\ell-1})\N_\ell(\im_\ell)^{-1})\N_\ell(\im_{\ell-1})\P_\ell.
\end{equation*}
After $\L$ such compositions we obtain a final representation of the image that can be used for a classification task.

Precisely, given a set of images $\im^{(1)},\dots, \im^{(n)}$ with corresponding labels $y^{(1)},\dots, y^{(n)}$, we
consider a linear classifier parametrized by $W_{L+1}$ and a loss $\mathcal{L}$, leading to the optimization problem 
\begin{align}\label{eq:training_pb}
\min_{\W_1,\dots,\W_{L+1}}\quad &  \frac{1}{n} \sum_{i=1}^n \mathcal{L}\left(y^{(i)}, \left\langle \W_{L+1}, \im_L^{(i)}\right\rangle\right)+\lambda \Vert \W_{L+1}\Vert_F^2 \\
\mbox{subject to} \quad & W_\ell \in S^{d_\ell} \quad \mbox{for $\ell=1,\ldots, L$}.\nonumber
\end{align}
Here $\lambda \geq 0$ is a regularization parameter for the classifier and $S^{d_\ell}$ is the product of Euclidean spheres at the layer $\ell$; see Appendix~\ref{app:gradient_proofs}.

\subsection{Connections to ConvNets}
\label{sec:translations}
\begin{table*}
	\caption{\label{tab:correspondence} Correspondences between ConvNets and CKNs
}
	\centering
	\begin{tabular}{ll}
		ConvNet component & CKN component \\ \hline
		Convolutional layer & Projection onto the same subspace for all patch locations \\
		Partially connected layer & Projection onto a different subspace for each region \\
		Fully connected layer &  Projection onto a subspace for the entire image representation\\
		Convolution + no nonlinearity & Applying feature map of linear kernel \\
		Convolution + tanh nonlinearity & Applying feature map of arc-cosine kernel of order 0$^*$ \\
		Convolution + ReLU nonlinearity & Applying feature map of arc-cosine kernel of order 1  \\
		Average pooling & Averaging of feature maps \\
		Local response normalization & Dividing patches by their norm$^*$ \\
		\bottomrule
	\end{tabular}
	\begin{tablenotes}
		\footnotesize
		\item $^*$ denotes an inexact correspondence.
	\end{tablenotes}
\end{table*}

CKNs may be viewed as infinite-dimensional analogues of ConvNets.  Table~\ref{tab:correspondence}
lists a set of transformations between ConvNets and CKNs. These are discussed below in more detail. For the remainder of this section we let $G\in\mbr^{\f\times\s_{1}\times\s_{2}}$ denote the feature representation of an image in a ConvNet. For clarity of exposition we represent it as a 3D tensor rather than a 2D matrix as for CKNs above. Here the first dimension indexes the features while the second and third dimensions index the spatial location. We denote the element of $G$ in feature map $z$ at spatial location $(x,y)$ by $(G)_{z,x,y}$.

\paragraph{Convolution and activation function}
The main component of ConvNets is the convolution of patches with filters, followed by a pointwise nonlinearity. More precisely, denote the filters by $W\in\mbr^{\f\times \s}$, a patch from $G$ by $\impatch\in\mbr^s$, and a nonlinearity by $a: \reals^\s \times \reals^\s \rightarrow \reals$. A ConvNet computes $a(W\impatch)$ for every patch $\impatch$ in $G$ where $a$ is understood to be applied element-wise.
This can be seen as an approximation of a kernel, as stated in the following proposition (see \cite{daniely2016} for more details). 
\begin{proposition}
	Consider a measurable space $\Omega$ with probability measure $\mu$ and an activation function $a: \reals^\s \times \reals^\s \rightarrow \reals$ such that $a(\cdot, \impatch)$ is  square integrable with respect to $\mu$ for any $\impatch \in \reals^\s$. Then  the pair $(\mu, a)$
	defines a kernel on patches $\impatch, \impatch' $ as the dot product of the functions $a(\cdot, \impatch)$ and $a(\cdot, \impatch')$ on the measurable space $\Omega$ with probability measure $\mu$, i.e., %
	\begin{align*}
		\patchk(\impatch,\impatch') \coloneqq \Expect_{\filter\sim\mu}[a(\filter,\impatch)a(\filter,\impatch')].
	\end{align*}
\end{proposition}

Hence, the convolution and pointwise nonlinearity in ConvNets with random weights approximate a kernel on patches. This approximation converges to the true value of the kernel as the number of filters $\f$ goes to infinity. The downside to using such a random feature approximation is that it produces less concise approximations of kernels than e.g., the Nystr\"om method. In order to assess whether trained CKNs perform similarly regardless of the approximation, we approximate CKNs using the Nystr\"om method.

Several results have been proven relating specific activation functions to their corresponding kernels.
For example, the ReLU corresponds to the arc-cosine kernel of order 1 \citep{cho2009} and the identity map corresponds to the linear kernel. The tanh nonlinearity may be approximated by a step function, and a step function corresponds to the arc-cosine kernel of order 0 \citep{cho2009}. 

\paragraph{Layer type}
ConvNets may have several types of layers, including convolutional, partially-connected, and fully connected layers. Each layer is parameterized by filters. Convolutional layers define patches and apply the same set of filters to each patch. On the other hand, partially-connected layers in ConvNets define patches and apply filters that differ across image regions to the patches.  Finally, fully connected layers in ConvNets are equivalent to convolutional layers where the size of the patch is the size of the image.

As in ConvNets, CKNs may have convolutional, partially-connected, and fully connected layers. Recall from Section~\ref{sec:ckn_formulation} that
CKNs project onto a subspace at each layer and that the subspace is defined by a set of filters. At convolutional layers in CKNs the projection is performed onto the same subspace for every patch location. On the other hand, for partially connected layers for CKNs we project onto a different subspace for each image region. Finally, for fully connected layers CKNs project onto a subspace defined by filters that are the size of the feature representation of an entire image.  

\paragraph{Pooling}
Pooling in ConvNets can take many forms, including average pooling. In each case, one defines spatial neighborhoods within the dimensions of the current feature representation (e.g., all $2\times2$ blocks). Within each neighborhood a local statistic is computed from the points within each feature map.\footnote{In the ConvNet literature, in contrast to the kernel literature, a feature map is defined as a slice of the feature representation along the depth dimension. That is, for a given $z$, $(G)_{z, \cdot, \cdot}$ is a feature map.} Concretely, for a spatial neighborhood $\nbhd\subset\{1,\dots,\s_{1}\}\times\{1,\dots,\s_{2}\}$ centered at the point $(x,y)$, average pooling computes
\begin{align}
(\tilde G)_{z, x, y} = \frac{1}{|\nbhd|}\sum_{(x',y') \in \nbhd} (G)_{z,x',y'}
\end{align}
for all $z=1,\dots, \f$.

 Average pooling in ConvNets corresponds to an averaging of the feature maps in CKNs. In addition, any weighted averaging, where the weights are the same across layers, corresponds to a weighted averaging of the feature maps. Specifically, note that in \cite{mairal2014,mairal2016} and \cite{paulin2017} the authors proposed Gaussian pooling for CKNs. In this formulation the weight of a feature map at location $\patch'$ when averaging about a feature map at location $\patch$ is given by $\exp({-\|\patch'-\patch\|_2^2}/(2\sigma^2))$. Here $\sigma$ is a hyperparameter.

\paragraph{Normalization}
There are a wide range of normalizations that have been proposed in the ConvNet literature. %
Normalizations of ConvNets modify the representation at each location $(z,x,y)$ of $G$ by taking into account values in a neighborhood $\nbhd$ of $(z,x,y)$.
One such normalization is local response normalization. 
Local response normalization computes
\begin{equation*}
(\tilde G)_{z,x,y} = \frac{(G)_{z,x,y}}{\left(\alpha + \beta\sum_{(z',x',y')\in\nbhd} (G)_{z',x',y'}^2\right)^{\gamma}},
\end{equation*}
for all $z,x,y$ where $\alpha, \beta$, and $\gamma$ are parameters that can be learned.
In local response normalization the neighborhood $\nbhd$ is typically defined to be at a given spatial location $(x,y)$ across some or all of the feature maps. However, the spatial scale of the neighborhood could be expanded to be defined across multiple locations within feature maps.

In CKNs there is not a meaningful counterpart to defining a neighborhood across only a subset of the feature maps. Therefore, we present only a counterpart to using all feature maps at once. Consider local response normalization in ConvNets when taking the neighborhood to be the locations across all feature maps within a given spatial area. This roughly corresponds to dividing by a power of the norm of a patch in CKNs when $\alpha=0$ and $\beta=1$. %

\subsection{Example translation}
We illustrate the translation from ConvNets to CKNs on the LeNet-5 architecture. LeNet-5 was one of the first modern ConvNets. However, the order of its modules differs from that of many recent ConvNets, as the nonlinearities follow the pooling. See Appendix~\ref{app:convnet_desc} for the details of the LeNet-5 ConvNet and Figure~\ref{fig:lenet5_arch} for a depiction of the translated CKN architecture. For clarity of exposition we represent the features at each layer of the CKN as a 3D tensor rather than a 2D matrix as in Section~\ref{sec:ckn_formulation}. In performing the translation, we use the approximate correspondence between the tanh activation and the arc-cosine kernel of order 0.  

\paragraph{Layer 1}
Let $\F_0\in\mbr^{1\times32\times32}$ denote the initial representation of an image. The first layer is the counterpart to a convolutional layer and consists of applying a linear kernel and projecting onto a subspace. Let $\W_1\in\mbr^{6\times1\times5\times5}$. %
For $k=1,\dots, 6$, let
\begin{equation*}
\F'_1(k, \cdot, \cdot) = F_0(1, \cdot, \cdot)\star \W_1(k, 1, \cdot, \cdot)
\end{equation*}
where $\star$ denotes the convolution operation.
Then the output of the first layer is given by $\F_1\in\mbr^{6\times28\times28}$ with 
\begin{equation*}
\F_1(\cdot, i, j) = \left(\sum_{m,n=1}^5 \W_1(\cdot, 1, m, n)\W_1(\cdot, 1, m, n)^T\right)^{-1/2}\F'_1(\cdot, i, j)
\end{equation*}
for $i,j=1,\dots, 28$.

\paragraph{Layer 2}
 Next, the second layer in the ConvNet performs average pooling and subsampling with learnable weights and then applies a pointwise nonlinearity (tanh). The corresponding CKN pools and subsamples and then applies an arc-cosine kernel on $1\times1$ patches. Define $\emat_{2}=(e_1, e_3, e_5, \dots, e_{27})$ where $e_i\in\mbr^{27}$ is a vector with 1 in element $i$ and 0 elsewhere.  The pooling and subsampling result in $\F'_2\in\mbr^{6\times14\times14}$ given by
\begin{equation*}
\F'_2(k, \cdot, \cdot) = \emat_{2}^T\left(\F_1(k, \cdot, \cdot)\star\frac{1}{4}\mathbbm{1}_{2\times2}\right)\emat_{2}
\end{equation*}
for $k=1,\dots, 6$.
Next, let $W_2\in\mbr^{6\times6}$ be the identity matrix and let $\k_2:\mbr^6\times\mbr^6\to\mbr$ be the arc-cosine kernel of order zero. 
The output of the second layer is then $\F_2\in\mbr^{6\times14\times14}$ given by
\begin{equation*}
\F_2(\cdot, i, j) = \left(\left[\k_2\left(\W_2(m, \cdot), \W_2(n, \cdot)\right)\right]_{m,n=1}^{6}\right)^{-1/2}\left[\k_2\left(\F'_2(\cdot, i, j), W_2(m, \cdot)\right)\right]_{m=1}^6
\end{equation*}
for $i,j=1,\dots, 14$.

\paragraph{Layer 3}
The third layer in LeNet-5 is again a convolutional layer. Here we use a complete connection scheme since for the ConvNet we found that empirically a complete connection scheme outperforms an incomplete connection scheme (see Appendix~\ref{app:lenet_incomplete}). Therefore, this layer again consists of applying a linear kernel and projecting onto a subspace. 
For $k=1,\dots, 16$, let
\begin{equation*}
\F'_3(k, \cdot, \cdot) = \sum_{z=1}^{6}  \F_2(z, \cdot, \cdot)\star \W_3(k, z, \cdot, \cdot).
\end{equation*}
Then the output of the third layer is given by $\F_3\in\mbr^{16\times10\times10}$ with 
\begin{equation*}
\F_3(\cdot, i, j) = \left(\sum_{m=1}^6\sum_{n,p=1}^5 \W_3(\cdot, m, n, p)\W_3(\cdot, m, n, p)^T\right)^{-1/2}\F'_3(\cdot, i, j).
\end{equation*}

\paragraph{Layer 4}
The fourth layer is similar to the second layer. The CKN pools and subsamples and then applies an arc-cosine kernel on $1\times1$ patches. Define $\emat_{4}=(e_1, e_3, e_5,  e_{7})$ where $e_i\in\mbr^{7}$ is a vector with 1 in element $i$ and 0 elsewhere.  The pooling and subsampling result in $\F'_4\in\mbr^{16\times5\times5}$ given by
\begin{equation*}
\F'_4(k, \cdot, \cdot) = \emat_{4}^T\left(\F_3(k, \cdot, \cdot)\star\frac{1}{4}\mathbbm{1}_{2\times2}\right)\emat_{4}
\end{equation*}
for $k=1,\dots, 4$.
Next, let $W_4\in\mbr^{16\times16}$ be the identity matrix and let $\k_4:\mbr^{12}\times\mbr^{12}\to\mbr$ be the arc-cosine kernel of order zero. 
The output of the fourth layer is then $\F_4\in\mbr^{16\times5\times5}$ given by
\begin{equation*}
\F_4(\cdot, i, j) = \left(\left[\k_4\left(\W_4(m, \cdot), \W_4(n, \cdot)\right)\right]_{m,n=1}^{16}\right)^{-1/2}\left[\k_4\left(\F'_4(\cdot, i, j), \W_4(m, \cdot)\right)\right]_{m=1}^{16}
\end{equation*}
for $i,j=1,\dots, 5$.

\paragraph{Layer 5}
The fifth layer is a fully connected layer. Let $\W_5\in\mbr^{120\times16\times5\times5}$ and let $\k_5:\mbr^{16}\times\mbr^{16}\to\mbr$ be the arc-cosine kernel of order zero.
 Then the output of this layer is given by $\F_5\in\mbr^{120}$ given by
\begin{equation*}
\F_5 = \left(\left[\k_5\left(\vec(\W_5(m, \cdot, \cdot, \cdot)), \vec(\W_5(n, \cdot, \cdot, \cdot))\right)\right]_{m,n=1}^{120}\right)^{-1/2}\left[\k_5\left(\vec(\F_4), \vec(\W_5(m, \cdot, \cdot, \cdot))\right)\right]_{m=1}^{120}.
\end{equation*}

\paragraph{Layer 6}
Finally, the sixth layer is also a fully connected layer. Let $\W_6\in\mbr^{84\times120}$ and let $\k_6:\mbr^{120}\times\mbr^{120}\to\mbr$ be the arc-cosine kernel of order zero. Then the output is given by $\F_6\in\mbr^{84}$ with
\begin{equation*}
\F_6 = \left(\left[\k_6\left(\W_6(m, \cdot), \W_6(n, \cdot)\right)\right]_{m,n=1}^{84}\right)^{-1/2}\left[\k_6\left(\F_5, \W_6(m, \cdot)\right)\right]_{m=1}^{84}.
\end{equation*}
The output from this layer is the set of features provided to a classifier.

\begin{figure}
\vspace*{-3cm}
\hspace*{-2.5cm}
\begin{center}
\includegraphics[scale=1.0,trim={1cm 5cm 0cm 1cm},clip]{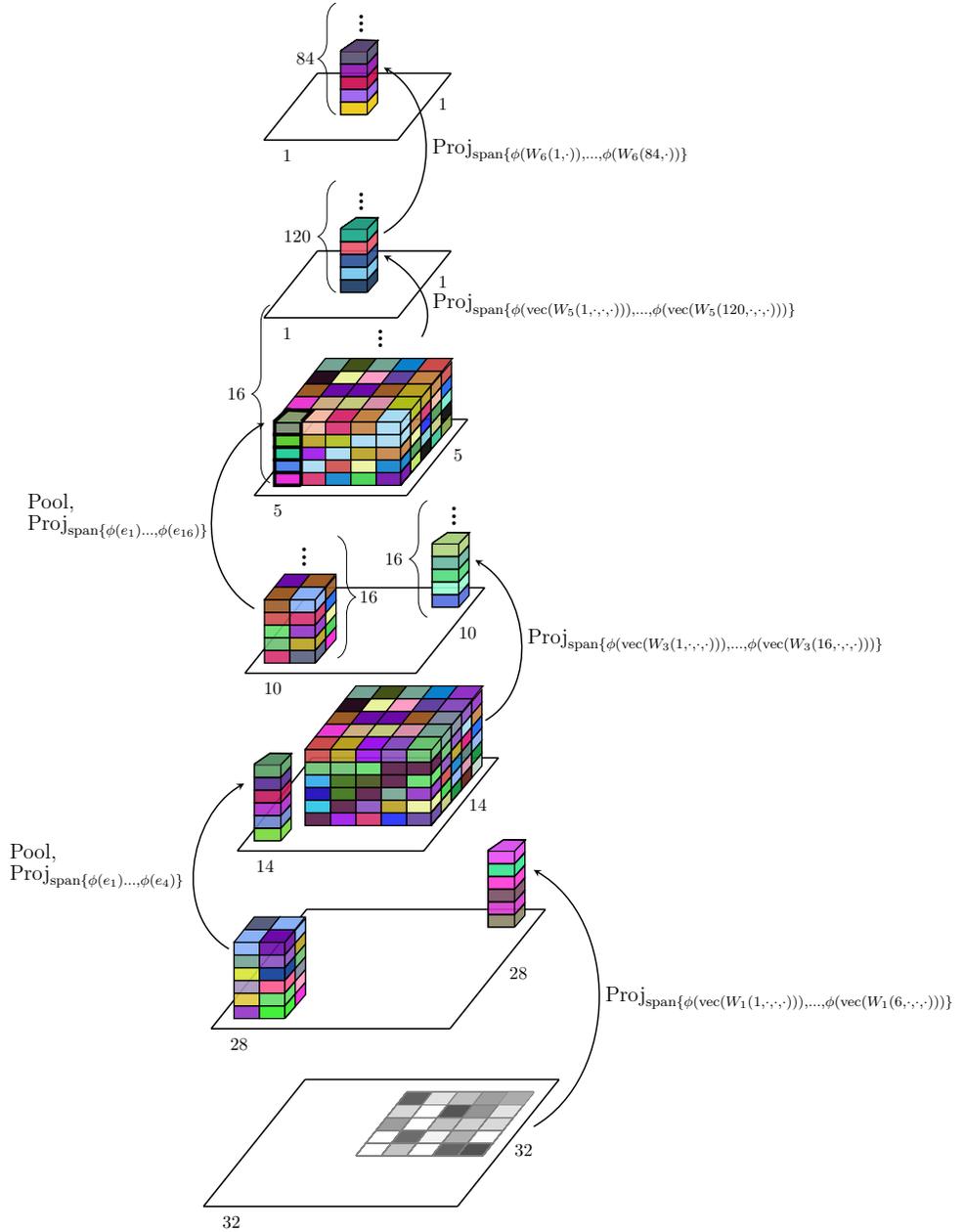}
\end{center}
\caption{\label{fig:lenet5_arch} LeNet-5 CKN architecture. The dimensions of the stacks of blocks are the dimensions of the filters at each layer. The numbers next to curly brackets indicate the number of filters at each layer when the number of filters is not the same as the height of the stack of blocks. The arrows indicate how one block gets transformed into the block at the next layer. The numbers on the sides of the parallelograms indicate the spatial dimensions of the feature representations at each layer. }%
\end{figure}

%% file: sections/sec4.tex
Like any functional mapping defined as a composition of modules differentiable with respect to their parameters, a CKN can be trained using a gradient-based optimization algorithm. An end-to-end learning approach to train a CKN in a supervised manner was first considered in \cite{mairal2016}. There are three essential ingredients to optimizing CKNs that we contribute in this section: (i) a rigorously derived general gradient formula; (ii) a numerically accurate gradient computation algorithm; and (iii) an efficient stochastic optimization algorithm.

\subsection{Gradient formula}
\label{sec:grad_formula}
When a CKN uses a differentiable dot product kernel $\patchk$, each layer of the CKN is differentiable with respect to its weights and inputs. Therefore, the entire CKN is differentiable. This provides a benefit over commonly used ConvNets that use non-differentiable activation functions such as the ReLU, which must be trained using subgradient methods rather than gradient methods. Note that, while widely used, only weak convergence guarantees are known for stochastic subgradient methods. Moreover, they require sophisticated topological non-smooth analysis notions \citep{davis2019}. As we shall show here, a CKN with a kernel corresponding to a smooth nonlinearity performs comparably to a ConvNet with non-smooth nonlinearities.

The derivatives of the loss function $\mcl$ from Section~\ref{sec:ckn_formulation} with respect to the filters at each layer and the inputs at each layer can be derived using the chain rule. First recall the output of a single convolutional layer presented in Section~\ref{sec:ckn_formulation}:
\begin{equation}\label{eq:ckn_layer}
F_{\ell} = \Psi_{\ell}(\im_{\ell-1}, \W_{\ell})\triangleq 
	(\patchk(\W_\ell\W_\ell^T)+\epsilon I)^{-1/2}\patchk(\W_\ell\E_\ell(\im_{\ell-1})\N_\ell(\im_{\ell-1})^{-1})\N_\ell(\im_{\ell-1})\P_\ell.
\end{equation}
We then have the following proposition, which is detailed in Appendix~\ref{app:gradient_proofs}. 

\begin{proposition}\label{prop:loss_grad}
	Let  $\mcl(y, \langle \W_{L+1}, \F_L\rangle)$ be the loss incurred by an image-label sample $(F_0, y)$, where 
	$F_L$ is the output of $L$th layer of the network described by ~\eqref{eq:ckn_layer} and $W_{L+1}$ parameterizes the linear classifier. 
	Then the 
	Jacobian of the loss 
	with respect to the inner weights $W_{\ell}$, $1\leq\ell\leq L$ is given by
	\begin{align*}
	\nabla_{\vec(W_{\ell})} \mcl\left(y, \left\langle \W_{L+1}, \F_L\right\rangle\right) =
	\mcl'\vec(W_{L+1})^T\left[\prod_{\ell '=\ell+1}^{L} \nabla_{\vec\left(F_{\ell'-1}\right)} \vec\left(\Psi_{\ell'}\right)\right]
	\nabla_{\vec(W_{\ell})} \vec\left(\Psi_{\ell}\right),
	\end{align*}
	where $\mcl'=\frac{\partial \mcl(\bar y, \hat y)}{\partial \hat y }|_{(\bar y, \hat y) = (y,\left\langle \W_{L+1}, \F_L\right\rangle)} $,  $\nabla_{\vec(F_{\ell'-1})} \vec(\Psi_{\ell'})$ is detailed in Proposition~\ref{prop:2layer_grad_part1}, and $\nabla_{\vec(W_{\ell})} \vec(\Psi_{\ell})$ is detailed in Proposition~\ref{prop:1layer_grad}.
\end{proposition}
Computing the derivatives of the output of a convolutional layer involves several linear algebra manipulations. The critical component lies in differentiating through the matrix inverse square root. For this we use the following lemma.

\begin{lemma}
\label{lemma:matrix_sqrt_main}
Define the matrix square root function $\gfun:S^n_{++}\to\mbr^{n\times n}$ by $\gfun(A)=A^{1/2}$.
Then for a positive definite matrix $A\in S^n_{++}$ and a matrix $H\in\mbr^{n\times n}$ such that $A+H \in S^n_{++}$ we have
\begin{align*}
\vec(\gfun(A+H)) = \vec(\gfun(A)) +  
		\left(\id_n \otimes A^{1/2}+A^{1/2}\otimes \id_n\right)^{-1}\vec(H) + o(\|H\|_F).
\end{align*}
\end{lemma}

Hence, computing the gradient of the CKN in this manner consists of solving a continuous Lyapunov equation. The remainder of the gradient computations involve Kronecker products and matrix multiplications. 

\subsection{Differentiating through the matrix inverse square root}
\label{sec:intertwined_newton}
\begin{algorithm}[t!]         
	\caption{\label{alg:intertwined_newton} $\textsc{Intertwined Newton Method for Matrix Inverse Square Root}$}      
	\begin{algorithmic}                     
		\STATE {\bfseries Input:}{
			Positive definite matrix $M\in\mbr^{d\times d}$,
			number of iterations $t_{\text{max}}$ \\
		}
		\STATE {\bfseries Initialize:}{
			$S_0=\|M\|_F^{-1}M$, $T_0=\id_d$
		}
		\FOR{$t=1,\dots, t_{\text{max}}$}{ 
			\STATE{$S_{t+1} \gets \frac{1}{2}S_t\left(3\id_d - T_tS_t\right)$}\\
			\STATE{$T_{t+1} \gets \frac{1}{2}\left(3\id_d - T_tS_t\right)T_t$}
		}
		\ENDFOR\\
		$ T\gets \|M\|_F^{-1/2}T_{t_{\text{max}}}$\\
		\STATE {\bfseries Output: }{$T$ (the approximate matrix inverse square root of $M$)}
	\end{algorithmic}
\end{algorithm}
The straightforward approach to computing the derivative of the matrix inverse square root involved in Proposition~\ref{prop:loss_grad} is to call a solver for continuous-time Lyapunov equations. However this route becomes an impediment for large-scale problems requiring fast matrix-vector computations on GPUs. An alternative is to leave it to current automatic differentiation software, which computes it through a singular value decomposition. This route does not leverage the structure and leads to worse estimates of  gradients (See Section~\ref{sec:training_comparison}).

Here we propose a simple and effective approach based on two intertwined Newton methods. Consider the matrix $M=\k\left(\W\W^T\right) + \epsilon \id_\f \in S_{++}^\f$. We aim to compute $M^{-1/2}$ by an iterative method.
Denote by $\lambda_1, \ldots, \lambda_\f$ the eigenvalues of $M$ and let
\[
H = \left(\begin{matrix}
0 & M \\
\id_d & 0
\end{matrix}\right).
\]
As the eigenvalues of $H$ are $(\pm \sqrt{\lambda_i})_{i=1,\ldots, \f}$, $H$ can be diagonalized as
$H = U \Lambda U^{-1}$ with $\Lambda$ the diagonal matrix of its eigenvalues. Let $\sign(H) = U \sign(\Lambda) U^{-1}$ where $\sign(\Lambda)$ is a diagonal matrix whose diagonal is the sign of the eigenvalues in $\Lambda$. Then $\sign(H)$ satisfies \citep[Theorem 5.2]{higham2008}
\[
\sign(H) =  \left(\begin{matrix}
0 & M^{1/2} \\
M^{-1/2} & 0 
\end{matrix}\right).
\]
The sign matrix of $H$ is a square root of the identity, i.e., $\sign(H)\sign(H)= \id_\f$. It can then be computed by a Newton's method starting from $X_0= H$ followed by $X_{t+1} = (X_t + X_t^{-1})/2$. Provided that $\|H\|_2\leq 1$, it converges quadratically to $\sign(H)$~\citep[Theorem 5.6]{higham2008}. Decomposing the iterates of this Newton's method on the blocks defined in $H$ give Denman and Beavers' algorithm \citep{denman1976}. This algorithm begins with $S_0=M$ and $T_0=\id_\f$ and proceeds with the iterations $S_{t+1} = (S_t + T_t^{-1})/2$ and $T_{t+1} = (T_t + S_t^{-1})/2$. The sequence $T_t$ then converges to $M^{-1/2}$.

Each iteration, however, involves the inverses of the iterates $T_t$ and $S_t$, which are expensive to compute when using a large number of filters.
We propose applying the Newton method one more time, yet now to compute $T_t^{-1}$ and $S_t^{-1}$ (sometimes called the Newton-Schulz method), starting respectively from $S_t$ and $T_t$ as initial guesses \citep{higham1997}. An experimental evaluation of this strategy when we run, say, $20$ iterations of the outer Newton method (to compute the inverse matrix square root) yet only $1$ iteration of the inner Newton method (to compute the inverse matrices) demonstrates that it is remarkably effective in practice (See Figure~\ref{fig:svd_vs_newton} in Section~\ref{sec:training_comparison}). We present the pseudocode in Algorithm~\ref{alg:intertwined_newton} for the case where one iteration of the inner Newton's method is used. Note that we first scale the matrix $M$ by its Frobenius norm to ensure convergence.

By differentiating through these iterations we can obtain the derivatives of $M^{-1/2}$ with respect to the entries of $M$. Comparing the accuracy of the gradient obtained using this algorithm to the one returned using automatic differentiation in PyTorch, we find that our approach is twice as accurate. Furthermore, the algorithm only involves matrix multiplications, which is critical to scale to large problems with GPU computing.  
Hence, this provides a better means of computing the gradient. %

%% file: sections/optim_ckn.tex
\subsection{Training procedures}
\label{sec:supervised_training}

	The training of CKNs consists of two main stages: unsupervised and supervised learning. The unsupervised learning entails first initializing the filters with an unsupervised method and then fixing the filters and optimizing the ultimate layer using the full dataset. The supervised learning entails training the whole initialized architecture using stochastic estimates of the objective. 
	Here we detail the second stage, for which we propose a new approach. Algorithm~\ref{alg:supervised_learning} outlines the overall CKN training with this new method. 
	
	\begin{algorithm}[t!]         
\caption{\label{alg:supervised_learning} $\textsc{Supervised Training of CKNs}$}                         
  \begin{algorithmic}
\STATE {\bfseries Input:}{\\
- Uninitialized CKN with $L$ layers\\
 - Inputs $\F_0^{(i)}\in\mbr^{\f_0\times\p_0'}$ and labels $y^{(i)}$, $i=1,\dots, n$\\
 - Number of iterations of alternating minimization to perform, $T$ \\
}
\vspace{0.1cm}
\STATE{\textbf{Initialization:}}
\STATE{-- Perform unsupervised training of the filters $\W_1,\dots, \W_L$ using spherical $k$-means.}
\STATE{-- Compute the features $\F_L^{(i)}$ for all inputs $i$. Train the weights $\W_{L+1}$ of the classifier using a quasi-Newton method.}\\  
\vspace{0.1cm}
\STATE{\textbf{Supervised training:}}
\FOR{$j=1,\dots, T$}
\STATE{-- Compute the features $\F_L^{(i)}$ for a mini-batch of inputs.}\\  
\STATE{-- Perform one step of the Ultimate Layer Reversal method (see Algorithm~\ref{alg:ulr}).}
\ENDFOR
\STATE{-- Compute the features $\F_L^{(i)}$ for all inputs $i$. Train the weights $\W_{L+1}$ of the classifier using a quasi-Newton method.}\\  
\STATE {\bfseries Output: }{$\W_1,\dots, \W_L, \W_{L+1}$ (the optimal filters and classifier weights) }
\end{algorithmic}
\end{algorithm}

	\paragraph{Stochastic gradient optimization on manifolds}\label{sec:stoc_grad}
	A major difference between CKNs and ConvNets is the spherical constraints imposed on the inner layers. On the implementation side, this simply requires an additional projection during the gradient steps for those layers. 
	On the theoretical side it amounts to a stochastic gradient step on a manifold whose convergence to a stationary point is still ensured, provided that the classifier is not regularized but constrained. Specifically, given image-label pairs $(F_0^{(i)}, y^{(i)})_{i=1}^n$, we consider the constrained empirical risk minimization problem
	\begin{align}
	\min_{\W_1,\dots,\W_{L+1}}\quad &  \frac{1}{n} \sum_{i=1}^n  \mathcal{L}\left(y^{(i)}, \left\langle \W_{L+1}, \im_L^{(i)}\right\rangle\right) \label{eq:training_constrained}\\
	\mbox{subject to} \quad & W_\ell \in S^{d_\ell} \qquad \qquad \mbox{for $\ell=1,\ldots, L$} \nonumber \\ 
	& \|W_{L+1}\|_F \leq \lambda ,\nonumber
	\end{align}
	where $S^{d_\ell}$ is the product of Euclidean unit spheres at the $\ell$\textsuperscript{th} layer and $F_L^{(i)}$ is the output of $L$ layers of the network described by~\eqref{eq:ckn_layer}.
	Projected stochastic gradient descent draws a mini-batch $B_t$ of samples at iteration $t$, forming an estimate $f_{B_t}$ of the objective, and performs the following update: 
	\begin{align} \tag{SGO}
	W_\ell^{(t+1)} & = \Proj_{S^{d_\ell}}\left(W_\ell^{(t)} - \stepsize_{t} \nabla_{W_\ell}  f_{B_t}(W^{(t)}) \right) \qquad \qquad \mbox{for $\ell=1,\ldots, L$} \label{eq:sgo}\\
	W_{L+1}^{(t+1)} & = \Proj_{\mathcal{B}_{2,\lambda}} \left(W_{L+1}^{(t)} - \stepsize_{t} \nabla_{W_{L+1}}  f_{B_t}(W^{(t)}) \right), \nonumber
	\end{align}
	where $\mathcal{B}_{2,\lambda}$ is the Euclidean ball centered at the origin of radius $\lambda$ and $\gamma_t$ is a step size. Its convergence is stated in the following proposition, detailed in Appendix~\ref{app:sg_manifolds}.
	\begin{proposition}\label{prop:sgo_constrained}
	Assume the loss in the constrained training problem~\eqref{eq:training_constrained} and the kernel defining the network~\eqref{eq:ckn_layer} are continuously differentiable.
	Projected stochastic gradient descent with step size $\stepsize_t = c/\sqrt{T}$, where $c>0$ and $T$ is the maximum number of iterations, finds an $O(1/\sqrt{T})$-stationary point.
	\end{proposition}
In practice we use the penalized formulation to compare with classical optimization schemes for ConvNets. %
	
	\paragraph{Ultimate layer reversal} \label{sec:reversal}
	The network architectures present a discrepancy between the inner layers and the ultimate layer: the former computes a feature representation, while the latter is a simple classifier that could be optimized easily once the inner layers are fixed. This motivates us to back-propagate the gradient in the inner layers \emph{through the classification performed in the ultimate layer}. 
	Formally, consider the regularized empirical risk minimization problem~\eqref{eq:training_pb},
	\begin{align*}
		\min_{W \in \mathcal{C}, V}\quad & f(W,V) \coloneqq \frac{1}{n} \sum_{i=1}^n \mathcal{L}\left(y^{(i)}, \left\langle V, \im_L^{(i)}(W)\right\rangle\right)+\lambda \Vert \V\Vert_F^2
	\end{align*}
	where $W = (W_1,\ldots,W_L)$ denotes the parameters of the inner layers constrained on spheres in the set $\mathcal{C}$, $V = W_{L+1}$ parameterizes the last layer and $\im_L^{(i)}(W)$, $i=1,\dots, n$ are the feature representations of the images output by the network. The problem can be simplified as
	\[
	\min_{W \in \mathcal{C},V} f(W,V) = \min_{W \in \mathcal{C}}  \simpobj(W) \qquad  \mbox{where} \qquad \simpobj(W) \triangleq \min_V f(W, V) .
	\]
	Strong convexity of the classification problem ensures that the simplified problem is differentiable and its stationary points are stationary points of the original objective. This is recalled in the following proposition, which is detailed in Appendix~\ref{app:ulr}.
	\begin{proposition}\label{prop:ulr_grad}
	Assume that $f(W,V)$ is twice differentiable and that for any $W$, the partial functions $V \rightarrow f(W,V)$ are strongly convex. Then the simplified objective $\simpobj(W) = \min_V f(W, V) $ is differentiable and satisfies
	\[
	\|\nabla \simpobj(W)\|_2 = \|\nabla f(W, V^*)\|_2,
	\]
	where $V^* = \argmin_V f(W, V) $. 
	\end{proposition}
	Therefore if a given $W^*$ is $\epsilon$-near stationary for the simplified objective $\simpobj$, then the pair $(W^*, V^*(W))$, where $V^*(W) = \argmin_V f(W, V)$, is $\epsilon$-near stationary for the original objective $f$.
	
	\emph{Least squares loss} In the case of the least squares loss, the computations can be performed analytically, as shown in Appendix~\ref{app:ulr}. 
	However, the objective cannot be simplified on the whole dataset, since it would lose its decomposability in the samples. Instead, we apply this strategy on mini-batches. I.e., at iteration $t$, denoting $f_{B_t}$ the objective formed by a mini-batch $B_t$ of the samples, the algorithm updates the inner layers via, for $\ell =1,\ldots L$,
	\[
	W_\ell^{(t+1)} = \Proj_{S^{d_\ell}}\left(W_\ell^{(t)} - \stepsize_{t} \Proj_{S^{d_\ell}}\left(\nabla_{W_\ell} \simpobj_{B_t}(W^{(t)})\right) \right)
	\]
	where $\simpobj_{B_t}(W^{(t)}) = \min_V f_{B_t}(W^{(t)}, V) $ and we normalize the gradients by projecting them on the spheres to use a single scaling for all layers.

	\emph{Other losses}
	For other losses such as the multinomial loss, no analytic form exists for the minimization. At each iteration $t$ we therefore approximate the partial objective $g_{B_t}(\cdot;W):V \rightarrow f_{B_t}(W, V)$ on the mini-batch $B_t$ by a regularized quadratic approximation and perform the step above on the inner layers. 
The ultimate layer reversal step at iteration $t$ is detailed in Algorithm~\ref{alg:ulr}.
 The quadratic approximation $q_{g_{B_t}}^{(t)}(V; W_{1:L})$ in Step~\ref{item:quad_approx} depends on the current point $V^{(t)}$ and can be formed using the full Hessian or a diagonal approximation of the Hessian. 
	The gradient in Step~\ref{item:grad_step} is computed by back-propagating through the operations. 

	\begin{algorithm}[t!]         
\caption{\label{alg:ulr} $\textsc{Ultimate Layer Reversal Step}$}                         
  \begin{algorithmic}
\STATE {\bfseries Input:}{\\
- Mini-batch of inputs $B_t$\\
- Overall objective function $f(W,V)$\\
 - Current iterates $W^{(t)}$, $V^{(t)}$\\
- Step size $\gamma_t$\\
- Regularization parameter $\tau$\\
}
\vspace{0.1cm}
\algorithmicrepeat
\STATE{1. Approximate the classifier around current point $V^{(t)}:$   \label{item:quad_approx}
		\[
		 g_{B_t}(V;  W^{(t)})  \approx q_{g_{B_t}}^{(t)}(V; W^{(t)}) + \frac{\tau}{2} \|V-V^{(t)}\|_F^2.
		 \]}
\STATE{2. Minimize the approximation:
		\[
		\simpobj _{B_t}(W^{(t)}) = \min_V q_{g_{B_t}}^{(t)}(V; W^{(t)}) + \frac{\tau}{2} \|V-V^{(t)}\|_2^2.
		\]}
\STATE{3. Take a projected gradient step:  \label{item:grad_step}
		\[
		 W_\ell^{(t+1)}  = \Proj_{S^{d_\ell}}\left(W_\ell^{(t)}  - \stepsize_{t} \Proj_{S^{d_\ell}}\left(\nabla_{W_\ell} \simpobj_{B_t}(W^{(t)})\right) \right) \quad \mbox{for $\ell=1,\ldots, L$}.
		 \]}
\STATE{4. Update the classifier on which the approximation is taken: 
		\[
		V^{(t+1)}  = \underset{V}{\text{argmin}} \,q_{g_{B_t}}^{(t)}(V; W^{(t+1)}) + \frac{\tau}{2} \|V-V^{(t)}\|_F^2.
		\]}
\STATE {\bfseries Output: }{$W^{(t+1)}$, $V^{(t+1)}$ }
\end{algorithmic}
\end{algorithm}

%% file: sections/sec5.tex
In the experiments we seek to address the following two questions:
\begin{enumerate}
\item  How well do the proposed training methods perform for CKNs?
\item Can a supervised CKN attain the same performance as its ConvNet counterpart? 
\end{enumerate}
Previous works reported that specially-designed CKNs can achieve comparable performance to ConvNets in general on MNIST and CIFAR-10 \citep{mairal2014,mairal2016}. Another set of previous works designed hybrid architectures mixing kernel-based methods and ConvNet ideas \citep{bruna2013,huang2014,dai2014,yang2015,oyallon2017,oyallon2018}. 
We are interested here in whether, given a ConvNet architecture, an analogous CKN can be designed and trained to achieve similar or superior performance. Our purely kernel-based approach stands in contrast to previous works as, for each (network, dataset) pair, we consider a ConvNet and its CKN \emph{counterpart}, hence compare them on an equal standing, for varying numbers of filters.

\subsection{Experimental details}
\label{sec:experiments_methods}
The experiments use the datasets MNIST and CIFAR-10 \citep{lecun1998a,krizhevsky2009}.  
MNIST consists of 60,000 training images and 10,000 test images of handwritten digits numbered 0-9 of size $28\times28$ pixels. In contrast, CIFAR-10 consists of 50,000 training images and 10,000 test images from 10 classes of objects of size $3\times32\times32$ pixels. %

The raw images are transformed prior to being input into the networks. Specifically, the MNIST images are standardized while the CIFAR-10 images are standardized channel-wise and then ZCA whitened on a per-image basis. %
Validation sets are created for MNIST and CIFAR-10 by randomly separating the training set into two parts such that the validation set has 10,000 images. 

The networks we consider in the experiments are LeNet-1 and LeNet-5 on MNIST \citep{lecun1998a} and All-CNN-C on CIFAR-10 \citep{springenberg2014,krizhevsky2009}. %
 LeNet-1 and LeNet-5 are prominent examples of first modern versions of ConvNets. They use convolutional layers and pooling/subsampling layers and achieved state-of-the-art performance on digit classification tasks on datasets such as MNIST.
 The ConvNets from \citet{springenberg2014}, including All-CNN-C, were the first models used to make the claim that pooling is unnecessary. All-CNN-C was one of the best-performing models on CIFAR-10 at the time of publication. %
 For mathematical descriptions of the ConvNets and their CKN counterparts, see Appendices~\ref{app:convnet_desc} and \ref{app:ckn_desc}, respectively.

Using the principles outlined in Section~\ref{sec:translation}, we translate each architecture to its CKN counterpart. The networks are in general reproduced as faithfully as possible. However, there are a few differences between the original implementations and ours. In particular, the original LeNets have an incomplete connection scheme at the third layer in which each feature map is only connected to a subset of the feature maps from the previous layer. This was included for computational reasons. In our implementation of the LeNets we find that converting the incomplete connection scheme to a complete connection scheme does not decrease performance (See Appendix~\ref{app:lenet_incomplete}). We therefore use the complete connection schemes in our ConvNet and CKN implementations. In addition, the original All-CNN-C has a global average pooling layer as the last layer. In order to have trainable unconstrained parameters in the CKN, we add a fully connected layer after the global average pooling layer in the ConvNet and CKN. Also note that we apply zero-padding at the convolutional layers that have a stride of one to maintain the spatial dimensions at those layers. Moreover, we omit the dropout layers. %
Since CKNs do not have biases, we omit the biases from the ConvNets.
Lastly, as the arc-cosine kernels are not differentiable, we switch to using the RBF kernel on the sphere for the supervised CKN implementations. The nonlinearity generated by this kernel resembles the ReLU \citep{mairal2014}. We fix the bandwidths to 0.6.\footnote{Note, however, that it is possible to train the bandwidths; see Proposition~\ref{prop:ckn_loss_grad_bw} in Appendix~\ref{app:gradient_proofs}.}

The training of the ConvNets used in the experiments is performed as follows. The initialization is performed using draws from a mean-zero random normal distribution. For the LeNets the standard deviation is set to 0.2 while for All-CNN-C the standard deviation is set to 0.3. The output features are normalized in the same way as for the CKNs, so they are centered and on average have an $\ell_2$ norm of one. The multinomial logistic loss is used and trained with SGD with momentum set to 0.9.  The batch size is set to the largest power of two that fits on the GPU when training the CKN counterpart (see Table~\ref{tab:batch_sizes} in Appendix~\ref{app:training_details_results}). 
The step size is chosen initially from the values $2^i$ for $i=-10, -9, \dots, 2$ by training for five iterations with each step size and choosing the step size yielding the lowest training loss on a separate mini-batch. The same method is used to update the step size every 100 iterations, except at subsequent updates the step size is selected from $2^{i}s$ for $i=-3,-2,\dots, 3$, where $s$ is the current step size. For All-CNN-C we monitor the training accuracy every epoch. If the accuracy decreases by more than 2\% from one epoch to the next we replace the current network parameters with those from the previous epoch and decrease the learning rate by a factor of 4. 
Cross-validation is performed over the values $2^i$ for $i=-40, -39, \dots, 0$ for the $L_2$ penalty of the multinomial logistic loss parameters. During cross-validation the optimization is performed for 1000 iterations. The final optimization using the optimal penalty is performed for 10,000 iterations. 

Now we detail the unsupervised CKN initialization. The unsupervised training of the CKNs entails approximating the kernel at each layer and then training a classifier on top. Unless otherwise specified, the kernel approximations are performed using spherical $k$-means layer-wise with 10,000 randomly sampled non-constant patches per layer, all from different images. Unless otherwise specified, when evaluating the CKN at each layer the intertwined Newton method is used. In order to achieve a high accuracy but keep the computational costs reasonable the number of outer Newton iterations is set to 20 and the number of inner Newton iterations is set to 1. The regularization of the Gram matrix on the filters is set to 0.001. After the unsupervised training the features are normalized as in \citet{mairal2014} so that they are centered and on average have an $\ell_2$ norm of one. A classifier is trained on these CKN features using the multinomial logistic loss. The loss function is optimized using L-BFGS \citep{liu1989} on all of the features with the default parameters from the Scipy implementation.  Cross-validation is performed over the values $2^i$ for $i=-40, -39, \dots, 0$ for the $L_2$ penalty of the multinomial logistic loss parameters.  Both the cross-validation and the final optimization with the optimal penalty are performed for a maximum of 1000 iterations.

Finally, we describe the supervised CKN training. The supervised training of CKNs begins with the unsupervised initialization. The multinomial logistic loss is used and trained with our ultimate layer reversal method. 
In the ultimate layer reversal method we use an approximation of the full Hessian for all but the LeNet-1 experiment with 128 filters per layer. Due to memory constraints we use a diagonal approximation to the Hessian for the LeNet-1 experiment with 128 filters per layer. The regularization parameter $\tau$ of the Hessian was selected via cross-validation and set to 0.03125 for the LeNets and to 0.0625 for All-CNN-C.
The batch size is set to the largest power of two that fits on the GPU (see Table~\ref{tab:batch_sizes} in Appendix~\ref{app:training_details_results}). The step sizes are determined in the same way as for the ConvNets, but the initial step sizes considered are $2^i$ for $i=-6, -5, \dots, 2$. 
The $L_2$ penalty of the multinomial logistic loss parameters is fixed to the initial value from the unsupervised training throughout the ULR iterations. 
After 10,000 iterations of ULR, the parameters of the loss function are once again optimized with L-BFGS for a maximum of 1000 iterations. Cross-validation over the $L_2$ penalty of the multinomial logistic loss parameters is once again performed at this final stage in the same manner as during the unsupervised initialization.
 
The code for this project was primarily written using PyTorch \citep{paszke2017} and may be found online at \url{https://github.com/cjones6/yesweckn}. FAISS \citep{johnson2017} is used during the unsupervised initialization of the CKNs. We ran the experiments on Titan Xps, Titan Vs, and Tesla V100s. The corresponding time to run the experiments on an NVIDIA Titan Xp GPU would be more than 20 days.

\subsection{Comparison of training methods}
\label{sec:training_comparison}
We commence by demonstrating the superiority of our proposed training methods described in Section~\ref{sec:training} to the standard methods.
 
\paragraph{Accuracy of the gradient computation}
 \begin{figure*}[t!]
\centering
\begin{subfigure}{.48\textwidth}
  \centering
\includegraphics[width=1.05\linewidth]{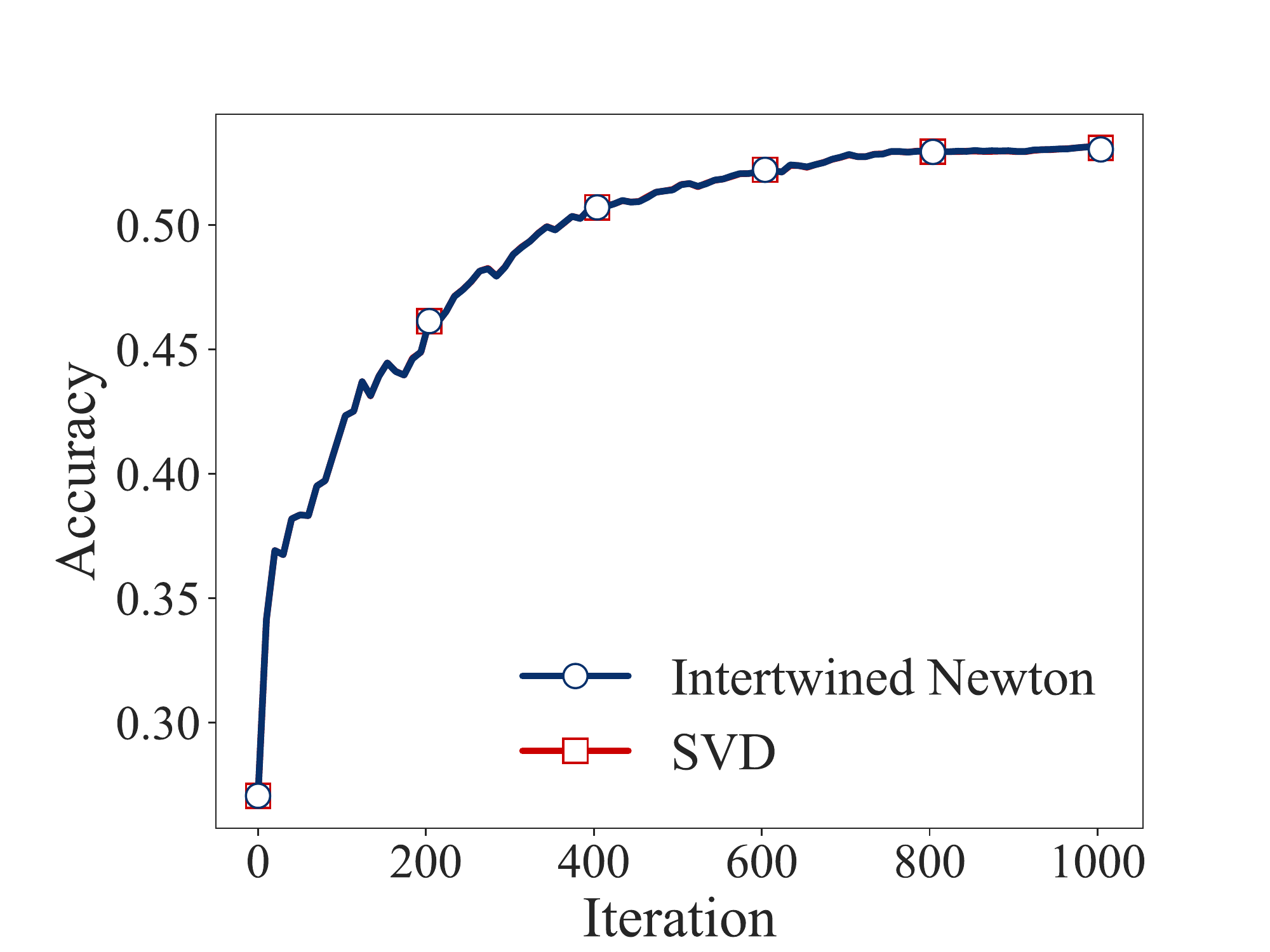}
 \caption{\label{fig:allcnn_newton8} All-CNN-C CKN on CIFAR-10 with 8 filters/layer}
\end{subfigure}\hfill
\begin{subfigure}{.48\textwidth}
  \centering
\includegraphics[width=1.05\linewidth]{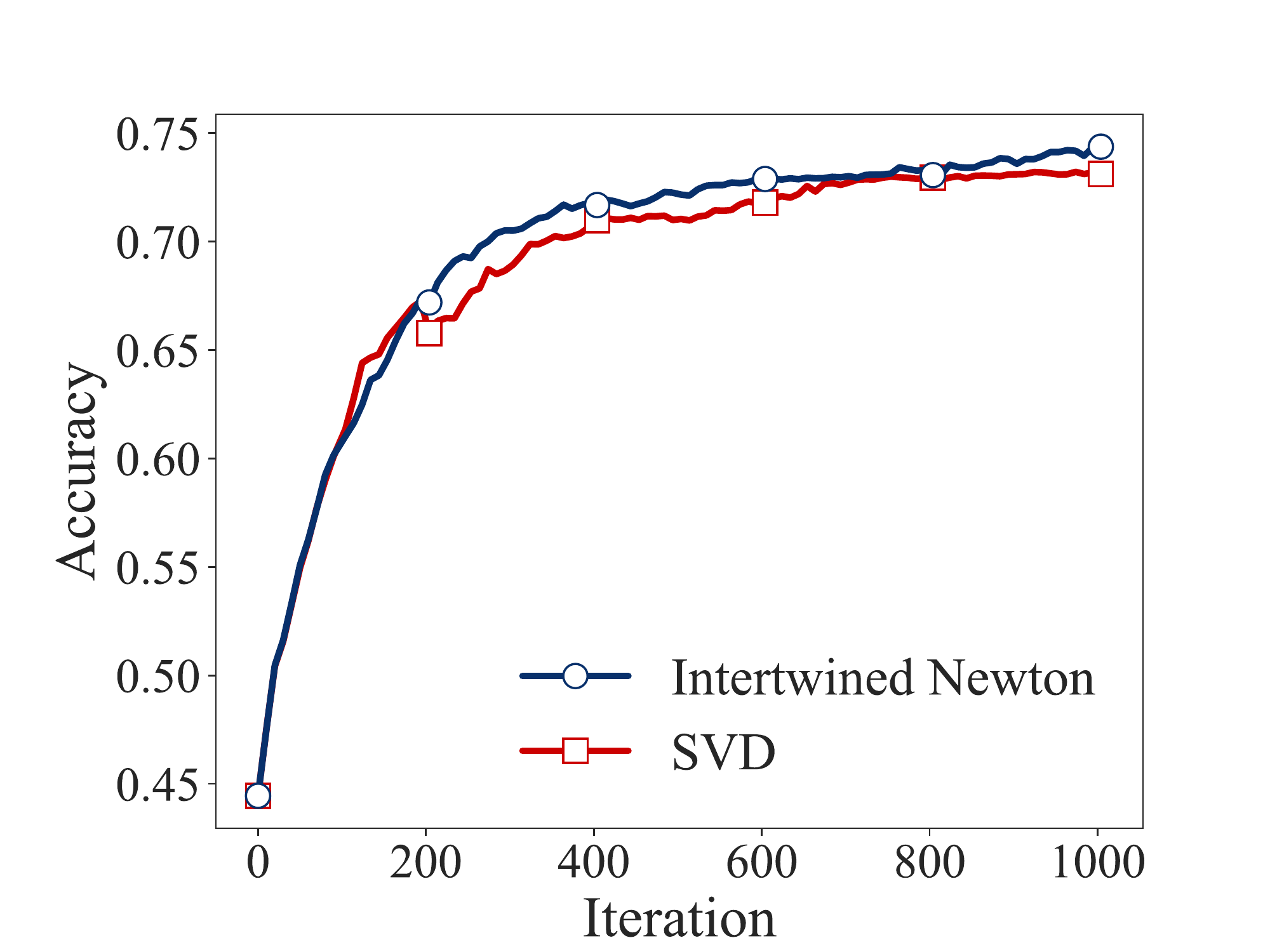}
 \caption{\label{fig:allcnn_newton128} All-CNN-C CKN on CIFAR-10 with 128 filters/layer}
\end{subfigure}\hfill
\caption{\label{fig:svd_vs_newton} Performance of CKNs when using an SVD to compute the matrix inverse square root vs. using our intertwined Newton method. We report results from using 50 iterations of the outer Newton method and one iteration of the inner Newton method.}
\end{figure*}

The straightforward way of computing the gradient of a CKN is by allowing automatic differentiation software to differentiate through SVDs. In Section~\ref{sec:intertwined_newton} we introduced an alternative approach: the intertwined Newton method.  Here we compare the two approaches when training the deepest network we consider in the experiments: the CKN counterpart to All-CNN-C. %
We compare the gradients from differentiating through the SVD and the intertwined Newton method in two ways: directly and also indirectly via the performance when training a CKN. 

First, we compare the gradients from each method to the result from using a finite difference method. We find that for the CKN counterpart to All-CNN-C on CIFAR-10 with 8 filters/layer, differentiating through 20 Newton iterations yields relative errors that are 2.5 times smaller than those from differentiating through the SVD. This supports the hypothesis that differentiating through the SVD is more numerically unstable than differentiating through Newton iterations. We moreover note that using Newton iterations allows us to control the numerical accuracy of the gradient and of the matrix inverse square root itself.

Given that the gradients from the intertwined Newton method are more accurate, we now investigate whether this makes a difference in the training. 
Figure~\ref{fig:svd_vs_newton} compares the performance of the two methods on All-CNN-C with 8 and 128 filters/layer. We set the number of outer Newton iterations to 50 and leave the number of inner Newton iterations at 1. From the plots we can see that for 8 filters/layer there is no difference in the training performance, despite the gradients for the intertwined Newton method being more accurate. However, for 128 filters/layer the intertwined Newton method begins to outperform the SVD after approximately 200 iterations. After 1000 iterations the accuracy from differentiating through the intertwined Newton method is 1.7\% better than that from differentiating through the SVD. The intertwined Newton method therefore appears to be superior for larger networks.

\paragraph{Efficiency of training methods}
 \begin{figure*}[t!]
\centering
\begin{subfigure}{.48\textwidth}
  \centering
\includegraphics[width=1.05\linewidth]{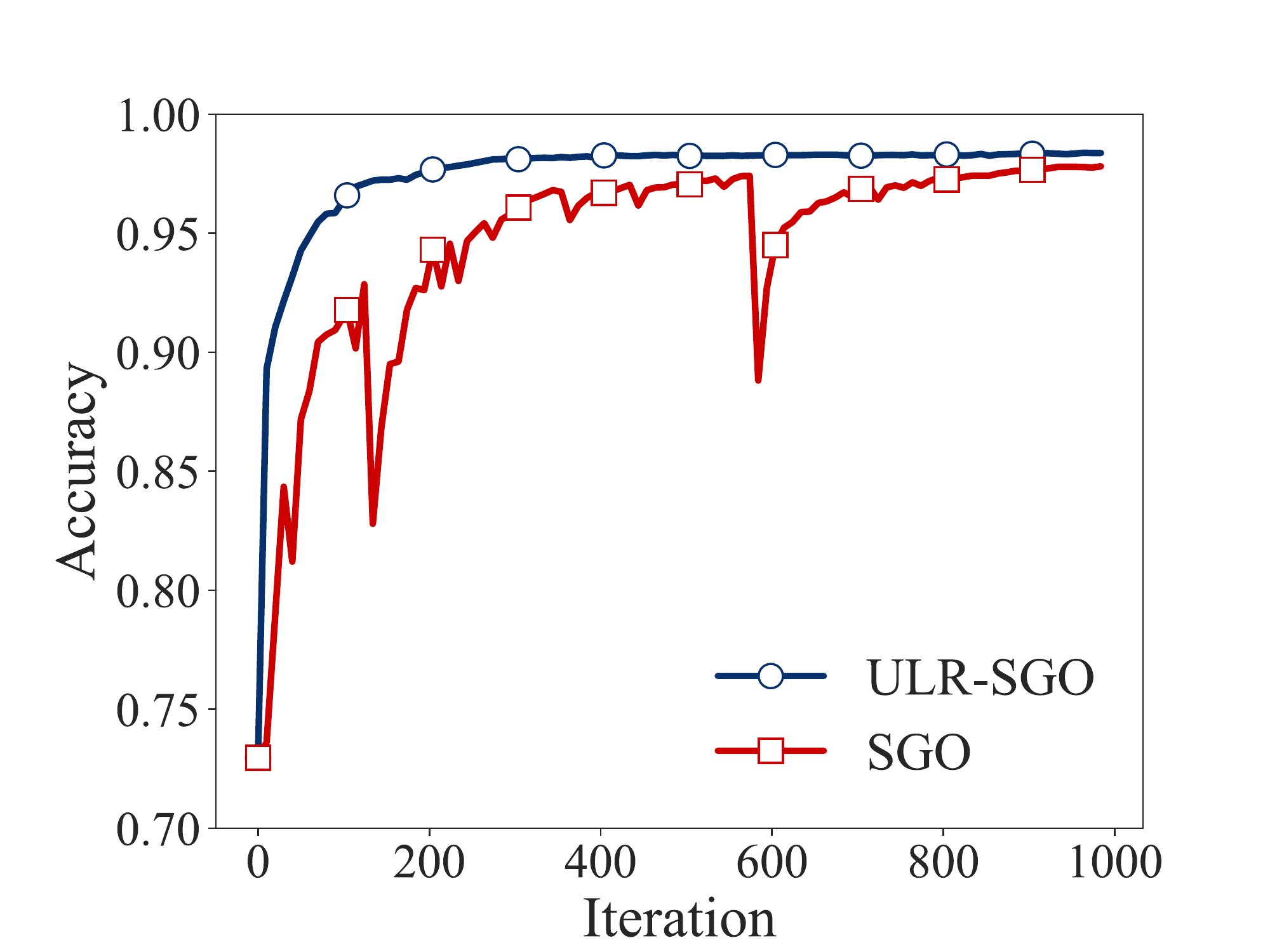}
 \caption{\label{fig:lenet5_8_ulr_sgo} LeNet-5 CKN on MNIST with 8 filters/layer}
\end{subfigure}\hfill
\begin{subfigure}{.48\textwidth}
  \centering
\includegraphics[width=1.05\linewidth]{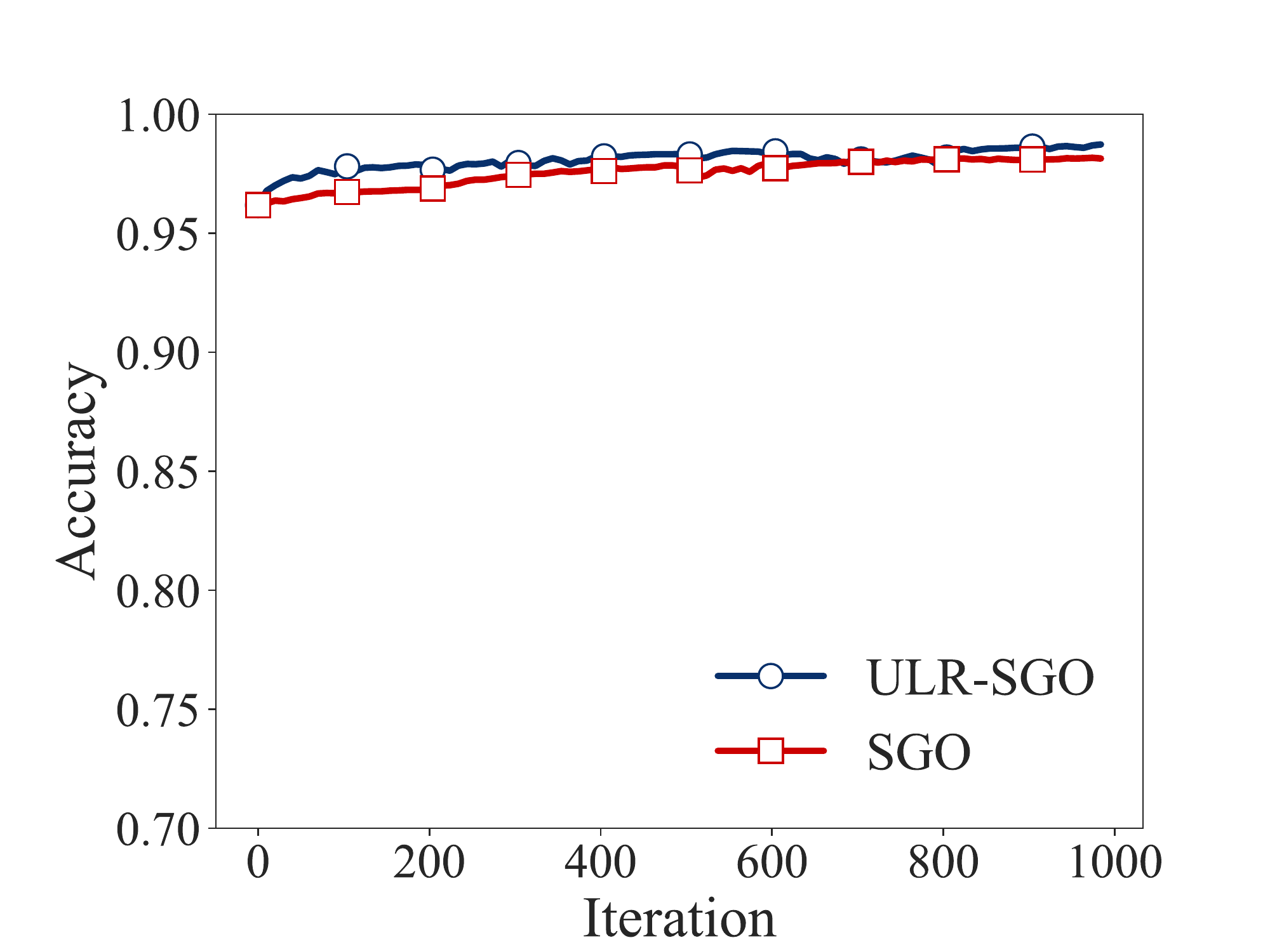}
 \caption{\label{fig:lenet5_128_ulr_sgo} LeNet-5 CKN on MNIST with 128 filters/layer}
\end{subfigure}\hfill
\begin{subfigure}{.48\textwidth}
  \centering
\includegraphics[width=1.05\linewidth]{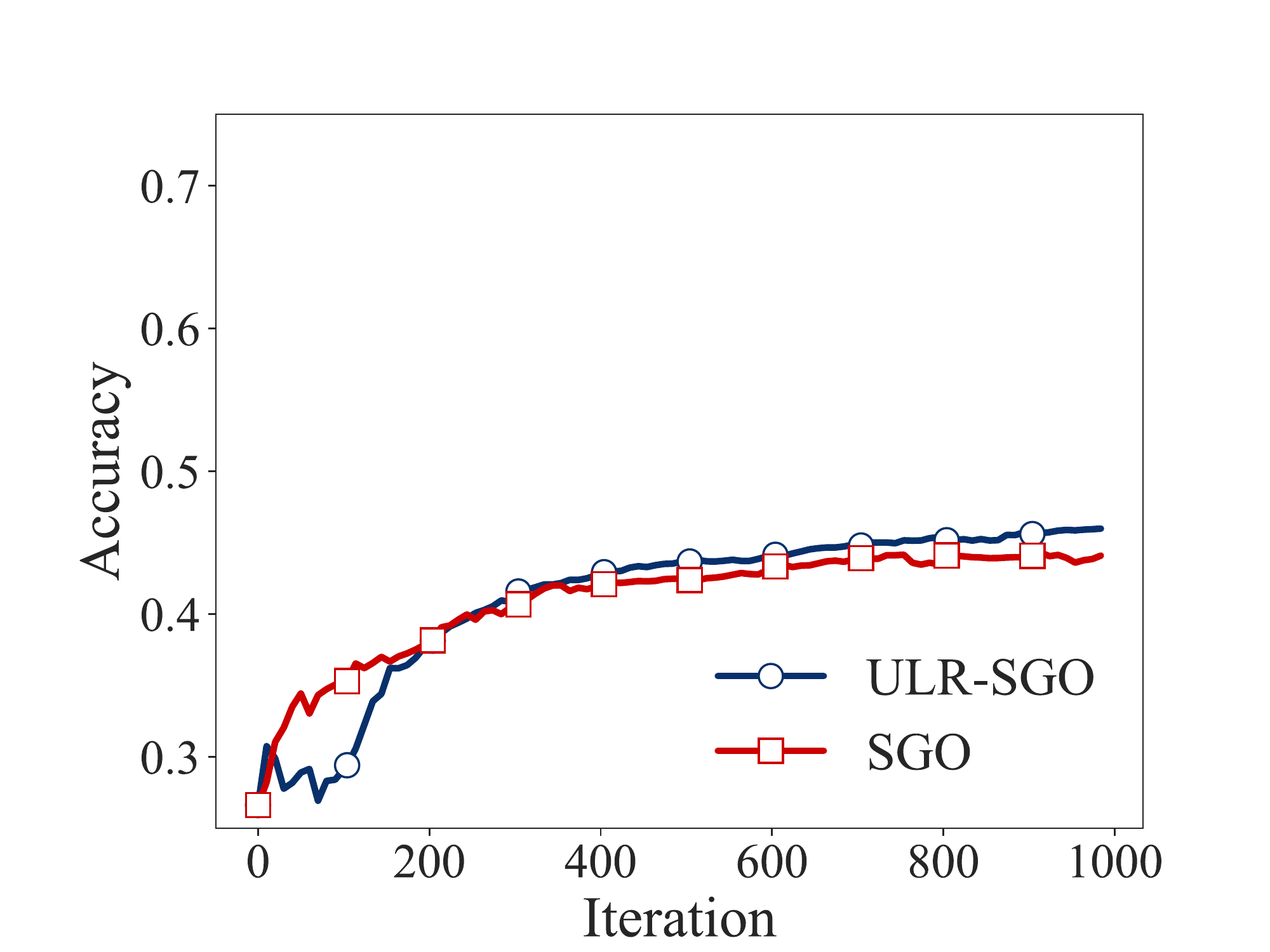}
 \caption{\label{fig:lenet5_8_ulr_sgo} All-CNN-C CKN on CIFAR-10 with 8 filters/layer}
\end{subfigure}\hfill
\begin{subfigure}{.48\textwidth}
  \centering
\includegraphics[width=1.05\linewidth]{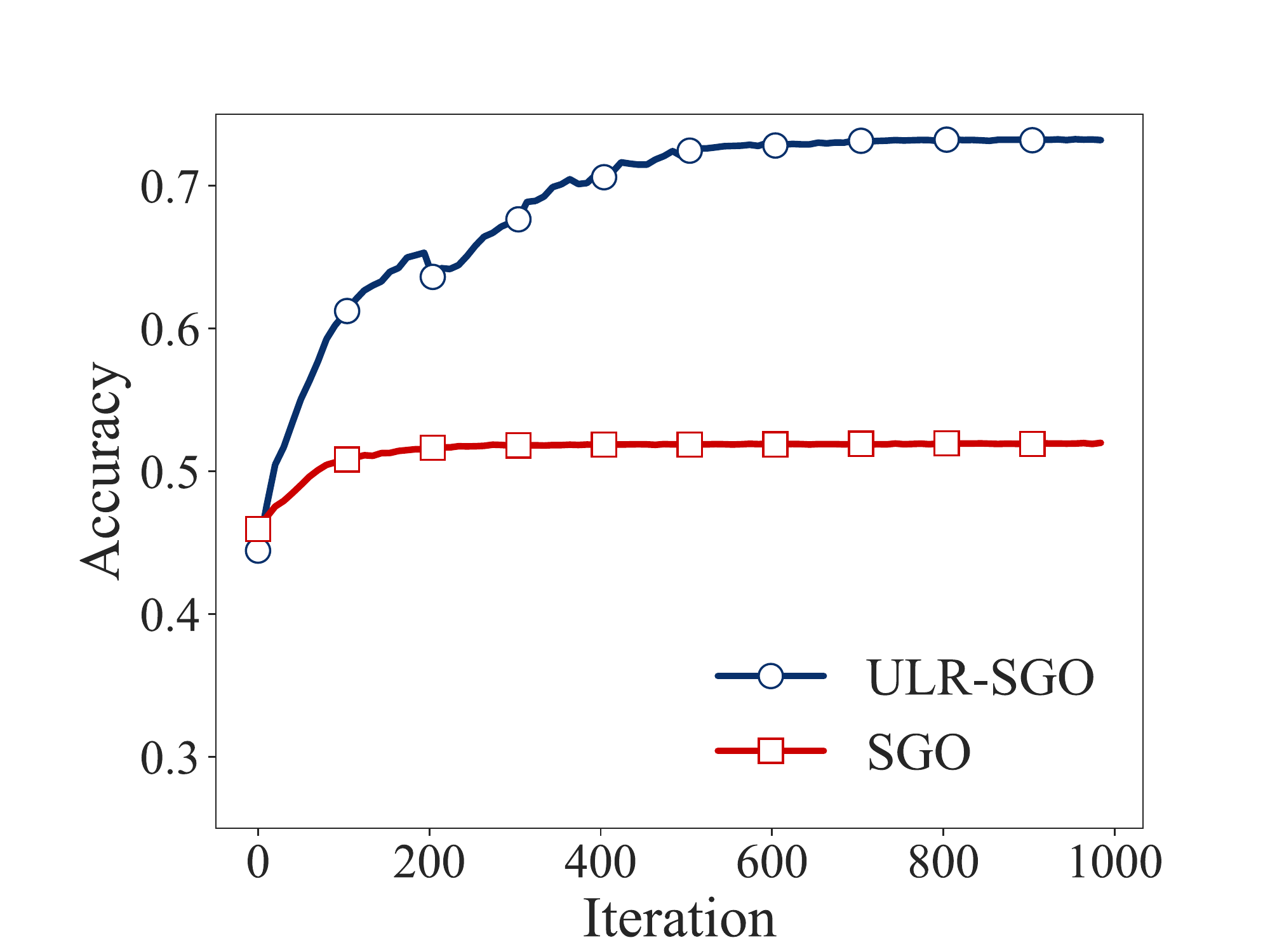}
 \caption{\label{fig:lenet5_128_ulr_sgo} All-CNN-C CKN on CIFAR-10 with 128 filters/layer}
\end{subfigure}\hfill
\caption{\label{fig:sgo_ulr} Performance of CKNs when using stochastic gradient optimization (SGO) vs. our Ultimate Layer Reversal method (ULR-SGO) in terms of accuracy vs. iteration.}
\end{figure*}

Next, we compare training CKNs using stochastic gradient optimization (SGO)  to using our proposed ultimate layer reversal method (ULR-SGO) as detailed in Section~\ref{sec:reversal}. In our SGO implementation we use the version of the optimization in which $\lambda$ is penalty parameter rather than a constraint.

Figure~\ref{fig:sgo_ulr} displays the results of the comparison for the CKN counterparts to LeNet-5 on MNIST and All-CNN-C on CIFAR-10 with 8 and 128 filters/layer. From the plots we can see that ULR-SGO is nearly always better than SGO throughout the iterations. This difference is most pronounced for the experiments in which the accuracy increased the most from the initialization: the LeNet-5 CKN with 8 filters/layer and the All-CNN-C CKN with 128 filters/layer. The final accuracy from the ULR-SGO method after 1000 iterations ranges from being 0.5\% better on the easier task of classifying MNIST digits with the LeNet-5 CKN architectures to 4\% and 40\% better on the harder task of classifying CIFAR-10 images with the All-CNN-C CKN architectures. It is also interesting to note that the ULR-SGO curve is much smoother in the case of LeNet-5 CKN with 8 filters/layer. In addition, in the final case of the All-CNN-C CKN with 128 filters/layer, the SGO method seems to have gotten stuck, whereas this was not a problem for ULR-SGO. The initial drop in performance for the All-CNN-C CKN plot with 8 filters/layer is due  to the method choosing an initial learning rate that was too large. The learning rate was corrected when it was next updated, at 100 iterations, at which point the accuracy proceeds to increase again.

While it is clear that ULR-SGO dominates SGO in terms of performance over the iterations, it is also important to ensure that this is true in terms of time. Figure~\ref{fig:sgo_ulr_time} in Appendix~\ref{app:training_details_results} provides the same plots as Figure~\ref{fig:sgo_ulr}, except that the x-axis is now time. The experiments for the LeNet-5 CKN were performed using an Nvidia Titan Xp GPU while the All-CNN-C CKN experiments were performed using an Nvidia Tesla V100 GPU.
From the plots we can see that the ULR-SGO method still outperforms the SGO method in terms of accuracy vs. time. %

\subsection{CKNs vs. ConvNets}
\label{sec:convnet_ckn_exp}

 \begin{figure*}[t!]
\centering
\begin{subfigure}{.33\textwidth}
  \centering
\includegraphics[width=1.05\linewidth]{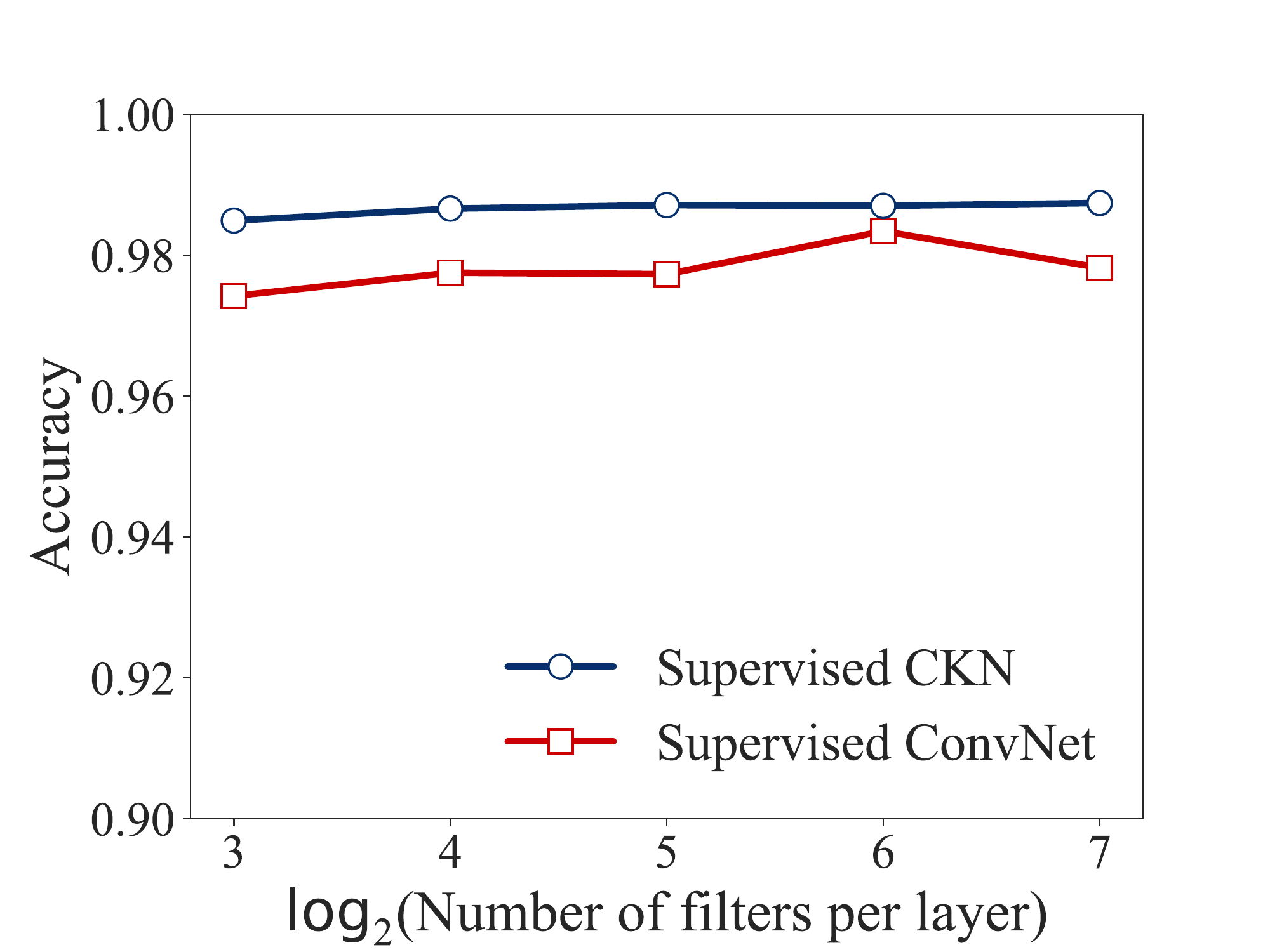}
 \caption{\label{fig:lenet1}LeNet-1 on MNIST}
\end{subfigure}\hfill
\begin{subfigure}{.33\textwidth}
  \centering
\includegraphics[width=1.05\linewidth]{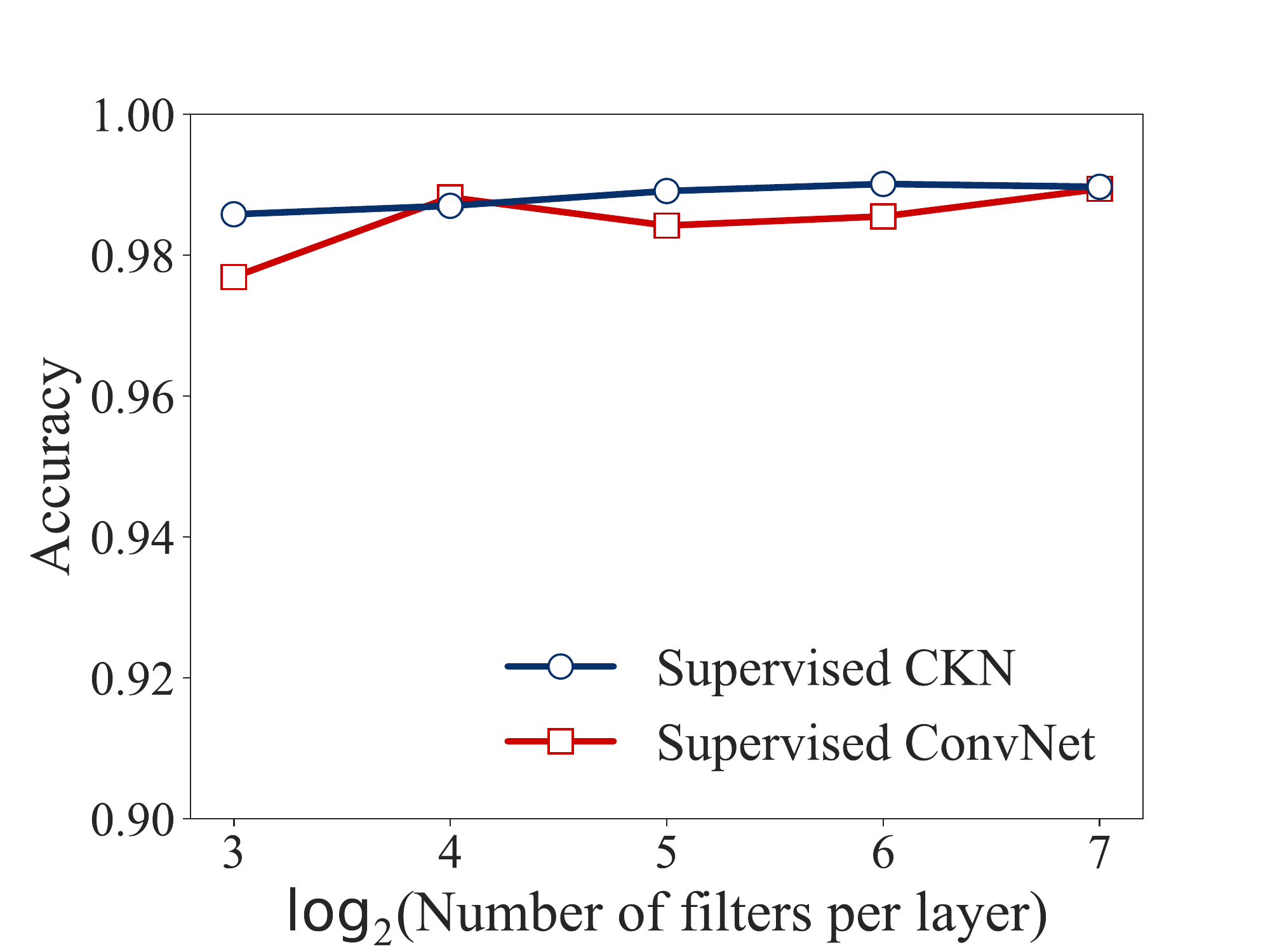}
 \caption{\label{fig:lenet5}LeNet-5 on MNIST}
\end{subfigure}\hfill
\begin{subfigure}{.33\textwidth}
  \centering
\includegraphics[width=1.05\linewidth]{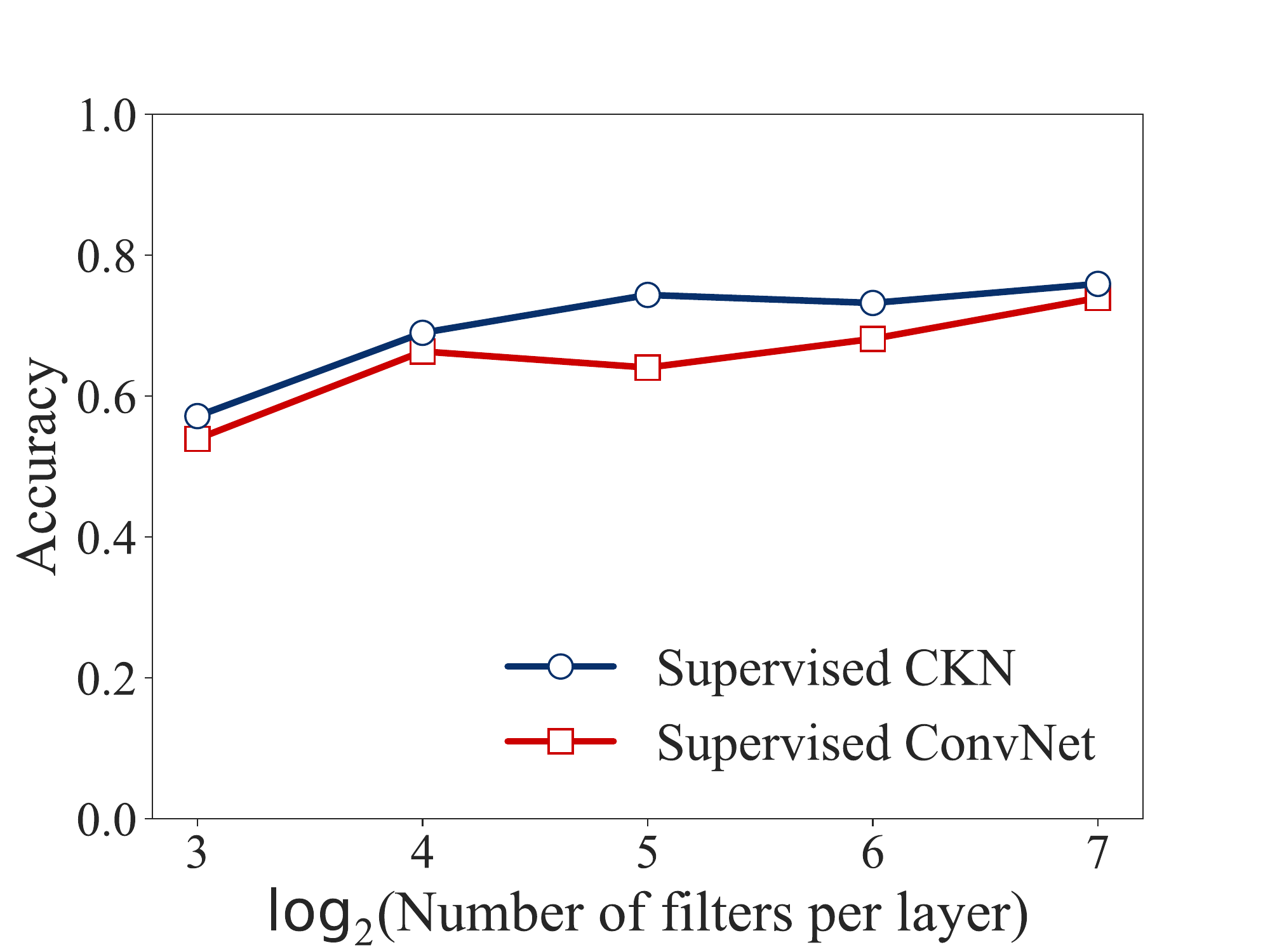}
 \caption{\label{fig:allcnn} All-CNN-C on CIFAR-10}
\end{subfigure}\hfill
\caption{\label{fig:acc_vs_nfilt} Performance of CKNs and ConvNet counterparts when varying number of filters per layer. Note that the y-axes for the LeNets begin at 0.9.}
\label{fig:acc_nfilt}
\end{figure*}

Now we turn to the comparison between CKNs and ConvNets. We perform this comparison for LeNet-1 and LeNet-5 on MNIST and for All-CNN-C on CIFAR-10. Figure~\ref{fig:acc_nfilt} displays the results when we vary the number of filters per layer by powers of two, from 8 to 128.
 Beginning with the LeNets, we see that both the CKN and ConvNet perform well on MNIST over a wide range of the number of filters per layer. The CKN outperforms the ConvNet for almost every number of filters per layer.  At best the performance of the CKN is 1\% better and at worst it is 0.1\% worse. 
The former value is large, given that the accuracy of both the CKNs and the ConvNets exceed 97\%.
The success of the CKNs continues for All-CNN-C on CIFAR-10. For All-CNN-C the CKN outperforms the ConvNet by 3-16\%. From the plot we can see that the CKN performance aligns well with the ConvNet performance toward the endpoints of the range considered. Overall, the results for the LeNets and All-CNN-C suggest that translated CKNs do perform similarly to their ConvNet counterparts. %

Recalling from Section~\ref{sec:supervised_training} that CKNs are initialized in an unsupervised manner, we also compare the performance of unsupervised CKNs to their supervised CKN and ConvNet counterparts. We explore this in Figure~\ref{fig:acc_vs_nfilt_unsup} in Appendix~\ref{app:training_details_results}. For LeNet-1 the unsupervised CKN performs extremely well, achieving at minimum 98\% of the accuracy of the corresponding supervised CKN. The performance is slightly worse for LeNet-5, with the unsupervised CKN achieving 64-97\% of the performance of the supervised CKN with the same number of filters. The relative performance is the worst for All-CNN-C, with the unsupervised performance being 44-59\% of that of the supervised CKN with the same number of filters.  Therefore, the supervised training contributes tremendously to the overall performance of the LeNet-5 CKN with a small number of filters and to the All-CNN-C CKN. These results also suggest that for more complex tasks the unsupervised CKN may require more than 16 times as many filters to achieve comparable performance to the supervised CKN and ConvNet.

%% file: sections/sec6.tex
In this work we provided a systematic study of the translation of a ConvNet to its CKN counterpart. We presented a new stochastic gradient algorithm to train a CKN in a supervised manner. When trained using this method, the CKNs we studied achieved comparable performance to their ConvNet counterparts. As with the training of any deep network, there are a number of design choices we made that could be modified. Each such choice in the ConvNet world has a counterpart in the CKN world. For example, we could perform the initialization of the filters of the ConvNet and CKN using a different method. In addition, we could use additional normalizations in the ConvNet and CKN. We leave the exploration of the effects of these alternatives to future work.

%% file: sections/acknowledgements.tex
\section*{Acknowledgements}
This work was supported by NSF TRIPODS Award CCF-1740551, the program ``Learning in Machines and Brains’' of CIFAR, and faculty research awards.

%% file: sections/appendix0-intro.tex
In this appendix we provide additional details related to the architectures, the training, and the experiments.
Specifically, we begin in Appendix~\ref{app:convnet_desc} by providing a mathematical description of the ConvNets we consider in our experiments. Two of the architectures, LeNet-1 and LeNet-5 \citep{lecun1998a}, originally had incomplete connection schemes. We compare their performance with incomplete vs. complete connection schemes in Appendix~\ref{app:lenet_incomplete}.  We then show how we translated the ConvNet architectures into CKNs in Appendix~\ref{app:ckn_desc}. Following this, we derive the CKN gradient formulas in Appendix~\ref{app:gradient_proofs}. %
We train the overall network by performing stochastic gradient steps on a manifold within our new ultimate layer reversal method. The details of these methods are contained in Appendices~\ref{app:sg_manifolds} and \ref{app:ulr}. Finally, we %
 report additional experimental details and results in Appendix~\ref{app:training_details_results}.

\appendix

%% file: sections/appendix2-convnet_math_details.tex
\section{Mathematical description of ConvNets}
\label{app:convnet_desc}
Convolutional neural network architectures are typically described by their components, including convolutions, non-linearities, and pooling. In this section we provide what we believe to be the first mathematical formulations of several historical architectures. For a broader review of ConvNets and their components, the reader may consult \cite{rawat2017}. Unless otherwise specified, the parameters $\W_i$ discussed below are learned via backpropagation. In contrast to Section~\ref{sec:training} and Appendix~\ref{app:gradient_proofs}, here and in Section~\ref{app:ckn_desc} we describe the features at each layer of the ConvNets and CKNs using tensors for clarity of exposition.

\subsection{LeNets}
We begin with LeNet-1 and LeNet-5 \citep{lecun1995,lecun1998a}, which were among the first modern versions of a ConvNet. These networks used convolutional layers and pooling/subsampling layers. Our description of LeNet-5 here differs slightly from the original paper, as the original paper used an RBF layer as the last layer rather than a fully connected layer. LeNet-1 and LeNet-5 differ from typical modern ConvNets because the pooling is average pooling, the modules are not in the typical order, and the second convolutional layer uses an incomplete connection scheme.

The activation functions used throughout the LeNet architectures are scaled tanh functions of the form $f(x) = 1.7159 \tanh(2/3 x)$. This scaling is such that $f(1)\approx 1$ and $f(-1)\approx-1$. It was argued \citep{lecun1989,lecun1998a,lecun1998b} that these choices speed up convergence for several reasons. First, for standardized inputs $f$ will output values with variance approximately equal to one. In addition, the second derivatives of $f$ are largest in absolute value at $\pm1$. Assuming the target outputs are $\pm1$ this implies that near the optimum when the predictions are close to $\pm1$ the gradients change faster. It was noted by \citet{lecun1989} that this parameterization is for convenience and does not necessarily improve performance.

\paragraph{LeNet-1}
We begin by detailing LeNet-1. Let $\F_0\in\mbr^{1\times28\times28}$ be the initial representation of the image. Let $\W_1\in\mbr^{4\times1\times5\times5}$ and $\b_1\in\mbr^4$ be learnable parameters. The first convolutional layer convolves $\F_0$ and $\W_1$ and adds a bias term, resulting in $\F_1\in\mbr^{4\times24\times24}$:
\begin{align*}
\F_1(k, \cdot, \cdot) %
&= \F_0(1, \cdot, \cdot)\star \W_1(k, 1, \cdot, \cdot)+ \b_1(k)
\end{align*}
for $k=1,\dots, 4$ where $\star$ denotes the convolution operation.

Next, the second layer consists of average pooling with learnable parameters and subsampling, followed by the application of a nonlinearity. Let $\w_2, \b_2\in\mathbb{R}^4$ and let $e_{i}\in\mbr^{23}$ be a vector with entry 1 in element $i$ and 0 elsewhere. Define $\emat_2=(e_1, e_3, \dots, e_{23})$ and $\mathbbm{1}_{d\times d}$ to be a matrix of ones of size $d\times d$. Then the second layer computes $\F_2\in\mbr^{4\times12\times12}$ given by
\begin{align*}
\F_2(k, \cdot, \cdot) = 1.7159 \tanh\left(\frac{2}{3} \emat_2^T\left[\left(\F_1(k, \cdot, \cdot) \star \frac{1}{4}\mathbbm{1}_{2\times 2}\right)\w_2(k) + \b_2(k)\mathbbm{1}_{23\times 23} \right]\emat_2 \right),
\end{align*}
for $k=1,\dots, 4$ where $\tanh$ is understood to be applied element-wise.

The third layer is a convolutional layer with an incomplete connection scheme between the filters at the previous layer and at this layer. Define the matrix
\begin{align*}
\C_3 = \begin{bmatrix}
1 & 1 & 1 & 0 & 1 & 1 & 0 & 0 & 0 & 0 & 0 & 0 \\
0 & 1 & 1 & 1 & 1 & 1 & 0 & 0 & 0 & 0 & 0 & 0 \\
0 & 0 & 0 & 0 & 0 & 0 & 1 & 1 & 1 & 0 & 1 & 1\\
0 & 0 & 0 & 0 & 0 & 0 & 0 & 1 & 1 & 1 & 1 & 1\\
\end{bmatrix}^T
\end{align*}
to be the connection scheme.  Moreover, suppose $\W_3\in\mbr^{12\times 4\times 5\times 5}$ and $\b_3\in\mbr^{12}$. The output of layer 3 is then given by $\F_3\in\mbr^{12\times8\times8}$, with
\begin{align*}
\F_3(k, \cdot, \cdot)
&= \sum_{z=1}^4 \left\{\left[\F_2(z, \cdot, \cdot)\star \W_3(k, z, \cdot, \cdot)+ \b_3(k)\mathbbm{1}_{8\times 8}\right]\right\}\C_3(k,z).
\end{align*}
for $k=1,\dots, 12$. 

The fourth layer is analogous to the previous subsampling layer. Let $\w_4, \b_4\in\mbr^{12}$ and $\emat_4=(e_1, e_3, \dots, e_{7})$. Here $e_i\in\mbr^{7}$ with a one in element $i$ and 0 elsewhere for all $i=1,\dots, 7$. Then we obtain $\F_4\in\mbr^{12\times4\times4}$ with
\begin{align*}
\F_4(k, \cdot, \cdot) = 1.7159 \tanh\left(\frac{2}{3} \emat_4^T\left[\left(\F_3(k, \cdot, \cdot) \star \frac{1}{4}\mathbbm{1}_{2\times 2}\right)\w_4(k) + \b_4(k)\mathbbm{1}_{7\times 7} \right]\emat_4 \right),
\end{align*}
for $k=1,\dots, 12$ where $\tanh$ is again understood to be applied element-wise.

Finally, the last layer is a fully connected layer. Let $\W_5\in\mbr^{10\times12\times4\times4}$ and $\b_5\in\mbr^{10}$. Then the output is given by 
$\F_5\in\mbr^{10}$ with
\begin{align*}
\F_5(k) = \sum_{z=1}^{12} \left[\F_4(z, \cdot, \cdot)\star \W_5(k, z, \cdot, \cdot)+ \b_5(k)\right]
\end{align*}
for $k=1,\dots, 10$.

\paragraph{LeNet-5}
Now we describe LeNet-5.
 Let $\F_0\in\mbr^{1\times32\times32}$ be the initial representation of the image. Let $\W_1\in\mbr^{6\times1\times5\times5}$ and $\b_1\in\mbr^6$ be learnable parameters. The first convolutional layer convolves $\F_0$ and $\W_1$ and adds a bias term, resulting in $\F_1\in\mbr^{6\times28\times28}$:
\begin{align*}
\F_1(k, \cdot, \cdot) %
&= \F_0(1, \cdot, \cdot)\star \W_1(k, 1, \cdot, \cdot)+ \b_1(k)
\end{align*}
for $k=1,\dots, 6$.

Next, the second layer consists of average pooling with learnable parameters and subsampling, followed by the application of a nonlinearity. Let $\w_2, \b_2\in\mathbb{R}^6$ and let $e_{i}\in\mbr^{27}$ be a vector with entry 1 in element $i$ and 0 elsewhere. Define $\emat_2=(e_1, e_3, \dots, e_{27})$ and $\mathbbm{1}_{d\times d}$ to be a matrix of ones of size $d\times d$. Then the second layer computes $\F_2\in\mbr^{6\times14\times14}$ given by
\begin{align*}
\F_2(k, \cdot, \cdot) = 1.7159 \tanh\left(\frac{2}{3} \emat_2^T\left[\left(\F_1(k, \cdot, \cdot) \star \frac{1}{4}\mathbbm{1}_{2\times 2}\right)\w_2(k) + \b_2(k)\mathbbm{1}_{27\times 27} \right]\emat_2 \right),
\end{align*}
for $k=1,\dots, 6$.

The third layer is a convolutional layer with an incomplete connection scheme between the filters at the previous layer and at this layer. Define the matrix
\begin{align*}
\C_3 = \begin{bmatrix}
1 & 0 & 0 & 0 & 1 & 1 & 1 & 0 & 0 & 1 & 1 & 1 & 1 & 0 & 1 & 1\\
1 & 1 & 0 & 0 & 0 & 1 & 1 & 1 & 0 & 0 & 1 & 1 & 1 & 1 & 0 & 1\\
1 & 1 & 1 & 0 & 0 & 0 & 1 & 1 & 1 & 0 & 0 & 1 & 0 & 1 & 1 & 1\\
0 & 1 & 1 & 1 & 0 & 0 & 1 & 1 & 1 & 1 & 0 & 0 & 1 & 0 & 1 & 1\\
0 & 0 & 1 & 1 & 1 & 0 & 0 & 1 & 1 & 1 & 1 & 0 & 1 & 1 & 0 & 1\\
0 & 0 & 0 & 1 & 1 & 1 & 0 & 0 & 1 & 1 & 1 & 1 & 0 & 1 & 1 & 1\\
\end{bmatrix}^T
\end{align*}
to be the connection scheme.  Moreover, suppose $\W_3\in\mbr^{16\times 6\times 5\times 5}$ and $\b_3\in\mbr^{16}$. The output of layer 3 is then given by $\F_3\in\mbr^{16\times10\times10}$, with
\begin{align*}
\F_3(k, \cdot, \cdot)
&= \sum_{z=1}^6 \left\{\left[\F_2(z, \cdot, \cdot)\star \W_3(k, z, \cdot, \cdot)+ \b_3(k)\mathbbm{1}_{10\times 10}\right]\right\}\C_3(k,z).
\end{align*}
for $k=1,\dots, 16$. 

The fourth layer is analogous to the previous subsampling layer. Let $\w_4, \b_4\in\mbr^{16}$ and $\emat_4=(e_1, e_3, \dots, e_{9})$. Here $e_i\in\mbr^{9}$ with a 1 in element $i$ and 0 elsewhere for all $i=1,\dots, 9$. Then we obtain $\F_4\in\mbr^{16\times5\times5}$ with
\begin{align*}
\F_4(k, \cdot, \cdot) = 1.7159 \tanh\left(\frac{2}{3} \emat_4^T\left[\left(\F_3(k, \cdot, \cdot) \star \frac{1}{4}\mathbbm{1}_{2\times 2}\right)\w_4(k) + \b_4(k)\mathbbm{1}_{9\times 9} \right]\emat_4 \right),
\end{align*}
for $k=1,\dots, 16$.

The fifth layer is a fully connected layer. Let $\W_5\in\mbr^{120\times16\times5\times5}$ and $\b_5\in\mbr^{120}$. Then the output is given by 
$\F_5\in\mbr^{120}$ with
\begin{align*}
\F_5(k) = 1.7159\tanh\left(\frac{2}{3}\left(\sum_{z=1}^{16} \left[\F_4(z, \cdot, \cdot)\star \W_5(k, z, \cdot, \cdot)+ \b_5(k)\right)\right]\right)
\end{align*}
for $k=1,\dots, 120$.

The sixth layer is also a fully connected layer. Let $\W_6\in\mbr^{84\times120}$ and $\b_6\in\mbr^{84}$. Then the output is given by 
$\F_6\in\mbr^{84}$ with
\begin{align*}
\F_6 = 1.7159\tanh\left(\frac{2}{3}\W_6\F_5+ \b_6\right).
\end{align*}

Finally, the output layer is also a fully connected layer. Let $\W_7\in\mbr^{10\times84}$ and $\b_7\in\mbr^{10}$. Then the output is given by 
$\F_7\in\mbr^{10}$ with
\begin{align*}
\F_7 = \W_7\F_6+ \b_7.
\end{align*}

\subsection{All-CNN-C}
\citet{springenberg2014} were the first to make the claim that pooling is unnecessary. Below we describe the architecture of the All-CNN-C model they presented, which consists of convolutions and ReLUs.  

The input to the model is an image $\F_0\in\mbr^{3\times32\times32}$. Let $\W_1\in\mbr^{96\times3\times3\times3}$ and $\b_1\in\mbr^{96}$ and define $\Z_1\in\mbr^{32\times34}$ such that $(\Z_1)_{ij}=1$ if $j=i+1$ for $i=1,\dots, 32$ and is 0 otherwise. The first layer zero pads $\F_0$ and then convolves the result with $\W_1$ with a $1\times1$ stride. It then adds a bias term and applies a ReLU activation, resulting in $\F_1\in\mbr^{96\times32\times32}$ given by
\begin{align*}
\F_1(k, \cdot, \cdot)
&=\max\left(\sum_{z=1}^3 \left[(\Z_1^T\F_0(z, \cdot, \cdot)\Z_1)\star \W_1(k, z, \cdot, \cdot)+ \b_1(k)\mathbbm{1}_{32\times 32}\right], 0\right), 
\end{align*}
for $k=1,\dots, 96$. Here the $\max$ is understood to be applied element-wise.

The second layer is of the same form as the first layer. Let $\W_2\in\mbr^{96\times96\times3\times3}$ and $\b_2\in\mbr^{96}$ and define $\Z_2\in\mbr^{32\times34}$ such that $(\Z_2)_{ij}=1$ if $j=i+1$ for $i=1,\dots, 32$ and is 0 otherwise. The second layer outputs $\F_2\in\mbr^{96\times32\times32}$ given by
\begin{align*}
\F_2(k, \cdot, \cdot)
&=\max\left(\sum_{z=1}^{96} \left[(\Z_2^T\F_1(z, \cdot, \cdot)\Z_2)\star \W_2(k, z, \cdot, \cdot)+ \b_2(k)\mathbbm{1}_{32\times 32}\right], 0\right), 
\end{align*}
for $k=1,\dots, 96$.

The third layer is a convolutional layer with a $2\times2$ stride. This layer acts as a replacement for a max pooling layer. Let $\W_3\in\mbr^{96\times96\times3\times3}$ and $\b_3\in\mbr^{96}$ and define $\emat_{3}=(e_1, e_3, e_5, \dots, e_{29})$ where $e_i\in\mbr^{30}$ is a vector with 1 in element $i$ and 0 elsewhere. The third layer outputs $\F_3\in\mbr^{96\times15\times15}$ given by
\begin{align*}
\F_3(k, \cdot, \cdot)
&=\max\left(\emat_{3}^T\left\{\sum_{z=1}^{96} \left[\F_2(z, \cdot, \cdot)\star \W_3(k, z, \cdot, \cdot)+ \b_3(k)\mathbbm{1}_{30\times 30}\right]\right\}\emat_{3}, 0\right), 
\end{align*}
for $k=1,\dots, 96$.

The fourth layer returns to being a convolutional layer with a $1\times1$ stride, but has 192 filters. Let $\W_4\in\mbr^{192\times96\times3\times3}$ and $\b_4\in\mbr^{192}$ and define $\Z_4\in\mbr^{15\times17}$ such that $(\Z_4)_{ij}=1$ if $j=i+1$ for $i=1,\dots, 15$ and is 0 otherwise. The fourth layer outputs $\F_4\in\mbr^{192\times15\times15}$ given by
\begin{align*}
\F_4(k, \cdot, \cdot)
&=\max\left(\sum_{z=1}^{96} \left[(\Z_4^T\F_3(z, \cdot, \cdot)\Z_4)\star \W_4(k, z, \cdot, \cdot)+ \b_4(k)\mathbbm{1}_{15\times 15}\right], 0\right), 
\end{align*}
for $k=1,\dots, 192$.

The fifth layer is similar to the fourth layer. Let $\W_5\in\mbr^{192\times192\times3\times3}$ and $\b_5\in\mbr^{192}$ and define $\Z_5\in\mbr^{15\times17}$ such that $(\Z_5)_{ij}=1$ if $j=i+1$ for $i=1,\dots, 15$ and is 0 otherwise. The fifth layer outputs $\F_5\in\mbr^{192\times15\times15}$ given by
\begin{align*}
\F_5(k, \cdot, \cdot)
&=\max\left(\sum_{z=1}^{192} \left[(\Z_5^T\F_4(z, \cdot, \cdot)\Z_5)\star \W_5(k, z, \cdot, \cdot)+ \b_5(k)\mathbbm{1}_{15\times 15}\right], 0\right), 
\end{align*}
for $k=1,\dots, 192$.

The sixth layer is similar to the third layer, except it has 192 filters. Let $\W_6\in\mbr^{192\times192\times3\times3}$ and $\b_6\in\mbr^{192}$ and define $\emat_{6}=(e_1, e_3, e_5, \dots, e_{13})$ where $e_i\in\mbr^{13}$ is a vector with 1 in element $i$ and 0 elsewhere. The sixth layer outputs $\F_6\in\mbr^{192\times7\times7}$ given by
\begin{align*}
\F_6(k, \cdot, \cdot)
&=\max\left(\emat_{6}^T\left\{\sum_{z=1}^{192} \left[\F_5(z, \cdot, \cdot)\star \W_6(k, z, \cdot, \cdot)+ \b_6(k)\mathbbm{1}_{13\times 13}\right]\right\}\emat_{6}, 0\right), 
\end{align*}
for $k=1,\dots, 192$.

The seventh layer is once again like the fifth layer. Let $\W_7\in\mbr^{192\times192\times3\times3}$ and $\b_7\in\mbr^{192}$ and define $\Z_7\in\mbr^{7\times9}$ such that $(\Z_7)_{ij}=1$ if $j=i+1$ for $i=1,\dots, 7$ and is 0 otherwise. The seventh layer outputs $\F_7\in\mbr^{192\times7\times7}$ given by
\begin{align*}
\F_7(k, \cdot, \cdot)
&=\max\left(\sum_{z=1}^{192} \left[(\Z_7^T\F_6(z, \cdot, \cdot)\Z_7)\star \W_7(k, z, \cdot, \cdot)+ \b_7(k)\mathbbm{1}_{7\times 7}\right], 0\right), 
\end{align*}
for $k=1,\dots, 192$.

The eighth layer has $1\times1$ convolutions rather than $3\times3$ convolutions. Let $\W_8\in\mbr^{192\times192}$ and $\b_8\in\mbr^{192}$. The eighth layer outputs $\F_8\in\mbr^{192\times7\times7}$ given by
\begin{align*}
\F_8(k, \cdot, \cdot)
&=\max\left(\sum_{z=1}^{192} \left[\F_7(z, \cdot, \cdot)\star \W_8(k, z)+ \b_8(k)\mathbbm{1}_{7\times 7}\right], 0\right),
\end{align*}
for $k=1,\dots, 192$.

The ninth layer again has $1\times1$ convolutions but has only ten filters. Let $\W_9\in\mbr^{10\times192}$ and $\b_9\in\mbr^{10}$. The ninth layer outputs $\F_9\in\mbr^{10\times7\times7}$ given by
\begin{align*}
\F_9(k, \cdot, \cdot)
&=\max\left(\sum_{z=1}^{192} \left[\F_8(z, \cdot, \cdot)\star \W_9(k, z)+ \b_9(k)\mathbbm{1}_{7\times 7}\right], 0\right),
\end{align*}
for $k=1,\dots, 10$.

Next, the tenth layer is a global average pooling layer. In this layer, all of the pixels in a given feature map are averaged. The result is given by
\begin{align*}
\F_{10}(k) = \F_9(k, \cdot, \cdot) \star \frac{1}{49}\mathbbm{1}_{7\times 7}.
\end{align*} 

In the original work, a softmax function is applied to the last layer. In this work, however, we will include a fully connected layer between the tenth layer and the softmax function so that the analogous CKN will have trainable parameters. Defining $\W_{11}\in\mbr^{10\times10}$ and $\b_{11}\in\mbr^{10}$, the output of our eleventh layer is thus
\begin{align*}
\F_{11} = \W_{10}\F_{10}+\b_{10}.
\end{align*}

%% file: sections/appendix3-lenet_incomplete_scheme.tex
\section{LeNet incomplete connection scheme}
\label{app:lenet_incomplete}
The LeNet architectures included incomplete connection schemes at the C3 (second convolutional) layer. With these schemes, feature maps in the C3 convolutional layers were only connected to certain feature maps at the previous layer.  This was primarily used for computational reasons, but was also argued to break symmetry in the network \citep{lecun1998a}. 

Training the LeNet-1 and LeNet-5 models for 100 epochs, we found that architectures with the complete connection schemes do just as well as or better than the architectures with the incomplete connection schemes. Table~\ref{tab:lenet_errors} reports the final accuracies on the test set of MNIST. Figure~\ref{fig:lenet_c3_comparison} demonstrates that the learning proceeds similarly regardless of whether a complete or incomplete connection scheme is used.
Note that our LeNet-1 architecture outperforms the original LeNet-1 result (98.3\%), which was on $16\times16$ images with padding. Moreover, with 1000 epochs our implementation of LeNet-5 with the incomplete connection scheme achieved an accuracy of 99.00\%, which is within the reported uncertainty of 0.1\% of the result in \cite{lecun1998a}. As there is little benefit to using the incomplete connection schemes, we use the models with complete connection schemes in this paper.

 \begin{table}[t]
 \centering
  \caption{\label{tab:lenet_errors} Accuracy of the LeNet ConvNets with incomplete and complete connection schemes at the C3 layers. The ConvNets were trained for 100 epochs.}
\begin{tabular}{c|ccc}
 & Incomplete scheme  & Complete scheme \\
\hline
LeNet-1 & 0.9854 & 0.9855  \\
LeNet-5 & 0.9868 & 0.9912  \\
\bottomrule
\end{tabular}
\end{table}

 \begin{figure}[t!]
\centering
\begin{subfigure}{.48\textwidth}
  \centering
\includegraphics[width=1\linewidth]{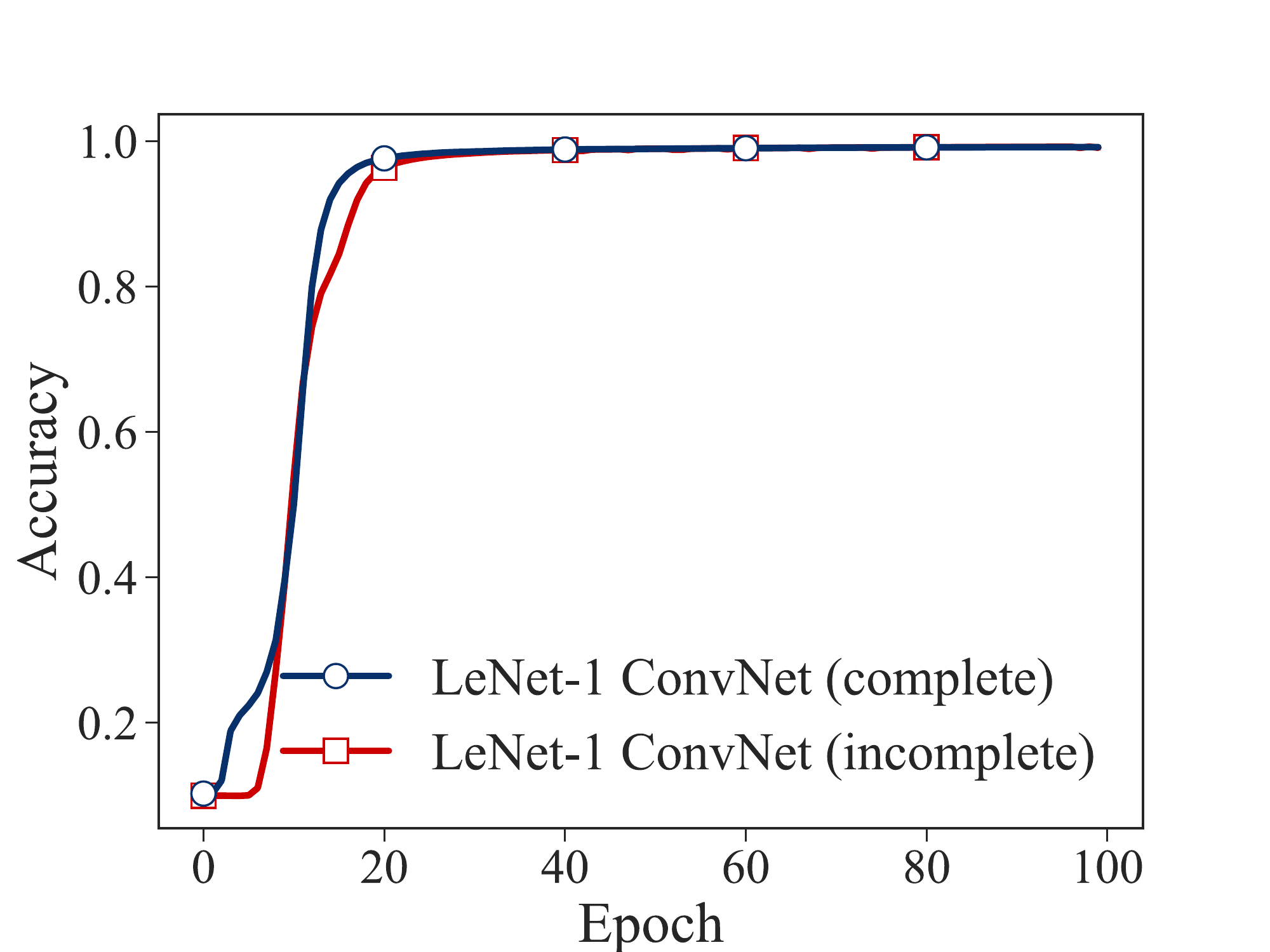}
\caption{LeNet-1}
\end{subfigure}\hfill
\begin{subfigure}{.48\textwidth}
  \centering
\includegraphics[width=1\linewidth]{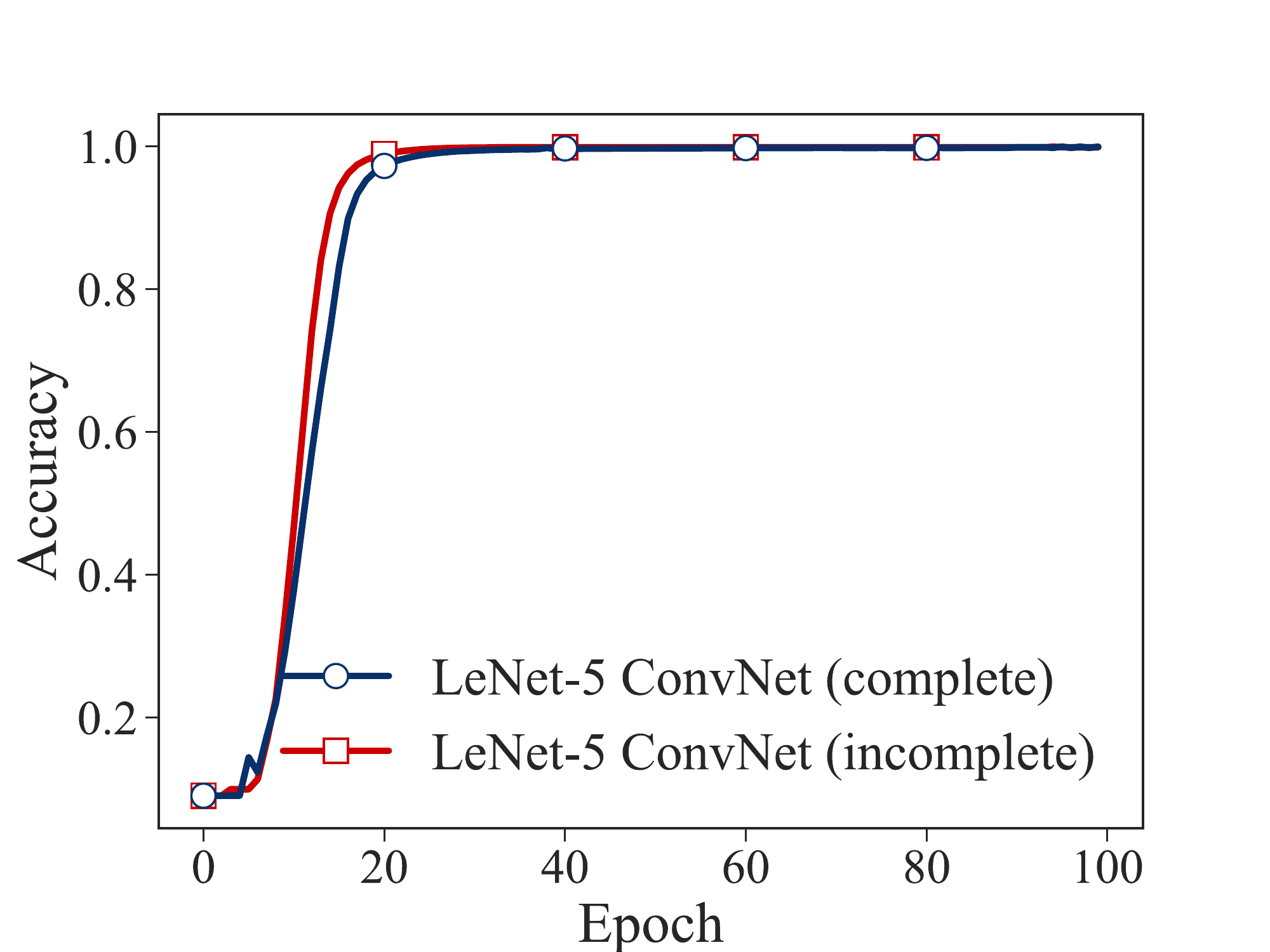}
\caption{LeNet-5}
\end{subfigure}
\caption{\label{fig:lenet_c3_comparison} Accuracy for complete and incomplete connections schemes at the C3 layer of the LeNet architectures on the training set of MNIST as a function of the number of epochs.}
\label{fig:lenet_acc_iter}
\end{figure}

%% file: sections/appendix4-ckn_translations.tex
\section{CKN counterparts to ConvNets}
\label{app:ckn_desc}
In this section we describe in detail the CKN counterparts to the ConvNets from Section~\ref{app:convnet_desc}. Unless otherwise specified, the filters $\W_i$ discussed below are initially chosen in our experiments via spherical $k$-means and then later trained using backpropagation. For clarity of the exposition we set the regularization parameter $\epsilon=0$ at each layer. Pictorial representations of the architectures may be found following their descriptions.

\subsection{LeNets} 
As noted in Section~\ref{sec:translation}, a linear activation function corresponds to a linear kernel and a convolution followed by the tanh activation is similar to the arc-cosine kernel of order zero. Moreover, recall from Section~\ref{sec:translation} that we may approximate kernels by projecting onto a subspace. The dimension of this projection corresponds to the number of filters.  We use these facts to define the CKN counterparts to LeNet-1 and LeNet-5. 

As noted in Section~\ref{app:lenet_incomplete}, the incomplete connection schemes in the LeNets were included for computational considerations. As the corresponding networks with complete connection schemes perform similarly or better, we translate the LeNets with a complete connection scheme.

\paragraph{LeNet-1}
Let $\F_0\in\mbr^{1\times28\times28}$ denote the initial representation of an image. The first layer is the counterpart to a convolutional layer and consists of applying a linear kernel and projecting onto a subspace. Let $\W_1\in\mbr^{4\times1\times5\times5}$. %
For $k=1,\dots, 4$, let
\begin{equation*}
\F'_1(k, \cdot, \cdot) = F_0(1, \cdot, \cdot)\star \W_1(k, 1, \cdot, \cdot),
\end{equation*}
where $\star$ denotes the convolution operation.
Then the output of the first layer is given by $\F_1\in\mbr^{4\times24\times24}$ with 
\begin{equation*}
\F_1(\cdot, i, j) = \left(\sum_{m,n=1}^5 \W_1(\cdot, 1, m, n)\W_1(\cdot, 1, m, n)^T\right)^{-1/2}\F'_1(\cdot, i, j)
\end{equation*}
for $i,j=1,\dots, 24$.

Next, the second layer in the ConvNet performs average pooling and subsampling with learnable weights and then applies a pointwise nonlinearity ($\tanh$). The corresponding CKN pools and subsamples and then applies the feature map of an arc-cosine kernel on $1\times1$ patches. Define $\emat_{2}=(e_1, e_3, e_5, \dots, e_{23})$ where $e_i\in\mbr^{23}$ is a vector with 1 in element $i$ and 0 elsewhere.  The pooling and subsampling result in $\F'_2\in\mbr^{4\times12\times12}$ given by
\begin{equation*}
\F'_2(k, \cdot, \cdot) = \emat_{2}^T\left(\F_1(k, \cdot, \cdot)\star\frac{1}{4}\mathbbm{1}_{2\times2}\right)\emat_{2}
\end{equation*}
for $k=1,\dots, 4$.
Next, let $\W_2\in\mbr^{4\times4}$ be the identity matrix and let $\k_2:\mbr^4\times\mbr^4\to\mbr$ be the arc-cosine kernel of order zero. 
The output of the second layer is then $\F_2\in\mbr^{4\times12\times12}$ given by
\begin{equation*}
\F_2(\cdot, i, j) = \left(\left[\k_2\left(\W_2(m, \cdot), \W_2(n, \cdot)\right)\right]_{m,n=1}^{4}\right)^{-1/2}\left[\k_2\left(\F'_2(\cdot, i, j), \W_2(m, \cdot)\right)\right]_{m=1}^4
\end{equation*}
for $i,j=1,\dots, 12$.

The third layer in LeNet-1 is again a convolutional layer. Here we use a complete connection scheme since for the ConvNet we found that empirically a complete connection scheme outperforms an incomplete connection scheme. Therefore, this layer again consists of applying a linear kernel and projecting onto a subspace. 
Let $\W_3\in\mbr^{12\times4\times5\times5}$. 
For $k=1,\dots, 12$, let
\begin{equation*}
\F'_3(k, \cdot, \cdot) = \sum_{z=1}^4  \F_2(z, \cdot, \cdot)\star \W_3(k, z, \cdot, \cdot).
\end{equation*}
Then the output of the third layer is given by $\F_3\in\mbr^{12\times8\times8}$ with 
\begin{equation*}
\F_3(\cdot, i, j) = \left(\sum_{m=1}^4\sum_{n,p=1}^5 \W_3(\cdot, m, n, p)\W_3(\cdot, m, n, p)^T\right)^{-1/2}\F'_3(\cdot, i, j).
\end{equation*}

The fourth layer is similar to the second layer. The CKN pools and subsamples and then applies an arc-cosine kernel on $1\times1$ patches. Define $\emat_{4}=(e_1, e_3, e_5,  e_{7})$ where $e_i\in\mbr^{7}$ is a vector with 1 in element $i$ and 0 elsewhere.  The pooling and subsampling result in $\F'_4\in\mbr^{12\times4\times4}$ given by
\begin{equation*}
\F'_4(k, \cdot, \cdot) = \emat_{4}^T\left(\F_3(k, \cdot, \cdot)\star\frac{1}{4}\mathbbm{1}_{2\times2}\right)\emat_{4}
\end{equation*}
for $k=1,\dots, 4$.
Next, let $\W_4\in\mbr^{12\times12}$ be the identity matrix and let $\k_4:\mbr^{12}\times\mbr^{12}\to\mbr$ be the arc-cosine kernel of order zero. 
The output of the fourth layer is then $\F_4\in\mbr^{12\times4\times4}$ given by
\begin{equation*}
\F_4(\cdot, i, j) = \left(\left[\k_4\left(\W_4(m, \cdot), \W_4(n, \cdot)\right)\right]_{m,n=1}^{12}\right)^{-1/2}\left[\k_4\left(\F'_3(\cdot, i, j), \W_4(m, \cdot)\right)\right]_{m=1}^{12}
\end{equation*}
for $i,j=1,\dots, 4$.
The output from this layer is the set of features provided to a classifier.

\begin{figure}
\vspace*{-3cm}
\hspace*{-2.5cm}
\begin{center}
\includegraphics[scale=1.0,trim={1cm 5cm 0cm 1cm},clip]{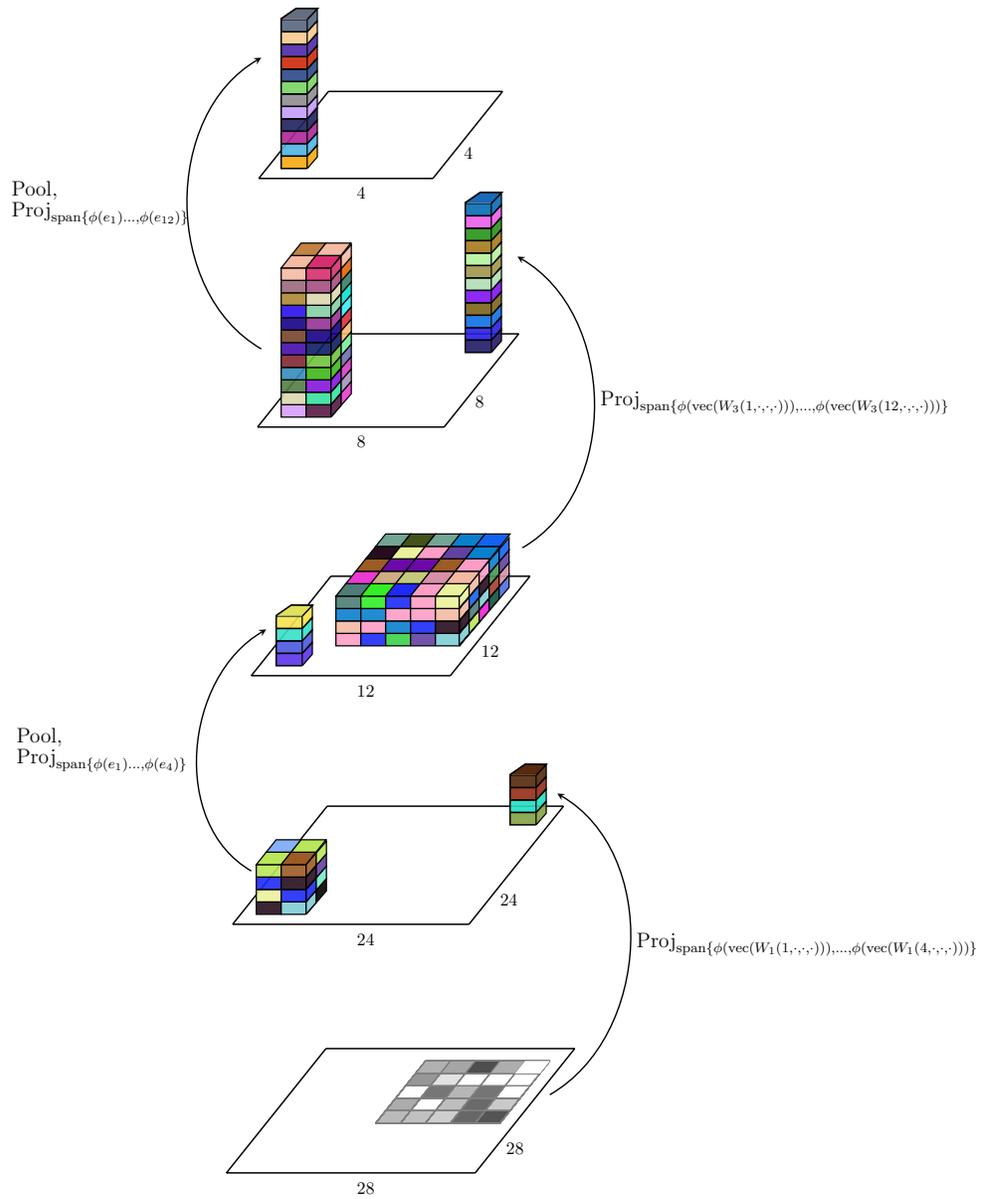}
\end{center}
\caption{LeNet-1 architecture. The dimensions of the stacks of blocks are the dimensions of the filters at each layer.  The height of each stack denotes the number of filters. The arrows indicate how a block gets transformed into the block at the next layer. The numbers on the sides of the parallelograms indicate the spatial dimensions of the feature representations at each layer. }%
\end{figure}

\paragraph{LeNet-5}
The CKN counterpart of LeNet-5 is similar to that of LeNet-1.
Let $\F_0\in\mbr^{1\times32\times32}$ denote the initial representation of an image. The first layer is the counterpart to a convolutional layer and consists of applying a linear kernel and projecting onto a subspace. Let $\W_1\in\mbr^{6\times1\times5\times5}$. %
For $k=1,\dots, 6$, let
\begin{equation*}
\F'_1(k, \cdot, \cdot) = F_0(1, \cdot, \cdot)\star \W_1(k, 1, \cdot, \cdot).
\end{equation*}
Then the output of the first layer is given by $\F_1\in\mbr^{6\times28\times28}$ with 
\begin{equation*}
\F_1(\cdot, i, j) = \left(\sum_{m,n=1}^5 \W_1(\cdot, 1, m, n)\W_1(\cdot, 1, m, n)^T\right)^{-1/2}\F'_1(\cdot, i, j)
\end{equation*}
for $i,j=1,\dots, 28$.

Next, the second layer in the ConvNet performs average pooling and subsampling with learnable weights and then applies a pointwise nonlinearity ($\tanh$). The corresponding CKN pools and subsamples and then applies an arc-cosine kernel on $1\times1$ patches. Define $\emat_{2}=(e_1, e_3, e_5, \dots, e_{27})$ where $e_i\in\mbr^{27}$ is a vector with 1 in element $i$ and 0 elsewhere.  The pooling and subsampling result in $\F'_2\in\mbr^{6\times14\times14}$ given by
\begin{equation*}
\F'_2(k, \cdot, \cdot) = \emat_{2}^T\left(\F_1(k, \cdot, \cdot)\star\frac{1}{4}\mathbbm{1}_{2\times2}\right)\emat_{2}
\end{equation*}
for $k=1,\dots, 6$.
Next, let $W_2\in\mbr^{6\times6}$ be the identity matrix and let $\k_2:\mbr^6\times\mbr^6\to\mbr$ be the arc-cosine kernel of order zero. 
The output of the second layer is then $\F_2\in\mbr^{6\times14\times14}$ given by
\begin{equation*}
\F_2(\cdot, i, j) = \left(\left[\k_2\left(\W_2(m, \cdot), \W_2(n, \cdot)\right)\right]_{m,n=1}^{6}\right)^{-1/2}\left[\k_2\left(\F'_2(\cdot, i, j), W_2(m, \cdot)\right)\right]_{m=1}^6
\end{equation*}
for $i,j=1,\dots, 14$.

The third layer in LeNet-5 is again a convolutional layer. Here we use a complete connection scheme since for the ConvNet we found that empirically a complete connection scheme outperforms an incomplete connection scheme. Therefore, this layer again consists of applying a linear kernel and projecting onto a subspace. 
For $k=1,\dots, 16$, let
\begin{equation*}
\F'_3(k, \cdot, \cdot) = \sum_{z=1}^{6}  \F_2(z, \cdot, \cdot)\star \W_3(k, z, \cdot, \cdot).
\end{equation*}
Then the output of the third layer is given by $\F_3\in\mbr^{16\times10\times10}$ with 
\begin{equation*}
\F_3(\cdot, i, j) = \left(\sum_{m=1}^6\sum_{n,p=1}^5 \W_3(\cdot, m, n, p)\W_3(\cdot, m, n, p)^T\right)^{-1/2}\F'_3(\cdot, i, j).
\end{equation*}

The fourth layer is similar to the second layer. The CKN pools and subsamples and then applies an arc-cosine kernel on $1\times1$ patches. Define $\emat_{4}=(e_1, e_3, e_5,  e_{7})$ where $e_i\in\mbr^{7}$ is a vector with 1 in element $i$ and 0 elsewhere.  The pooling and subsampling result in $\F'_4\in\mbr^{16\times5\times5}$ given by
\begin{equation*}
\F'_4(k, \cdot, \cdot) = \emat_{4}^T\left(\F_3(k, \cdot, \cdot)\star\frac{1}{4}\mathbbm{1}_{2\times2}\right)\emat_{4}
\end{equation*}
for $k=1,\dots, 4$.
Next, let $W_4\in\mbr^{16\times16}$ be the identity matrix and let $\k_4:\mbr^{12}\times\mbr^{12}\to\mbr$ be the arc-cosine kernel of order zero. 
The output of the fourth layer is then $\F_4\in\mbr^{16\times5\times5}$ given by
\begin{equation*}
\F_4(\cdot, i, j) = \left(\left[\k_4\left(\W_4(m, \cdot), \W_4(n, \cdot)\right)\right]_{m,n=1}^{16}\right)^{-1/2}\left[\k_4\left(\F'_4(\cdot, i, j), \W_4(m, \cdot)\right)\right]_{m=1}^{16}
\end{equation*}
for $i,j=1,\dots, 5$.

The fifth layer is a fully connected layer. Let $\W_5\in\mbr^{120\times16\times5\times5}$ and let $\k_5:\mbr^{16}\times\mbr^{16}\to\mbr$ be the arc-cosine kernel of order zero.
 Then the output of this layer is given by $\F_5\in\mbr^{120}$ given by
\begin{equation*}
\F_5 = \left(\left[\k_5\left(\vec(\W_5(m, \cdot, \cdot, \cdot)), \vec(\W_5(n, \cdot, \cdot, \cdot))\right)\right]_{m,n=1}^{120}\right)^{-1/2}\left[\k_5\left(\vec(\F_4), \vec(\W_5(m, \cdot, \cdot, \cdot))\right)\right]_{m=1}^{120}.
\end{equation*}

Finally, the sixth layer is also a fully connected layer. Let $\W_6\in\mbr^{84\times120}$ and let $\k_6:\mbr^{120}\times\mbr^{120}\to\mbr$ be the arc-cosine kernel of order zero. Then the output is given by $\F_6\in\mbr^{84}$ with
\begin{equation*}
\F_6 = \left(\left[\k_6\left(\W_6(m, \cdot), \W_6(n, \cdot)\right)\right]_{m,n=1}^{84}\right)^{-1/2}\left[\k_6\left(\F_5, \W_6(m, \cdot)\right)\right]_{m=1}^{84}.
\end{equation*}
The output from this layer is the set of features provided to a classifier.

\begin{figure}
\vspace*{-3cm}
\hspace*{-2.5cm}
\begin{center}
\includegraphics[scale=1.0,trim={1cm 5cm 0cm 1cm},clip]{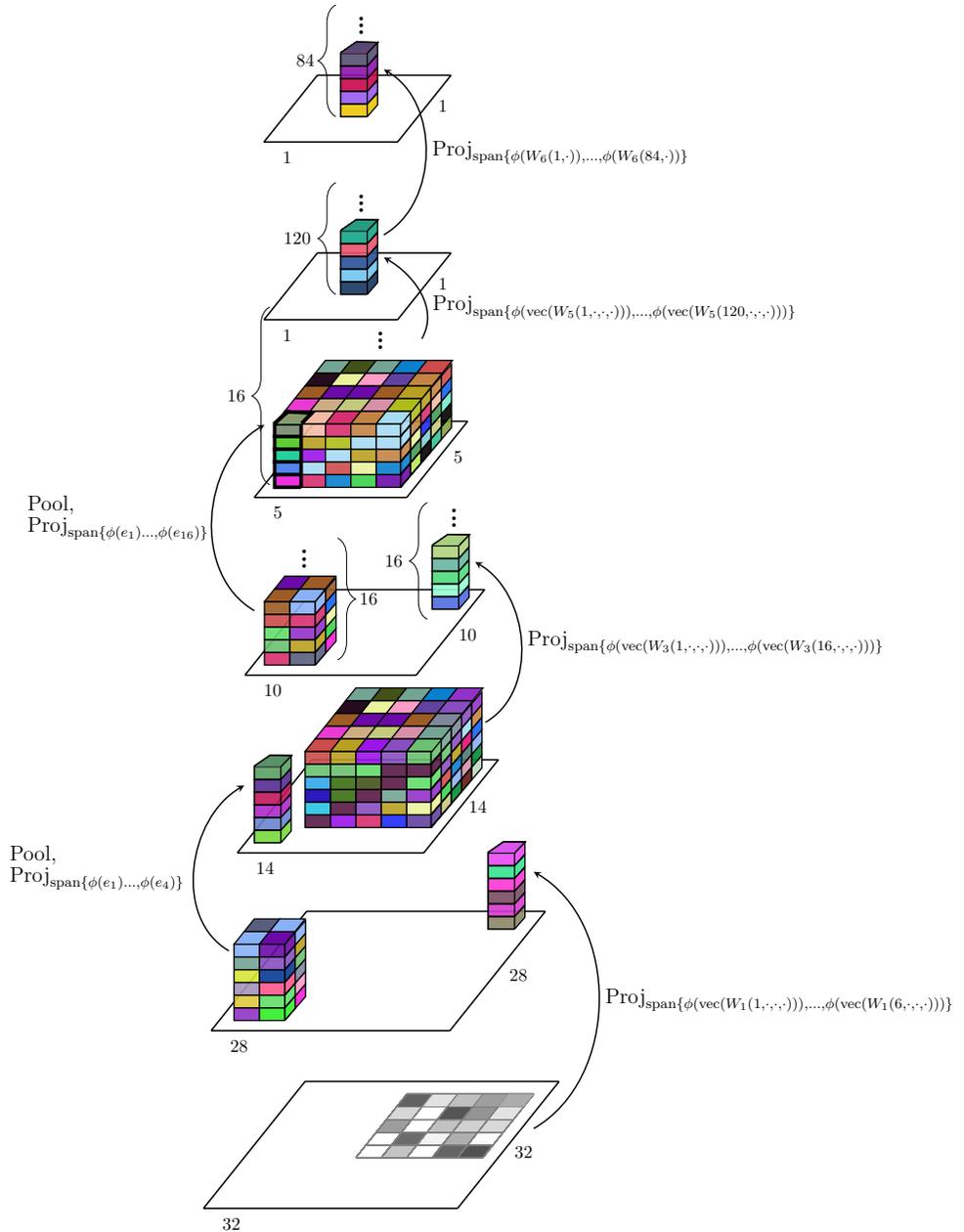}
\end{center}
\caption{LeNet-5 architecture. The dimensions of the stacks of blocks are the dimensions of the filters at each layer. The numbers next to curly brackets indicate the number of filters at each layer when the number of filters is not the same as the height of the stack of blocks. The arrows indicate how one block gets transformed into the block at the next layer. The numbers on the sides of the parallelograms indicate the spatial dimensions of the feature representations at each layer. }%
\end{figure}

\subsection{All-CNN-C}
For the CKN counterpart of All-CNN-C we use the fact that a convolution followed by ReLU activation corresponds to an arc-cosine kernel of order 1.

The input to the model is an image $\F_0\in\mbr^{3\times32\times32}$. Let $\W_1\in\mbr^{96\times3\times3\times3}$. The initial layer consists of projecting onto a subspace spanned by feature maps $\phi(\vec(W_{1, \cdot, \cdot, \cdot})), \dots, \phi(\vec(W_{96, \cdot, \cdot, \cdot}))$  from the arc-cosine kernel on normalized patches and then multiplying by the norms of the patches. Define $\Z_1\in\mbr^{32\times34}$ such that $(\Z_1)_{ij}=1$ if $j=i+1$ for $i=1,\dots, 32$ and is $0$ otherwise.
Let $\N_1$ be a matrix containing the squared norms of $3\times3\times3$ patches with a $1\times1$ stride:
\begin{equation*}
\N_1=\sum_{k=1}^3\left\{\left[(\Z_1^T\F_0(k, \cdot, \cdot)\Z_1)\odot(\Z_1^T\F_0(k, \cdot, \cdot)\Z_1)\right]\star\mathbbm{1}_{3\times3}\right\},
\end{equation*}
where $\odot$ is the Hadamard product.
Let $\k:\mbr\to\mbr$ be defined as $\k(\langle x, y\rangle) = \frac{1}{\pi}\sin(\cos^{-1}\left(\langle x, y\rangle\right)) + (\pi-\theta)\langle x, y\rangle$. Define
\begin{equation*}
\F'_1(k, \cdot, \cdot) = \sum_{z=1}^3\left[(\Z_1^T\F_0(z, \cdot, \cdot)\Z_1)\star \W_1(k, z, \cdot, \cdot)\right]
\end{equation*}
for $k=1,\dots, 96$.
The output from the projection is then given by $\F_1\in\mbr^{96\times32\times32}$ with
\begin{equation*}
\F_1(\cdot, i, j) = \N_1(i,j)^{1/2}\k\left(\sum_{m=1}^{3}\sum_{n,p=1}^3 \W_1(\cdot, m, n, p)\W_1(\cdot, m, n, p)^T\right)^{-1/2}\k\left(\N_1(i,j)^{-1/2}\F'_1(\cdot, i, j) \right),
\end{equation*}
for $i,j=1,\dots, 32$ where $\k$ is understood to be applied element-wise.

The second layer is of the same form as the first layer. Let $\W_2\in\mbr^{96\times96\times3\times3}$ and define $\Z_2\in\mbr^{32\times34}$ such that $(\Z_2)_{ij}=1$ if $j=i+1$  for $i=1,\dots, 32$ and is $0$ otherwise.
Let $\N_2$ be a matrix containing the squared norms of $96\times3\times3$ patches with a $1\times1$ stride:
\begin{equation*}
\N_2=\sum_{k=1}^{96}\left\{\left[(\Z_2^T\F_1(k, \cdot, \cdot)\Z_2)\odot(\Z_2^T\F_1(k, \cdot, \cdot)\Z_2)\right]\star\mathbbm{1}_{3\times3}\right\}.
\end{equation*}
Define
\begin{equation*}
\F'_2(k, \cdot, \cdot) = \sum_{z=1}^{96}\left[(\Z_2^T\F_1(z, \cdot, \cdot)\Z_2)\star \W_2(k, z, \cdot, \cdot)\right]
\end{equation*}
for $k=1,\dots, 96$.
The output from the projection is then given by $\F_2\in\mbr^{96\times32\times32}$ with
\begin{equation*}
\F_2(\cdot, i, j) = \N_2(i,j)^{1/2}\k\left(\sum_{m=1}^{96}\sum_{n,p=1}^3 \W_2(\cdot, m, n, p)\W_2(\cdot, m, n, p)^T\right)^{-1/2}\k\left(\N_2(i,j)^{-1/2}\F'_2(\cdot, i, j) \right),
\end{equation*}
for $i,j=1,\dots, 32$.

The third layer is similar to the previous layer, but subsamples. Let $\W_3\in\mbr^{96\times96\times3\times3}$ and define $\emat_{3}=(e_1, e_3, e_5, \dots, e_{31})$ where $e_i\in\mbr^{32}$ is a vector with 1 in element $i$ and 0 elsewhere. %
Let $\N_3$ be a matrix containing the squared norms of $96\times3\times3$ patches with a $2\times2$ stride:
\begin{equation*}
\N_3=\sum_{k=1}^{96}\emat_{3}^T\left\{\left[\F_1(k, \cdot, \cdot)\odot\F_1(k, \cdot, \cdot)\right]\star\mathbbm{1}_{3\times3}\right\}\emat_{3}.
\end{equation*}
Define
\begin{equation*}
\F'_3(k, \cdot, \cdot) = \sum_{z=1}^{96}\emat_{3}^T\left[\F_2(z, \cdot, \cdot)\star \W_3(k, z, \cdot, \cdot)\right]\emat_{3}
\end{equation*}
for $k=1,\dots, 96$.
The output from the projection is then given by $\F_3\in\mbr^{96\times15\times15}$ with
\begin{equation*}
\F_3(\cdot, i, j) = \N_3(i,j)^{1/2}\k\left(\sum_{m=1}^{96}\sum_{n,p=1}^3 \W_3(\cdot, m, n, p)\W_3(\cdot, m, n, p)^T\right)^{-1/2}\k\left(\N_3(i,j)^{-1/2}\F'_3(\cdot, i, j) \right),
\end{equation*}
for $i,j=1,\dots, 15$.

The fourth layer is of the same form as the second layer, but has 192 filters. Let $\W_4\in\mbr^{192\times96\times3\times3}$ and define $\Z_4\in\mbr^{15\times17}$ such that $(\Z_4)_{ij}=1$ if $j=i+1$  for $i=1,\dots, 15$ and is $0$ otherwise..
Let $\N_4$ be a matrix containing the squared norms of $96\times3\times3$ patches with a $1\times1$ stride:
\begin{equation*}
\N_4=\sum_{k=1}^{96}\left\{\left[(\Z_4^T\F_3(k, \cdot, \cdot)\Z_4)\odot(\Z_4^T\F_3(k, \cdot, \cdot)\Z_4)\right]\star\mathbbm{1}_{3\times3}\right\}.
\end{equation*}
Define
\begin{equation*}
\F'_4(k, \cdot, \cdot) = \sum_{z=1}^{96}\left[(\Z_4^T\F_3(z, \cdot, \cdot)\Z_4)\star \W_4(k, z, \cdot, \cdot)\right]
\end{equation*}
for $k=1,\dots, 192$.
The output from the projection is then given by $\F_4\in\mbr^{192\times15\times15}$ with
\begin{equation*}
\F_4(\cdot, i, j) = \N_4(i,j)^{1/2}\k\left(\sum_{m=1}^{96}\sum_{n,p=1}^3 \W_4(\cdot, m, n, p)\W_4(\cdot, m, n, p)^T\right)^{-1/2}\k\left(\N_4(i,j)^{-1/2}\F'_4(\cdot, i, j) \right),
\end{equation*}
for $i,j=1,\dots, 15$.

The fifth layer is analogous to the fourth layer. Let $\W_5\in\mbr^{192\times192\times3\times3}$ and define $\Z_5\in\mbr^{15\times17}$ such that $(\Z_5)_{ij}=1$ if $j=i+1$  for $i=1,\dots, 15$ and is $0$ otherwise..
Let $\N_5$ be a matrix containing the squared norms of $192\times3\times3$ patches with a $1\times1$ stride:
\begin{equation*}
\N_5=\sum_{k=1}^{192}\left\{\left[(\Z_5^T\F_4(k, \cdot, \cdot)\Z_5)\odot(\Z_5^T\F_4(k, \cdot, \cdot)\Z_5)\right]\star\mathbbm{1}_{3\times3}\right\}.
\end{equation*}
Define
\begin{equation*}
\F'_5(k, \cdot, \cdot) = \sum_{z=1}^{192}\left[(\Z_5^T\F_4(z, \cdot, \cdot)\Z_5)\star \W_5(k, z, \cdot, \cdot)\right]
\end{equation*}
for $k=1,\dots, 192$.
The output from the projection is then given by $\F_5\in\mbr^{192\times15\times15}$ with
\begin{equation*}
\F_5(\cdot, i, j) = \N_5(i,j)^{1/2}\k\left(\sum_{m=1}^{192}\sum_{n,p=1}^3 \W_5(\cdot, m, n, p)\W_5(\cdot, m, n, p)^T\right)^{-1/2}\k\left(\N_5(i,j)^{-1/2}\F'_5(\cdot, i, j) \right),
\end{equation*}
for $i,j=1,\dots, 15$.

The sixth layer is similar to the third layer. Let $\W_6\in\mbr^{192\times192\times3\times3}$ and define $\emat_{6}=(e_1, e_3, e_5, \dots, e_{13})$ where $e_i\in\mbr^{14}$ is a vector with 1 in element $i$ and 0 elsewhere. %
Let $\N_6$ be a matrix containing the squared norms of $192\times3\times3$ patches with a $2\times2$ stride:
\begin{equation*}
\N_6=\sum_{k=1}^{192}\emat_{6}^T\left\{\left[\F_5(k, \cdot, \cdot)\odot\F_5(k, \cdot, \cdot)\right]\star\mathbbm{1}_{3\times3}\right\}\emat_{6}.
\end{equation*}
Define
\begin{equation*}
\F'_6(k, \cdot, \cdot) = \sum_{z=1}^{192}\emat_{6}^T\left[\F_5(z, \cdot, \cdot)\star \W_6(k, z, \cdot, \cdot)\right]\emat_{6}
\end{equation*}
for $k=1,\dots, 192$.
The output from the projection is then given by $\F_6\in\mbr^{192\times7\times7}$ with
\begin{equation*}
\F_6(\cdot, i, j) = \N_6(i,j)^{1/2}\k\left(\sum_{m=1}^{192}\sum_{n,p=1}^3 \W_6(\cdot, m, n, p)\W_6(\cdot, m, n, p)^T\right)^{-1/2}\k\left(\N_6(i,j)^{-1/2}\F'_6(\cdot, i, j) \right),
\end{equation*}
for $i,j=1,\dots, 7$.

The seventh layer is analogous to the fifth layer. Let $\W_7\in\mbr^{192\times192\times3\times3}$ and define $\Z_7\in\mbr^{7\times9}$ such that $(\Z_7)_{ij}=1$ if $j=i+1$  for $i=1,\dots, 7$ and is $0$ otherwise..
Let $\N_7$ be a matrix containing the squared norms of $192\times3\times3$ patches with a $1\times1$ stride:
\begin{equation*}
\N_7=\sum_{k=1}^{192}\left\{\left[(\Z_7^T\F_6(k, \cdot, \cdot)\Z_7)\odot(\Z_7^T\F_6(k, \cdot, \cdot)\Z_7)\right]\star\mathbbm{1}_{3\times3}\right\}.
\end{equation*}
Define
\begin{equation*}
\F'_7(k, \cdot, \cdot) = \sum_{z=1}^{192}\left[(\Z_7^T\F_6(z, \cdot, \cdot)\Z_7)\star \W_7(k, z, \cdot, \cdot)\right]
\end{equation*}
for $k=1,\dots, 192$.
The output from the projection is then given by $\F_7\in\mbr^{96\times7\times7}$ with
\begin{equation*}
\F_7(\cdot, i, j) = \N_7(i,j)^{1/2}\k\left(\sum_{m=1}^{192}\sum_{n,p=1}^3 \W_7(\cdot, m, n, p)\W_7(\cdot, m, n, p)^T\right)^{-1/2}\k\left(\N_7(i,j)^{-1/2}\F'_7(\cdot, i, j) \right),
\end{equation*}
for $i,j=1,\dots, 7$.

The eighth layer switches to $1\times1$ convolutions. Let $\W_8\in\mbr^{192\times192}$ and let $\N_8$ be a matrix containing the squared norms of $192\times1\times1$ patches with a $1\times1$ stride: 
\begin{equation*}
\N_8=\sum_{k=1}^{192}\left[\F_7(k, \cdot, \cdot)\odot\F_7(k, \cdot, \cdot)\right].
\end{equation*}
Define
\begin{equation*}
\F'_8(k, \cdot, \cdot) = \sum_{z=1}^{192}\left[\F_7(z, \cdot, \cdot)\star \W_8(k, z)\right]
\end{equation*}
for $k=1,\dots, 192$.
The output from the projection is then given by $\F_8\in\mbr^{192\times7\times7}$ with
\begin{equation*}
\F_8(\cdot, i, j) = \N_8(i,j)^{1/2}\k\left(\sum_{m=1}^{192}\W_8(\cdot, m)\W_8(\cdot, m)^T\right)^{-1/2}\k\left(\N_8(i,j)^{-1/2}\F'_8(\cdot, i, j) \right),
\end{equation*}
for $i,j=1,\dots, 7$.

The ninth layer again has $1\times1$ convolutions, but with 10 filters. Let $\W_9\in\mbr^{10\times192}$  and let $\N_9$ be a matrix containing the squared norms of $192\times1\times1$ patches with a $1\times1$ stride: 
\begin{equation*}
\N_9=\sum_{k=1}^{192}\left[\F_8(k, \cdot, \cdot)\odot\F_8(k, \cdot, \cdot)\right]
\end{equation*}
Define
\begin{equation*}
\F'_9(k, \cdot, \cdot) = \sum_{z=1}^{192}\left[\F_8(z, \cdot, \cdot)\star \W_9(k, z)\right]
\end{equation*}
for $k=1,\dots, 10$.
The output from the projection is then given by $\F_9\in\mbr^{10\times7\times7}$ with
\begin{equation*}
\F_9(\cdot, i, j) = \N_9(i,j)^{1/2}\k\left(\sum_{m=1}^{192}\W_9(\cdot, m)\W_9(\cdot, m)^T\right)^{-1/2}\k\left(\N_9(i,j)^{-1/2}\F'_9(\cdot, i, j) \right),
\end{equation*}
for $i,j=1,\dots, 7$.

The tenth layer performs pooling across the entire feature maps. The output is $\F_{10}\in\mbr^{10}$ given by
\begin{equation*}
\F_{10}(k) = \F_9(k, \cdot, \cdot)\star\frac{1}{49}\mathbbm{1}_{7\times7}
\end{equation*}
for $k=1,\dots, 10$. The output from this layer is the set of features provided to a classifier.

\begin{figure}
\vspace*{-3cm}
\hspace*{-2.5cm}
\includegraphics[scale=1.0,trim={0cm 5cm 1cm 1cm},clip]{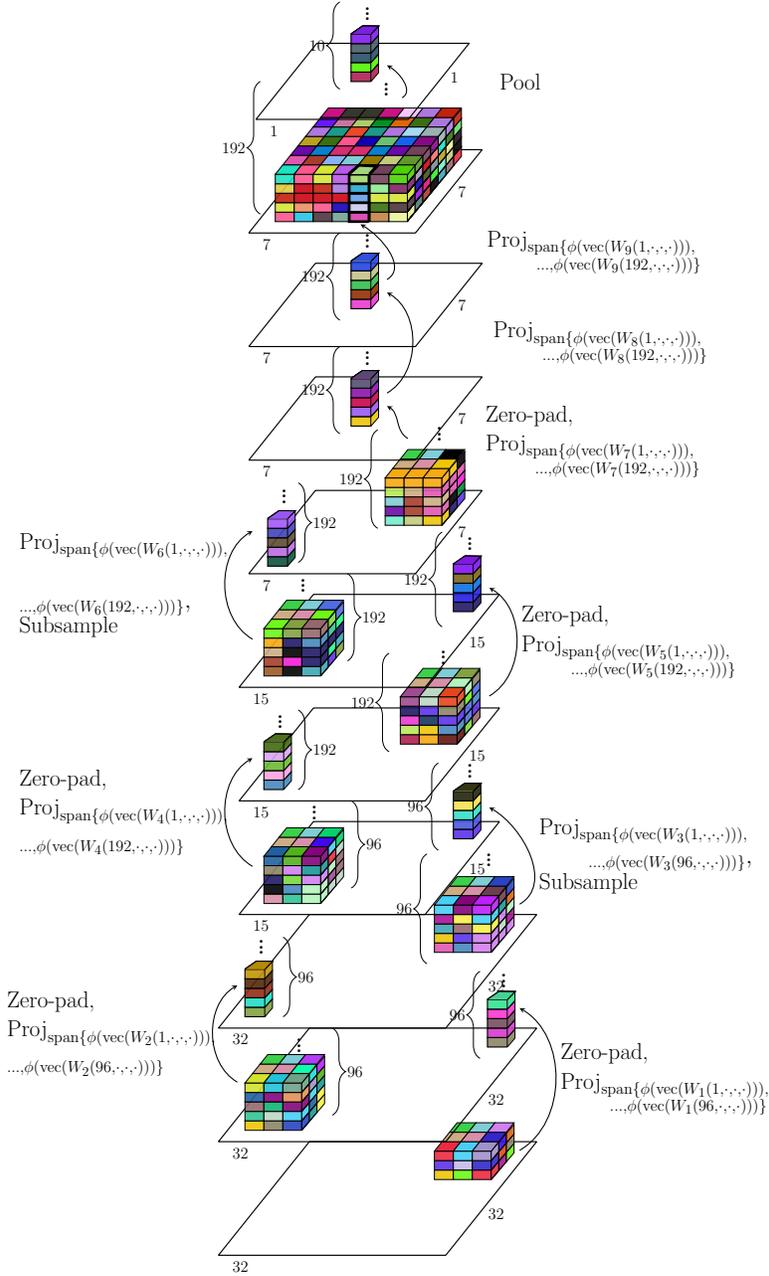}
\caption{All-CNN-C architecture. The dimensions of the stacks of blocks are the dimensions of the filters at each layer. The numbers next to curly brackets indicate the number of filters at each layer when the number of filters is not the same as the height of the stack of blocks. The arrows indicate how one block gets transformed into the block at the next layer. The numbers on the sides of the parallelograms indicate the spatial dimensions of the feature representations at each layer.}
\end{figure}

%% file: sections/appendix5-gradient_proofs.tex
\section{CKN gradient}
\label{app:gradient_proofs}
In this section we derive the gradient of the loss function with respect to the filters for a CKN. Along the way we compute the gradient for a single layer of the CKN with respect to both the filters and the inputs. Following this, for the case when the kernels are RBF kernels on the sphere we compute the gradient of the loss function with respect to the bandwidths of the kernels.

\subsection{Notations}\label{ssec:notations_ckn_grad}
\paragraph{Linear algebra}
 We denote $\id_d$ the identity matrix in $\reals^{d\times d}$. For a matrix $A\in \reals^{d \times n} =(a_1, \ldots, a_n)$, we denote $\vec(A) = (a_1;\ldots;a_n) \in \reals^{dn}$ the concatenation of its columns. Given two dimensions $d, n$, we denote $T_{d,n} \in \reals^{dn \times dn}$ the linear operator such that for a matrix $A\in \reals^{d\times n}$, $\vec(A^T) = T_{d,n} \vec(A)$.
Given matrices $A_1, \ldots, A_n$, we denote $\prod_{i=1}^{n} A_i = A_n A_{n-1} \ldots A_1$, i.e., the multiplication is performed from right to left in increasing order of the indices.
 
We denote the set of all positive definite matrices of size $n\times n$ by $S^n_{++}$. We assume all matrices have real-valued entries and use the notation $A^{1/2}$ to denote the square root of a positive semi-definite matrix. That is, if $A=UDU^T$ is the eigendecomposition of $A$ with $U$ orthonormal and $D$ diagonal, then $A^{1/2}=UD^{1/2}U^T$.

\paragraph{Derivatives}
For a multivariate function $f: \reals^d \rightarrow \reals^n$, we denote by
$\nabla f(x) = \left(\frac{\partial f_i(x)}{\partial x_j}\right)_{i=1,\ldots,n, j=1,\ldots, d} \in \reals^{n \times d}$ its Jacobian at $x \in \reals^{d}$, where $f_i(x)$ is the $i$\textsuperscript{th} coordinate of $f(x)$. For a function $f: \reals^p \times \reals^q \rightarrow \reals^n$, denoted $f(x,y)$ for $x\in \reals^p$, $y \in \reals^q$, we denote by $\nabla_x f(x,y)$ its partial Jacobian with respect to the variable $x$ at a point $(x,y) \in \reals^p \times \reals^q$, i.e., $\nabla_x f(x,y) = \left(\frac{\partial f_i(x, y)}{\partial x_j}\right)_{i=1,\ldots,n, j=1,\ldots, p} \in \reals^{n \times p}$ and similarly $\nabla_y f(x,y) = \left(\frac{\partial f_i(x, y)}{\partial y_j}\right)_{i=1,\ldots,n, j=1,\ldots, q} \in \reals^{n \times q}$.

For a multivariate matrix function $\Psi: \reals^{m \times n} \times \reals^{p \times q} \rightarrow \reals^{p^+ \times q^+}$, denoted $\Psi(W,F)\in \reals^{p^+ \times q^+}$ for $W\in \reals^{m \times n}$, $F\in \reals^{p \times q}$, we denote by $\vec(\Psi):  \reals^{m n} \times \reals^{p q} \rightarrow \reals^{p^+  q^+}$ its vectorized counterpart such that $\vec(\Psi(W,F)) = \vec (\Psi)(\vec(W), \vec(F))$ for any $W\in \reals^{m \times n}$, $F\in \reals^{p \times q}$. We then denote by $\nabla_{\vec(W)} \vec(\Psi)(\vec(W), \vec(F)) \in \reals^{p^+ q^+ \times mn}$ the partial Jacobian of the vectorized counterpart of $\Psi$ with respect to its first argument $\vec(W)$, and analogously for $\nabla_{\vec(F)} \vec(\Psi)(\vec(W), \vec(F)) \in \reals^{p^+ q^+ \times pq}$.
In the following when $W,F$ are clear from the context we denote simply $\nabla_{\vec(W)} \vec(\Psi) = \nabla_{\vec(W)} \vec(\Psi)(\vec(W), \vec(F))$ and  $\nabla_{\vec(F)} \vec(\Psi) = \nabla_{\vec(F)} \vec(\Psi)(\vec(W), \vec(F))$.

\subsection{Detailed CKN layer}
The output of a CKN layer was informally presented in Section~\ref{sec:ckn_formulation}. Now we precisely describe each component of a layer and its corresponding dimension.

At layers $\ell=1,\dots, L$, let $\f_\ell$ be the number of filters, $\s_\ell$ be the total size of the patches, $\p_\ell$ be the number of patches prior to pooling, and $\p_\ell'$ be the number of patches after pooling. Now at each layer $\ell$ let
$\F_\ell\in\mbr^{\f_\ell\times\p_\ell'}$ denote the features, $W_\ell \in\mbr^{\f_\ell\times \s_\ell}$ denote the filters, and $\P_\ell\in\mbr^{\p_\ell\times\p_\ell'}$ denote the pooling matrix. 

Define the \emph{patch extraction function} at layer $\ell$, $\E_\ell:\mbr^{\f_{\ell-1}\times\p_{\ell-1}'}\to\mbr^{\s_\ell\times\p_\ell}$, by 
\begin{equation}\label{eq:patch_extract}
E_\ell(X) = \sum_{i=1}^{\s_\ell/f_{\ell-1}}E_{\ell1i}XE_{\ell2i}
\end{equation}
with $E_{\ell1i}\in \mbr^{\s_\ell\times \f_{\ell-1}}$ and $E_{\ell2i}\in\mbr^{\p'_{\ell-1}\times p_\ell}$ for $i=1,\dots, \s_\ell/\f_{\ell-1}$. This function takes as input the feature representation of a whole image from layer $\ell-1$ with size $f_{\ell-1} \times p'_{\ell-1}$  and outputs its $p_\ell$ patches put in columns of size $s_\ell$ of a matrix $E_\ell(X)\in \reals^{\s_\ell\times\p_\ell}$.
  
Define the \emph{patch normalization} function, $\tilde N_\ell: \mbr^{\s_\ell\times\p_\ell} \to\mbr^{\p_\ell\times\p_\ell}$, by 
\begin{equation}  \label{eq:patch_norm}
\tilde N_\ell(X)=\left[(X^TX)\odot \id_{p_\ell}\right]^{1/2}.
\end{equation}
This function outputs a square matrix whose diagonal contains the norms of the patches.
Finally we denote by
\begin{equation*}
N_\ell(X) = \tilde N_\ell(E_\ell(X))
\end{equation*}
the composition of the two previous operations.

The output of the convolutional kernel layer is then
\begin{align}\label{eq:ckn_layer_detailed}
F_\ell = \Psi_\ell(F_{\ell-1}, W_\ell) \coloneqq\left(k\left(W_\ell W_\ell^T\right)+\epsilon \id_{f_\ell}\right)^{-1/2}k\left(W_\ell\E_\ell(F_{\ell-1})N_\ell(\F_{\ell-1})^{-1}\right)N_\ell(\F_{\ell-1})P_\ell,
\end{align}
where $\epsilon>0$ and $k:\mbr\to\mbr$ is a differentiable dot product kernel understood to be applied element-wise. In the following we denote it compactly by 
\[
\Psi_\ell(F_{\ell-1}, W_\ell) \coloneqq A(W_\ell) B(W_\ell, F_{\ell-1}) N_\ell(\F_{\ell-1})P_\ell,
\]
where 
\begin{align}
	A(W_\ell) & \coloneqq \left(k\left(W_\ell W_\ell^T\right)+\epsilon \id_{f_\ell}\right)^{-1/2} \label{eq:ckn_layer_shortcuts}\\
	B(W_\ell, F_{\ell-1}) &  \coloneqq k\left(W_\ell\E_\ell(F_{\ell-1})N_\ell(\F_{\ell-1})^{-1}\right). \nonumber
\end{align}

Recall the objective on which we want to apply first order optimization schemes. Given a set of images $\im^{(1)}_0,\dots, \im^{(n)}_0$ with corresponding labels $y^{(1)},\dots, y^{(n)}$, we
consider a loss $\mathcal{L}$ with a linear classifier parameterized by $W_{L+1}$, leading to the optimization problem 
\begin{align*}
\min_{\W_1,\dots,\W_{L+1}} \quad & \frac{1}{n} \sum_{i=1}^n \mathcal{L}\left(y^{(i)}, \left\langle \W_{L+1}, \im_L^{(i)}\right\rangle\right)+\lambda \Vert \W_{L+1}\Vert_F^2 \\
\mbox{subject to} \quad & W_\ell \in S^{d_\ell} \quad \mbox{for $\ell=1,\ldots, L$}
\end{align*}
where $S^{d_\ell}=\prod_{j=1}^{\f_\ell} \mbs^{\s_\ell-1}$ is the Cartesian product of Euclidean unit spheres in $\reals^{\s_\ell}$ and $F_L^{(i)}$ is the output of the $L$th layer of the network described by~\eqref{eq:ckn_layer_detailed} applied to the $i$th image. 
We henceforth assume that $\mcl$ is differentiable with respect to its second argument.

\subsection{CKN gradient}
The gradient of the loss with respect to the weights of the network is given by the chain rule, as recalled in the following proposition. The key elements are then the derivative of a layer with respect to its inputs and its weights, which are detailed by a series of lemmas.
\begin{proposition_unnumbered}[\ref{prop:loss_grad}]
	Let  $\mcl(y, \langle \W_{L+1}, \F_L\rangle)$ be the loss incurred by an image-label sample $(F_0, y)$, where 
$F_L$ is the output of $L$ layers of the network described by ~\eqref{eq:ckn_layer} and $W_{L+1}$ parameterizes the linear classifier. 
Then the 
gradient of the loss 
with respect to the inner weights $W_{\ell}$, $1\leq\ell\leq L$ is given by
\begin{align*}
\nabla_{\vec(W_{\ell})} \mcl\left(y, \left\langle \W_{L+1}, \F_L\right\rangle\right) =
\mcl'\vec(W_{L+1})^T\left[\prod_{\ell '=\ell+1}^{L} \nabla_{\vec\left(F_{\ell'-1}\right)} \vec\left(\Psi_{\ell'}\right)\right]
\times\nabla_{\vec(W_{\ell})} \vec\left(\Psi_{\ell}\right),
\end{align*}
where $\mcl'=\frac{\partial \mcl(\bar y, \hat y)}{\partial \hat y }|_{(\bar y, \hat y) = (y,\left\langle \W_{L+1}, \F_L\right\rangle)} $,  $\nabla_{\vec(F_{\ell'-1})} \vec(\Psi_{\ell'})$ is detailed in Proposition~\ref{prop:2layer_grad_part1}, and $\nabla_{\vec(W_{\ell})} \vec(\Psi_{\ell})$ is detailed in Proposition~\ref{prop:1layer_grad}.
\end{proposition_unnumbered}
\subsubsection{Layer derivative with respect to its weights}
The proof of the derivative of a CKN layer with respect to its weights is based on decomposing the gradient computations into the following lemmas.
\begin{lemma}
	\label{lemma:matrix_sqrt}
	Define the function $F:S^n_{++}\to\mbr^{n\times n}$ by $F(A)=A^{1/2}$.
	Then for a positive definite matrix $A\in\mbr^{n\times n}$ and a matrix $H\in\mbr^{n\times n}$ such that $A+H$ is positive definite we have
	\begin{align*}
	\vec(F(A+H)) = \vec(F(A)) +  \left(\id_n \otimes A^{1/2}+A^{1/2}\otimes \id_n\right)^{-1}\vec(H) + o(\|H\|_F).
	\end{align*}
\end{lemma}
\begin{proof}
	First we will aim to find the matrix $C$ such that
	\begin{align*}
	(A+H)^{1/2} = A^{1/2} + C.
	\end{align*}
	To this end,  observe that
	\begin{align*}
	\vec(H) 
	&= \left(\id_n \otimes (A+H)^{1/2}+A^{1/2}\otimes \id_n\right)\vec(C).
	\end{align*}
	The term $\left(\id_n \otimes (A+H)^{1/2}+A^{1/2}\otimes \id_n\right)$ is invertible. This follows from the fact that $A$ and $A+H$ are positive definite, the eigenvalues of a Kronecker product are all products of the eigenvalues, and the sum of positive definite matrices is positive definite. Therefore, we obtain
	\begin{align*}
	\vec(C) = \left(\id_n \otimes (A+H)^{1/2}+A^{1/2}\otimes \id_n\right)^{-1}\vec(H). 
	\end{align*}
	Next, we will show that
	\begin{align*}
	\left(\id_n \otimes (A+H)^{1/2}+A^{1/2}\otimes \id_n\right)^{-1} = \left(\id_n \otimes A^{1/2}+A^{1/2}\otimes \id_n\right)^{-1}+O(\|H\|_F).
	\end{align*}
	To see this, first note that
	\begin{align*}
	&\left(\id_n \otimes (A+H)^{1/2}+A^{1/2}\otimes \id_n\right)^{-1}  \\
	=\:& \left(A^{1/2}\otimes \id_n\right)^{-1/2}\left[\left(A^{1/2}\otimes \id_n\right)^{-1/2}\left(\id_n \otimes (A+H)^{1/2}\right)\left(A^{1/2}\otimes \id_n\right)^{-1/2}+\id_n\right]^{-1} \left(A^{1/2}\otimes \id_n\right)^{-1/2}\\
	=\:& \left(A^{1/2}\otimes \id_n\right)^{-1/2}\left[\id_n - \left(A^{1/2}\otimes \id_n\right)^{-1/2}\left(\id_n \otimes (A+H)^{1/2}\right)\left(A^{1/2}\otimes \id_n\right)^{-1/2} \right. \\
	&+ \left. o\left(\left\Vert\left(A^{1/2}\otimes \id_n\right)^{-1/2}\left(\id_n \otimes (A+H)^{1/2}\right)\left(A^{1/2}\otimes \id_n\right)^{-1/2}\right\Vert_F\right)\right]\left(A^{1/2}\otimes \id_n\right)^{-1/2}\\
	=\:&  \left(A^{1/2}\otimes \id_n\right)^{-1}-\left(A^{1/2}\otimes \id_n\right)^{-1}\left(\id_n \otimes (A+H)^{1/2}\right)\left(A^{1/2}\otimes \id_n\right)^{-1} \\
	&+ o\left(\left\Vert\left(A^{1/2}\otimes \id_n\right)^{-1/2}\left(\id_n \otimes (A+H)^{1/2}\right)\left(A^{1/2}\otimes \id_n\right)^{-1/2}\right\Vert_F\right).
	\end{align*}
	By Lemma~\ref{lemma:matrix_sqrt_cont}, $(A+H)^{1/2} = A^{1/2}+O(\|H\|_F)$. Therefore, 
	\begin{align*}
	&\left(\id_n \otimes (A+H)^{1/2}+A^{1/2}\otimes \id_n\right)^{-1}  \\
	=\:& \left(A^{1/2}\otimes \id_n\right)^{-1}-\left(A^{1/2}\otimes \id_n\right)^{-1}\left(\id_n \otimes A^{1/2}\right)\left(A^{1/2}\otimes \id_n\right)^{-1} \\
	&+ o\left(\left\Vert\left(A^{1/2}\otimes \id_n\right)^{-1/2}\left(\id_n \otimes A^{1/2}\right)\left(A^{1/2}\otimes \id_n\right)^{-1/2}\right\Vert_F\right)+O(\|H\|_F)\\
	=\:& \left(\id_n \otimes A^{1/2}+A^{1/2}\otimes \id_n\right)^{-1}+O(\|H\|_F).
	\end{align*}
\end{proof}

\begin{lemma}
	\label{lemma:matrix_sqrt_cont}
	The matrix square root function $F:S^n_{++}\to\mbr^{n\times n}$ given by $F(A)=A^{1/2}$ is continuous. 
\end{lemma}
\begin{proof}
	By $1/2$-homogeneity of the square root function, it is sufficient to show that there exists $C>0$ such that for any $A,B \in S^n_{++}$, 
	\begin{equation}\label{eq:square_root_homo_cond}
	\|A-B\|_2 \leq 1 \Rightarrow \|A^{1/2} - B^{1/2}\|_2 \leq C
	\end{equation}
	where $\|\cdot\|_2$ denotes the operator norm associated with the Euclidean norm.
	Indeed, for any $A,B \in S^n_{++}$, denoting $\lambda = \|A-B\|_2$, $\tilde A = A/\lambda$, $\tilde B = B/\lambda$, we get $\|\tilde A - \tilde B\|_2 =1$ and if~\eqref{eq:square_root_homo_cond} holds, then $\|\tilde A^{1/2} - \tilde B^{1/2}\|_2 \leq C$, which reads 
	\[
	\|A^{1/2} - B^{1/2}\|_2 \leq C \sqrt{\lambda} = C \|A-B\|_2^{1/2}.
	\]
	1/2- H\"older continuity then implies continuity.
	
	Now to prove~\eqref{eq:square_root_homo_cond}, define 
	\begin{equation}\nonumber
	g(x) = \int_{0}^{1}\left(1 - \frac{1}{1+tx}\right)t^{-3/2}dt.
	\end{equation}
	The function $g$ can be extended as a function on positive definite matrices that acts on their spectra. Precisely, for $A \in S_{++}^n$, diagonalized as $A = U \Lambda U^\top$ with $U^{-1} = U^T$ and $\Lambda = \diag(\lambda_1, \ldots, \lambda_n)$, denote $g(A) =  U g(\Lambda) U^\top$ where $g(\Lambda) = \diag(g(\lambda_1), \ldots, g(\lambda_n))$.
	
	We will show that there exists $K>0$ such that 
	\begin{equation}\label{eq:approx_square root}
	\|KA^{1/2}- g(A) \| \leq 2
	\end{equation}
	and that $g$ is $K'$-Lipschitz-continuous on $S_{++}^n$ for a given $K'>0$. Then given  $A,B \in S^n_{++}$ with $\|A-B\|_2 \leq 1$ we would get 
	\begin{equation}\label{eq:continuity_goal}
	\|A^{1/2} - B^{1/2}\|_2 \leq \|A^{1/2} - g(A)/K\|_2 + \frac{1}{K}\|g(A) -g(B)\|_2 + \|B^{1/2} - g(B)/K\|_2 \leq \frac{4}{K} + \frac{K'}{K} :=C.
	\end{equation}
	This gives~\eqref{eq:square_root_homo_cond} and concludes the proof.
	
	For~\eqref{eq:approx_square root}, after the change of variables $s = xt$, we get 
	\begin{align*}
		g(x) = x^{1/2}\int_{0}^{x}\frac{s}{1+s}s^{-3/2}ds = x^{1/2}\left(\int_0^\infty\frac{s}{1+s}s^{-3/2}ds-\int_x^\infty\frac{s}{1+s}s^{-3/2}ds \right) = Kx^{1/2} + h(x),
	\end{align*}
	where $K= \int_0^\infty\frac{s}{1+s}s^{-3/2}ds=\int_0^1\frac{s}{1+s}s^{-3/2}ds+\int_1^\infty\frac{s}{1+s}s^{-3/2}ds \leq \int_0^1\frac{s}{1+s}s^{-3/2}ds + \int_1^\infty s^{-3/2}ds<\infty$ and $h(x)=-x^\frac{1}{2}\int_x^\infty\frac{s}{1+s}s^{-3/2}ds$
	with $|h(x)|\le x^\frac{1}{2}\int_x^\infty s^{-3/2}ds =2$.
	Therefore we get~\eqref{eq:approx_square root}, as
	\[
	\|g(A) - KA^{1/2}\|_2 = \|h(A)\|_2 \leq 2
	\]
	where $h(A)$ denotes the application of $h$ on the spectrum of $A$. Boundedness of $h$ on real numbers imply directly its boundedness on matrices.
	Now for the Lipschitz continuity of $g$, first note that the integral commutes with the matrix operations defining the diagonalization, such that
	\[
	g(A) = \int_0^1 \id_n-(\id_n+tA)^{-1}t^{-3/2}dt.
	\]
	Then 
	\[
	g(A)-g(B)=\int_0^1 \left[(\id_n+tB)^{-1}-(\id_n+tA)^{-1}\right]t^{-3/2}dt=\int_0^1(\id_n+tB)^{-1}(A-B)(\id_n+tA)^{-1}t^{-1/2}dt,
	\]
	where we used that $X^{-1}-Y^{-1}=X^{-1}(Y-X)Y^{-1}$. So finally 
	\begin{align*}
		\|f(A)-f(B)\|_2 = & \left\Vert \int_0^1(\id_n+tB)^{-1}(A-B)(\id_n+tA)^{-1}t^{-1/2}dt \right\Vert_2 \\ 
		& \le \int_0^1 \|(\id_n+tB)^{-1}(A-B)(\id_n+tA)^{-1}t^{-1/2}\|_2dt \\ & \le \int_0^1 \|(\id_n+tB)^{-1}\|_2\|A-B\|_2\|(\id_n+tA)^{-1}\|_2t^{-1/2}dt \\
		& \le \|A-B\|_2 \int_0^1 t^{-1/2} dt =2\|A-B\|_2,
	\end{align*}
	which ensures~\eqref{eq:continuity_goal} and therefore the 1/2- H\"older continuity of the matrix square root function.
\end{proof}

The following three lemmas may be proven via Taylor expansions. 
\begin{lemma}
	\label{lemma:matrix_inv}
	Define the function $F:S^n_{++}\to\mbr^{n\times n}$ by $F(A)=A^{-1}$. Then for a positive definite matrix $A\in\mbr^{n\times n}$ and a matrix $H\in\mbr^{n\times n}$ such that $A+H$ is positive definite we have
	\begin{align*}
	F(A+H) = F(A) - F(A)HF(A) + o(\|H\|_F).
	\end{align*} 
\end{lemma}

\begin{lemma}
	\label{lemma:kernel_aa}
	Define the function $F:\mbr^{m\times n}\to\mbr^{m\times m}$ by $F(A) = k\left(AA^T\right)$ where $k:\mbr\to\mbr$ is a differentiable dot product kernel computed element-wise on $AA^T$. Then for a matrix $A\in\mbr^{m\times n}$ we have
	\begin{align*}
	F(A+H) = F(A) + k'\left(AA^T\right)\odot\left(HA^T+AH^T\right) + o\left(\|H\|_F\right),
	\end{align*} 
	where $\odot$ denotes the Hadamard (element-wise) product.
\end{lemma}

\begin{lemma}
	\label{lemma:kernel_ab}
	Let $B\in\mbr^{p\times n}$ and define the function $F:\mbr^{m\times n}\to\mbr^{m\times p}$ by $F(A) = k\left(AB^T\right)$ where $k:\mbr\to\mbr$ is a differentiable dot product kernel computed element-wise on $AB^T$. Then for a matrix $A\in\mbr^{m\times n}$ we have
	\begin{align*}
	F(A+H) = F(A) + k'\left(AB^T\right)\odot\left(HB^T\right) + o(\|H\|_F).
	\end{align*} 
\end{lemma}

Now we are ready to differentiate a layer of a CKN with respect to its weights.
\begin{proposition}[Derivative of the CKN layer with respect to the weights]
\label{prop:1layer_grad}
The output of the $\ell$\textsuperscript{th} convolutional kernel layer defined in~\eqref{eq:ckn_layer_detailed},
\begin{align*}
\Psi_\ell(F_{\ell-1}, W_\ell) & =\left(k\left(W_\ell W_\ell^T\right)+\epsilon \id_{f_\ell}\right)^{-1/2}k\left(W_\ell\E_\ell(F_{\ell-1})N_\ell^{-1}(\F_{\ell-1})\right)N_\ell(\F_{\ell-1})P_\ell, \\
& = A(W_\ell) B(W_\ell, F_{\ell-1}) N_\ell(\F_{\ell-1})P_\ell
\end{align*}
where $A(W_\ell), B(W_\ell, F_{\ell-1})$ are defined in~\eqref{eq:ckn_layer_shortcuts}, has a partial derivative with respect to the weights given by
\begin{align*}
\nabla_{\vec(W_{\ell})} \vec(\Psi_\ell) 
=& -[(B(W_\ell, F_{\ell-1})  N_\ell(\F_{\ell-1})P_\ell)^T\otimes \id_{f_\ell}] \\
& \phantom{-}\times \left(\id_{f_\ell} \otimes A(W_\ell) + A(W_\ell)\otimes \id_{f_\ell}  \right)^{-1}\left(A(W_\ell)^2 \otimes A(W_\ell)^2 \right)\\
&\phantom{-} \times\diag\left[\vec\left(k'\left(W_\ell W_\ell^T\right)\right)\right](W_\ell\otimes \id_{f_\ell} + (\id_{f_\ell}\otimes W_\ell)T_{f_\ell, s_\ell})\\
&+\left[(N_\ell(\F_{\ell-1})P_\ell)^T\otimes A(W_\ell)\right] \\
& \phantom{+} \times \diag\left[\vec\left(k'\left(W_\ell\E_\ell(F_{\ell-1})N_\ell(F_{\ell-1})^{-1}\right)\right)\right]\left[(\E_\ell(F_{\ell-1})N_\ell(F_{\ell-1})^{-1})^T\otimes \id_{f_\ell}\right],
\end{align*}
where $T_{f_\ell, s_\ell}$ is defined in Section~\ref{ssec:notations_ckn_grad}.
\end{proposition}
\begin{proof}
We drop the layer index and simply denote $W = W_\ell$ and $F = F_{\ell-1}$. In addition, we denote $N = N(F)$ and $E = E(F)$.
Denote the Fréchet derivative of a function $f$ at a point $A$ in the direction $H$ by $L_f(A,H)$. On fixed inputs $F$, denote $\tilde \Psi(W) = \Psi(W, F)$ the restricted output function. Similarly we denote $\tilde B(W) = B(W,  F) N P$.
Observe that
\begin{align*}
\vec(\tilde \Psi (W)) &=  (\tilde B(W)^T\otimes \id_{f})\vec(A(W)).
\end{align*}
Therefore, we have by the product rule and chain rule \citep[Theorems 3.3, 3.4]{higham2008} that 
\begin{align*}
L_{\vec(\tilde \Psi)}(W,H) 
&= L_{\tilde B^T\otimes \id_{f}}(W,H)\vec(A(W)) + (\tilde B(W)^T\otimes \id_{f})L_{\vec(A)}(W,H).
\end{align*}
By the chain rule and Lemma~\ref{lemma:kernel_ab}, we have that 
\begin{align*}
&L_{\tilde B^T\otimes \id_{f}}(W,H)\vec(A(W)) \\
=& \left\{\left\{\left[k'\left(WE N^{-1}\right)\odot\left(H E N^{-1}\right)\right]NP\right\}^T\otimes \id_{f}\right\}\vec(A(W))\\
=& (P^TN^T\otimes A(W))\diag\left[\vec\left(k'\left(WE N^{-1}\right)\right)\right]((E N^{-1})^T\otimes \id_{f})\vec\left(H\right).
\end{align*}
Additionally, by the chain rule, and Lemmas~\ref{lemma:matrix_sqrt} - \ref{lemma:kernel_aa} we have
\begin{align*}
(\tilde B(W)^T\otimes \id_{f})L_{\vec(A)}(W,H) = &  -(\tilde B(W)^T\otimes \id_{f})\left(\id_{f} \otimes A(W) + A(W)\otimes \id_{f}  \right)^{-1} \\
& \times \vec\left\{A(W)^2
\left[k'\left(WW^T\right)\odot\left(HW^T+WH^T\right)\right]A(W)^2\right\}\\
=& -(\tilde B(W)^T\otimes \id_{f})\left(\id_{f} \otimes A(W) + A(W)\otimes \id_{f}  \right)^{-1}\left(A(W)^2 \otimes A(W)^2 \right)\\
&\quad\times\diag\left[\vec\left(k'\left(WW^T\right)\right)\right](W\otimes \id_{f} + (\id_{f}\otimes W)T^T)\vec(H),
\end{align*}
where $T_{f,s}$ is the vectorized transpose matrix satisfying $\vec(H^T) = T_{f,s}\vec(H)$ for $H \in \reals^{f \times s}$.
Therefore,
\begin{align*}
\nabla_{\vec(W)} \vec(\Psi) 
=& -[\tilde B(W)^T\otimes \id_{f}]\left(\id_{f} \otimes A(W) + A(W)\otimes \id_{f}  \right)^{-1}\left(A(W)^2 \otimes A(W)^2 \right)\\
&\times\diag\left[\vec\left(k'\left(WW^T\right)\right)\right](W\otimes \id_{f} + (\id_{f}\otimes W)T_{f,s})\\
&+(P^TN\otimes A(W))\diag\left[\vec\left(k'\left(WE N^{-1}\right)\right)\right]((E N^{-1}) ^T\otimes \id_{f}).
\end{align*}
\end{proof}

\subsubsection{Layer derivative with respect to its input}
The proof of the derivative of a CKN layer with respect to its inputs is based on decomposing the gradient computations into Lemma~\ref{lemma:kernel_ab} and the following additional lemmas. Lemmas~\ref{lemma:norm} and \ref{lemma:inv_norm} may be proven via Taylor expansions.

\begin{lemma}
\label{lemma:matrix_mul}
Let $M_1\in\mbr^{m\times n}$ and $M_2\in\mbr^{p\times q}$ and define the function $F:\mbr^{n\times p}\times\mbr^{m\times q}$ by $F(A) = M_1AM_2$. Then for $H\in\mbr^{n\times p}$ we have
\begin{align*}
F(A+H) = F(A) + F(H).
\end{align*} 
\end{lemma}

\begin{lemma}
\label{lemma:norm}
Define the function $F:\mbr^d\backslash\{0\}\to\mbr$ by $F(x) = \Vert x\Vert$. Then for $x\in\mbr^d\backslash\{0\}$ and $h\in\mbr^d$ we have
\begin{align*}
F(x+h) = F(x) + F^{-1}(x)x^Th+ o(\Vert h\Vert).
\end{align*} 
\end{lemma}

\begin{corollary}
\label{corollary:norm}
Define the function $F:\mbr^{m\times n}\to\mbr^{n\times n}$ by $F(A) = \left[(A^TA)\odot \id_n\right]^{1/2}$. Then for $A\in\mbr^{m\times n}$ with $A_{\cdot, j}\in\mbr^{m}\backslash\{0\}$  for all $j=1,\dots, n$ and $H\in\mbr^{m\times n}$ we have
\begin{align*}
F(A+H) = F(A) + F^{-1}(A)\odot (A^TH)+ o(\Vert H\Vert_F).
\end{align*} 
\end{corollary}

\begin{lemma}
\label{lemma:inv_norm}
Define the function $F:\mbr^d\backslash\{0\}\to\mbr$ by $F(x) = \Vert x\Vert^{-1}$. Then for $x\in\mbr^d\backslash\{0\}$ and $h\in\mbr^d$ we have
\begin{align*}
F(x+h) = F(x) - F^{3}(x)x^Th+ o(\Vert h\Vert).
\end{align*} 
\end{lemma}

\begin{corollary}
\label{corollary:inv_norm}
Define the function $F:\mbr^{m\times n}\to\mbr^{n\times n}$ by $F(A) = \left[(A^TA)\odot \id_n\right]^{-1/2}$. Then for $A\in\mbr^{m\times n}$ with $A_{\cdot, j}\in\mbr^{m}\backslash\{0\}$  for all $j=1,\dots, n$ and $H\in\mbr^{m\times n}$ we have
\begin{align*}
F(A+H) = F(A) - F(A)^{-3}\odot (A^TH)+ o(\Vert H\Vert_F).
\end{align*} 
\end{corollary}

\begin{proposition}[Derivative of the network with respect to the input]
\label{prop:2layer_grad_part1}
The output of the $\ell$\textsuperscript{th} convolutional kernel layer defined in~\eqref{eq:ckn_layer_detailed},
\begin{align*}
\Psi_\ell(F_{\ell-1}, W_\ell) & =\left(k\left(W_\ell W_\ell^T\right)+\epsilon \id_{f_\ell}\right)^{-1/2}k\left(W_\ell\E_\ell(F_{\ell-1})N_\ell(\F_{\ell-1})^{-1}\right)N_\ell(\F_{\ell-1})P_\ell, \\
& = A(W_\ell) B(W_\ell, F_{\ell-1}) N_\ell(\F_{\ell-1})P_\ell
\end{align*}
where $A(W_\ell), B(W_\ell, F_{\ell-1})$ are defined in~\eqref{eq:ckn_layer_shortcuts}, has a partial derivative with respect to the inputs given by
\begin{align*}
&\nabla_{\vec(F_{\ell-1})} \vec(\Psi_\ell) \\
=& \Bigg\{\bigg(\left[P_\ell^T N_\ell(F_{\ell-1})\otimes A(W_\ell)\right]\diag\left[\vec\left(k'(W_\ell\E_\ell(F_{\ell-1})N_\ell(\F_{\ell-1})^{-1}\right)\right] \\
&\phantom{\Bigg\{\bigg(}\times \Big[ \left[N_\ell(\F_{\ell-1})^{-1} \otimes W_\ell\right]
 \\
&\phantom{\Bigg\{\bigg(} \phantom{\times\{} 
- \left[\id_{p_\ell}\otimes (W_\ell\E_\ell(F_{\ell-1}))\right]\diag\left[\vec\left(N_\ell(\F_{\ell-1})^{-3}\right)\right]\times\left[\id_{p_\ell} \otimes\E_\ell(F_{\ell-1})^T\right]\Big]\bigg)\\
&+ \left[P_\ell^T \otimes \left(A(W_\ell)B(W_\ell, F_{\ell-1})\right)\right]\diag\left[\vec\left(N_\ell(\F_{\ell-1})^{-1}\right)\right]\times \left[\id_{p_\ell}\otimes \E_\ell(F_{\ell-1})^T\right]\Bigg\} \times \Sigma_\ell.
\end{align*}
where $\Sigma_\ell = 
\sum_{i=1}^{\s_\ell/f_{\ell-1}}\left(\E_{\ell2i}^T\otimes \E_{\ell1i}\right)$.
\end{proposition}
\begin{proof}
For simplicity in the proof we drop the layer index, denoting e.g., $W = W_\ell$ for the weights and $F = F_{\ell-1}$ for the input of this layer.
Denote the Fréchet derivative of a function $f$ at a point $A$ in the direction $H$ by $L_f(A,H)$. For fixed weights $W$, denote by $\tilde \Psi(F) = \Psi(W, F)$ the restricted output function.
Recall that $N(F) = \tilde N(E(F))$ where $\tilde N$ and $E$ are defined respectively in~\eqref{eq:patch_extract},~\eqref{eq:patch_norm}.
Furthermore, define $\Sigma \coloneqq\sum_{i=1}^{\s/f}\left(\E_{\ell2i}^T\otimes \E_{\ell1i}\right) $.
We have by the product rule and chain rule \citep[Theorems 3.3, 3.4]{higham2008} that 
\begin{align}
L_{\vec(\tilde \Psi)}(F,H) = &  \vec\left(A(W)L_{k}
\left(W\E(F)N(\F)^{-1}, \right.\right.\label{eq:2layergrad_term}\\
&\left.\left. W L_{\E}(F,H)N(\F)^{-1}+W\E(\F)L_{\tilde N^{-1}}(\E(F), L_{\E}(F,H))\right)N(F)P\right) \nonumber \\
&+ \vec\left(A(W)B(W, F)L_{\tilde N}(\E(F),L_{\E}(F,H))P \right).
\end{align}
Now consider the first of the two terms in equation~\eqref{eq:2layergrad_term}. By lemmas~\ref{lemma:kernel_ab} and \ref{lemma:matrix_mul} and Corollary~\ref{corollary:inv_norm} we have
\begin{align*}
&\vec\left(A(W)L_{k}
\left(W\E(F)N(\F)^{-1}, \right.\right.\\
&\left.\left. W L_{\E}(F,H)N(\F)^{-1}+W\E(\F)L_{\tilde N^{-1}}(\E(F), L_{\E}(F,H))\right)\tilde N(\E(\F))P\right) \\
=& \vec\left\{A(W)\left\{k'(W\E(\F)N(\F)^{-1})\right.\right.\\
&\odot\left.\left.\left[W\E(H)N(\F)^{-1}-W\E(F)(\tilde N^{-3}(\E(\F))\odot(\E(\F)^T\E(H)))\right]\right\}\tilde N(\E(\F))P\right\}\\
=& \left[(N(F)P)^T\otimes A(W)\right]\diag\left[\vec\left(k'(W\E(F)N(\F)^{-1}\right)\right] \\
&\times \left\{ \left((N(\F)^{-1})^T \otimes W\right)- \left[\id_{p}\otimes (W\E(F))\right]\diag\left(\vec\left[N(\F)^{-3}\right]\right)\left(\id_{p}\otimes \E(F)^T\right)\right\}\Sigma\vec\left(H\right).
\end{align*}
Now consider the second term in \eqref{eq:2layergrad_term}. By Lemma~\ref{lemma:matrix_mul} and Corollary~\ref{corollary:norm} we have
\begin{align*}
&\vec\left(A(W)B(W, F)L_{\tilde N}(\E(F),L_{\E}(F,H))P \right) \\
=& \vec\left(A(W)B(W, F)\left[\tilde N(E(\F))^{-1}\odot \left(\E(\F)^T\E(H) \right)\right]P \right) \\
=& \left[P^T \otimes \left(A(W)B(W, F)\right)\right]\diag\left[\vec(N(\F)^{-1})\right]\left(\id_{p}\otimes \E(F)^T\right)\Sigma\vec(H).
\end{align*}
Therefore,
\begin{align*}
L_{\vec(\tilde \Psi)}(F,H) = & \Big\{\left[(N(F)P)^T\otimes A(W)\right]\diag\left[\vec\left(k'(W\E(F)N(\F)^{-1}\right)\right] \\
&\times \left\{ \left[(N(\F)^{-1})^T \otimes W\right]- \left[\id_{p}\otimes (W\E(F))\right]\diag\left(\vec\left[N(\F)^{-3}\right]\right)\left( \id_{p} \otimes\E(F)^T\right)\right\}\\
&+ \left[P^T \otimes A(W)B(W, F)\right]\diag\left[\vec(N(\F)^{-1})\right]\left(\id_{p}\otimes \E(F)^T\right)\Big\}\Sigma\vec(H).
\end{align*}
\end{proof}

\subsection{Bandwidths derivatives}
Now we compute the derivatives with respect to the bandwidths when using Radial Basis Function (RBF) kernels on the sphere. We denote the RBF kernel for unit norm vectors $x, x'$ such that $\|x\|_2 = \|x'\|_2 =1$ by,
\[
k_\sigma(x, x') = \exp\left(-\frac{\|x-x'\|_2^2}{2\sigma^2}\right) = \exp\left(-\frac{1-\langle x, x'\rangle}{\sigma^2}\right),
\]
where $\sigma$ is its bandwidth.
The output of the convolutional kernel layer incorporates the bandwidth as
\begin{align}\label{eq:ckn_layer_RBF}
\Psi_\ell(F_{\ell-1}, W_\ell, \sigma_\ell) \coloneqq\left(k_{\sigma_{\ell}}\left(W_\ell W_\ell^T\right)+\epsilon \id_{f_\ell}\right)^{-1/2}k_{\sigma_{\ell}}\left(W_\ell\E_\ell(F_{\ell-1})N_\ell(\F_{\ell-1})^{-1}\right)N_\ell(\F_{\ell-1})P_\ell,
\end{align}
where $k_{\sigma_\ell}$ is understood to be applied element-wise.
First note that its gradient back-propagates through the network as in Proposition~\ref{prop:loss_grad}. This is stated in the following proposition.

\begin{proposition}	\label{prop:ckn_loss_grad_bw}
		Let  $\mcl(y, \langle \W_{L+1}, \F_L\rangle)$ be the loss incurred by an image-label sample $(F_0, y)$, where 
	$F_L$ is the output of $L$ layers of the network described by ~\eqref{eq:ckn_layer_RBF} and $W_{L+1}$ parameterizes the linear classifier. 
	Then the 
	gradient of the loss 
	with respect to the bandwidths $\sigma_{\ell}$, $1\leq\ell\leq L$, is given by
	\begin{align*}
	\nabla_{\sigma_\ell} \mcl\left(y, \left\langle \W_{L+1}, \F_L\right\rangle\right) =
	\mcl'\vec(W_{L+1})^T\left[\prod_{\ell '=\ell+1}^{L} \nabla_{\vec\left(F_{\ell'-1}\right)} \vec\left(\Psi_{\ell'}\right)\right]
	\times\nabla_{\sigma_\ell} \vec\left(\Psi_{\ell}\right),
	\end{align*}
	where $\mcl'=\frac{\partial \mcl(\bar y, \hat y)}{\partial \hat y }|_{(\bar y, \hat y) = (y,\left\langle \W_{L+1}, \F_L\right\rangle)} $,  $\nabla_{\vec(F_{\ell'-1})} \vec(\Psi_{\ell'})$ is detailed in Proposition~\ref{prop:2layer_grad_part1}, and $\nabla_{\sigma_{\ell}} \vec\left(\Psi_{\ell}\right)$ is detailed in Proposition~\ref{prop:1layer_grad_bw}.

\end{proposition}

The derivative of the RBF kernel on the sphere with respect to its bandwidth is given by the following lemma. This leads us to the derivative of the network with respect to the bandwidth in the proposition below.
\begin{lemma}
\label{lemma:rbf_sphere_bw_grad}
Let $A\in\mbr^{m\times n}$ and $B\in\mbr^{p\times n}$ and define the function $F:\mbr\to\mbr^{m\times p}$ by $F(\sigma)=\exp\left[-\frac{1}{\sigma^2}(\mbone_{m\times p}-AB^T)\right]$, where $\mbone_{m\times p}$ is an $m\times p$ matrix of ones and $\exp$ is understood to be applied element-wise. Then for a scalar $h\in\mbr$ we have
\begin{equation}
F(\sigma+h)=F(\sigma) + \frac{2}{\sigma^3}F(\sigma)\odot\left(\mbone_{m\times p}-AB^T\right)h + o(h).
\end{equation}
\end{lemma}
\begin{proposition}[Derivative of the network with respect to the bandwidth]
\label{prop:1layer_grad_bw}

The output of the $\ell$\textsuperscript{th} convolutional kernel layer defined in~\eqref{eq:ckn_layer_RBF},
\begin{align*}
\Psi_\ell(F_{\ell-1}, W_\ell, \sigma_\ell) & =\left(k_{{\sigma_\ell}}\left(W_\ell W_\ell^T\right)+\epsilon \id_{f_\ell}\right)^{-1/2}k_{{\sigma_\ell}}\left(W_\ell\E_\ell(F_{\ell-1})N_\ell(\F_{\ell-1})^{-1}\right)N_\ell(\F_{\ell-1})P_\ell, \\
& = A_{\sigma_\ell}(W_\ell) B_{\sigma_\ell}(W_\ell, F_{\ell-1}) N_\ell(\F_{\ell-1})P_\ell,
\end{align*}
where $A_{\sigma_\ell}(W_\ell), B_{\sigma_\ell}(W_\ell, F_{\ell-1})$ are defined as in~\eqref{eq:ckn_layer_shortcuts}, has a partial derivative with respect to the bandwidth of the kernel given by
\begin{align*}
\nabla_{{\sigma_\ell}} \vec(\Psi_\ell) =&  -[(B_{\sigma_\ell}(W_\ell, F_{\ell-1}) N_\ell(\F_{\ell-1})P_\ell)^T\otimes \id_{f}]\left(\id_{f} \otimes A_{\sigma_\ell}(W_\ell) + A_{\sigma_\ell}(W_\ell)\otimes \id_{f}  \right)^{-1}\\
&\phantom{-}\times \left(A_{\sigma_\ell}(W_\ell)^2 \otimes A_{\sigma_\ell}(W_\ell)^2 \right)\vec\left(\frac{2}{{\sigma_\ell}^3}k_{\sigma_\ell}\left(W_\ell W_\ell^T\right)\odot\left(\mbone_{f_\ell\times f_\ell}-W_\ell W_\ell^T\right)\right) \\
&+(P_\ell^TN_\ell(\F_{\ell-1})\otimes A_{\sigma_\ell}(W_\ell)) \\
& \phantom{+} \times
\vec\left(\frac{2}{{\sigma_\ell}^3}B_{\sigma_\ell}(W_\ell, F_{\ell-1})\odot\left(\mbone_{f_\ell\times p_\ell'}-W_\ell\E_{\ell}(F_{\ell-1})N_\ell(\F_{\ell-1})^{-1}\right)\right).
\end{align*}
\end{proposition}
\begin{proof}
For simplicity in the proof we drop the layer index, denoting e.g., $W = W_\ell$ for the weights and $F = F_{\ell-1}$ for the input of this layer.
Denote the Fréchet derivative of a function $f$ at a point $\sigma$ in the direction $h$ by $L_f(\sigma,h)$. For fixed weights $W$ and inputs $F$, denote $\tilde \Psi(\sigma) = \Psi(W, F, \sigma)$ the restricted output function. Similarly denote $\tilde A(\sigma) = A_\sigma(W)$ and $\tilde B(\sigma) = B_{\sigma}(W, F) N(\F)P$.

As in Proposition~\ref{prop:1layer_grad}, observe that
\begin{align*}
\vec(\tilde \Psi (\sigma))
 = (\tilde B(\sigma)^T\otimes \id_{f})\vec(\tilde A(\sigma)).
\end{align*}
By the chain rule  and Lemmas~\ref{lemma:matrix_sqrt}, \ref{lemma:matrix_inv}, and \ref{lemma:rbf_sphere_bw_grad} we have
\begin{align*}
(\tilde B(\sigma)^T\otimes \id_{f})L_{\vec(\tilde A)}(\sigma,h) =&  -(\tilde B(\sigma)^T\otimes \id_{f})\left(\id_{f} \otimes \tilde A(\sigma) + \tilde A(\sigma)\otimes \id_{f}  \right)^{-1}\left(\tilde A(\sigma)^2 \otimes \tilde A(\sigma)^2 \right)\\
&\times\vec\left(\frac{2}{\sigma^3}k_\sigma\left(WW^T\right)\odot\left(\mbone_{f\times f}-WW^T\right)\right)h.
\end{align*}
Additionally, by the chain rule, Lemma~\ref{lemma:kernel_ab}, and Lemma~\ref{lemma:matrix_inv} we have that 
\begin{align*}
&L_{\tilde B^T\otimes \id_{f}}(\sigma,h)\vec(\tilde A(\sigma)) \\
=& \left[P^TN(F)\otimes \tilde A(\sigma)\right]
\vec\left(\frac{2}{\sigma^3}k_\sigma\left(W\E(F)N(F)^{-1}\right)\odot\left(\mbone_{f\times p'}-W\E(F)N(F)^{-1}\right)\right)h.
\end{align*}
Therefore, 
\begin{align*}
\nabla_{\sigma} \vec(\Psi) =&  -[\tilde B(\sigma)^T\otimes \id_{f}]\left(\id_{f} \otimes \tilde A(\sigma) + \tilde A(\sigma)\otimes \id_{f}  \right)^{-1}\left(\tilde A(\sigma)^2 \otimes \tilde A(\sigma)^2 \right)\\
&\phantom{-}\times\vec\left(\frac{2}{\sigma^3}k_\sigma\left(WW^T\right)\odot\left(\mbone_{f\times f}-WW^T\right)\right) \\
&+\left[P^TN(F)\otimes \tilde A(\sigma)\right]
\vec\left(\frac{2}{\sigma^3}k_\sigma\left(W\E(F)N(F)^{-1}\right)\odot\left(\mbone_{f\times p'}-W\E(F)N(F)^{-1}\right)\right).
\end{align*}
\end{proof}

%% file: sections/appendix7-stochastic_gradient_manifolds_constrained.tex
\section{Stochastic gradient optimization on manifolds}
\label{app:sg_manifolds}

In this section we detail the basic stochastic gradient optimization used to train the CKNs, both in terms of the implementation and the theoretical guarantees.

For ease of presentation, denote by $w_\ell = \Vect(\W_\ell) \in \reals^{d_\ell}$ the vectorized weights at the $\ell$\textsuperscript{th} layer with $d_\ell=\f_\ell\times\s_\ell$ parameters and $w_{1:L+1} = (w_1; \ldots; w_{L+1})$ the concatenation of those layers. Furthermore denote by $S^{d_\ell}=\prod_{j=1}^{\f_\ell} \mbs^{\s_\ell}$ the Cartesian product of Euclidean unit spheres in $\reals^{\s_\ell}$ and by $\mathcal{B}_{2,\lambda}$ the Euclidean ball centered at the origin of radius $\lambda$. The problem then reads  
\begin{align}
	\min_{w_{1:L+1}} \qquad & f(w_{1:L+1}) \triangleq \frac{1}{n}\sum_{i=1}^{n} f_i(w_{1:L+1}) \label{eq:optim_pb} \\
	\mbox{subject to}  \qquad & w_{1:L+1} \in \mathcal{C} \triangleq S^{d_0} \times \cdots \times S^{d_{L-1}} \times \mathcal{B}_{2,\lambda} \nonumber
\end{align}
where $f_i(w_{1:L+1})$ is the loss incurred on the $i$\textsuperscript{th} sample and the set of constraints $ \mathcal{C}$ is a product of manifolds and is therefore a manifold itself. 

Optimization analysis on manifolds is characterized by smooth curves, called retractions, parametrized by a point on the manifold and a direction.
In our case, they amount to block-coordinate projections on the sphere. Formally, with a given set of weights $w_{1:L+1}$, given a direction $\delta_{1:L+1} = (\delta_1; \ldots; \delta_{L+1})$ where $\delta_\ell$ denotes the portion corresponding to the $\ell$\textsuperscript{th} layer, the retraction is defined as
\begin{align*}\label{eq:ckn_proj_grad}
& \step(\delta_{1:L+1}; w_{1:L+1}) = v_{1:L+1} \\
\mbox{where} \qquad & v_{\ell}  = \operatorname{Proj}_{S^{d_\ell}}\left[w_\ell + \delta_\ell\right] \qquad \mbox{for $\ell=1,\ldots, L$} \\
& v_{L+1}  = \Proj_{\mathcal{B}_{2,\lambda}}(w_{L+1} + \delta_{L+1})
\end{align*}
where $\operatorname{Proj}_{S^d}$ denotes the orthogonal projection on $S^d$, i.e., a block coordinate normalization.

The stochastic implementation, starting from a given $w_{1:L+1}^{(0)}$, consists at iteration $t$ of sampling a function $f_i$ and  performing the step
\begin{equation}\label{eq:SGO}\tag{SGO}
	w^{(t+1)}_{1:L+1} = \step(- \stepsize_t \nabla f_i(w_{1:L+1}^{(t)}); w_{1:L+1}^{(t)}).
\end{equation}
Recall that by constraining a manifold, the first order information that measures the stationarity of the problem is the projection of the gradients on the tangent space at the current point~\citep{boumal2016}. Formally, for $w \in S$, denote by $\mathcal{T}_{S,w}$\footnote{For $S^{d_\ell}=\prod_{j=1}^{\f_\ell} \mbs^{\s_\ell}$ the tangent space is the product of the tangent spaces and on a sphere $\mbs$, $\mathcal{T}_{\mbs,x} = x+ V(x)^\perp$ where $V(x)  =\{ \alpha x; \alpha \in \reals\}$ is the line generated by $x$.}, the tangent space at $w$ on $S$. The quantity governing the stationarity is then
\begin{align}
& \grad f(w_{1:L+1}) = (g_1 ;\ldots ; g_{L+1}) \label{eq:grad_manifold}\\
\mbox{where} \qquad & g_\ell = \operatorname{Proj}_{\mathcal{T}_{S^{d_\ell}, w_\ell}} \nonumber (\nabla_{w_\ell}f(w_{1:L+1})) \qquad  \mbox{for $\ell \in \{1, \ldots, L\}$}\\
& g_{L+1} = \Proj_{\mathcal{T}_{\mathcal{B}_{2,\lambda}, w_{L+1}}}(\nabla_{w_{L+1}} f(w_{1:L+1})) \nonumber
\end{align} 
where $\mathcal{T}_{\mathcal{B}_{2,\lambda}, w_{L+1}}= \reals^{d_{L+1}}$ if $w_{L+1} \in \mathcal{B}\strut^\mathrm{o}_{2,\lambda}$ or $\mathcal{T}_{\mathcal{B}_{2,\lambda}, w_{L+1}}= \mathcal{T}_{\mathbb{S}_{2,\lambda}^{d_{L+1}}, w_{L+1}}$ if $w_{L+1} \in \mathcal{B}_{2,\lambda} \setminus\mathcal{B}\strut^\mathrm{o}_{2,\lambda} = \mathbb{S}_{2,\lambda}^{d_{L+1}}$.

The ingredients to prove convergence to a stationary point are then similar to the unconstrained case: 
\begin{enumerate}[label=(\roman*) ]
	\item Upper quadratic bounds for the steps, i.e., there exists $M$ such that for all iterates $w_{1:L+1}^{(t)} \in \mathcal{C}$,
	\[
	f(\step(\delta_{1:L+1}; w_{1:L+1}^{(t)})) \leq f(w_{1:L+1}^{(t)}) + \grad f(w_{1:L+1}^{(t)})^\top \delta_{1:L+1} + \frac{M}{2}\|\delta_{1:L+1}\|_2^2,
	\]
	where $\delta_{1:L+1} \in \mathcal{T}_{\mathcal{C}, w_{1:L+1}^{(t)}}$, the tangent space of $\mathcal{C}$ at $w_{1:L+1}^{(t)}$. \label{ass:quad_bound}
	\item Bounded gradients, i.e., there exists $B$  such that for all iterates $w_{1:L+1}^{(t)} \in \mathcal{C}$, \label{ass:grad_bounded}
	\[
	\|\nabla \grad f_i(w_{1:L+1}^{(t)})\| \leq B.
	\]
\end{enumerate}
The convergence result is then as follows.
\begin{theorem}[{\citet[Theorem 9]{hosseini2017}}]	
	\label{thm.sgdcond4}
	Assume conditions \ref{ass:quad_bound} and \ref{ass:grad_bounded} hold. Then, for $\stepsize_t = c/\sqrt{T}$ with $c>0$ and $T$ the maximal number of iterations, the iterates of~\eqref{eq:SGO} satisfy
	\begin{equation*}
	\label{eq:11}
	\frac{1}{T}\sum_{t=1}^{T}\Expect[\|\grad f(w_{1:L+1}^{(t)})\|_2^2] \le \frac{1}{\sqrt{T}}\left( \frac{f(w_{1:L+1}^{(1)})-f^*}{c} + \frac{Mc}{2}B^2\right).
	\end{equation*}
\end{theorem}

Local Lipschitz and smoothness conditions are sufficient to ensure conditions~\ref{ass:quad_bound} and~\ref{ass:grad_bounded} on compact manifolds~\citep{boumal2016}. The convergence result then follows directly.

\begin{proposition_unnumbered}[\ref{prop:sgo_constrained}]
Assume the loss in the constrained training problem~\eqref{eq:training_constrained} and the kernel defining the network~\eqref{eq:ckn_layer} are continuously differentiable.
A projected stochastic gradient descent with stepsize $\stepsize_t = c/\sqrt{T}$ where $c>0$ and $T$ is the maximal number of iterations, finds an $O(1/\sqrt{T})$-stationary point.
\end{proposition_unnumbered}

%% file: sections/appendix8-ultimate_layer_reversal.tex
\section{Ultimate Layer Reversal}
\label{app:ulr}
In this section we provide additional details related to the Ultimate Layer Reversal method.

\subsection{Simplified objective}
We recall how the simplified objective and the original one are generally related through the following proposition.

\begin{proposition_unnumbered}[\ref{prop:ulr_grad}]
Assume that $f(W,V)$ is twice differentiable and that for any $W$, the partial functions $V \rightarrow f(W,V)$ are strongly convex. Then the simplified objective $\simpobj(W) = \min_V f(W, V) $ is differentiable and satisfies
\[
\|\nabla \simpobj(W)\|_2 = \|\nabla f(W, V^*)\|_2,
\]
where $V^* = \argmin_V f(W, V) $. 
\end{proposition_unnumbered}
\begin{proof}
	The pairs $(W, V^*)$ such that $V^* \in \argmin_V f(W, V) $ are solutions of the first order optimality condition $\nabla_V f(W, V) = 0$. By the implicit function theorem, using that $\nabla_V^2 f(W, V)$ is invertible for any pair $(W,V)$ by strong convexity of $V \rightarrow f(W,V)$, the mapping $V^*(W) = \argmin_V f(W, V)$ is differentiable. The simplified objective then reads $\simpobj(W) = f(W, V^*(W))$. It is differentiable as a composition of differentiable functions and satisfies $\nabla \simpobj(W) = \nabla_W f(W, V^*(W))$.
	On the other hand, $\nabla f(W, V^*) = \nabla_W f(W, V^*)$ for any $V^* = \argmin_V f(W, V)$ s.t. $\nabla_V f(W, V^*) = 0$ and therefore 
	$
	\|\nabla \simpobj(W)\|_2 = \|\nabla f(W, V^*)\|_2.
	$
\end{proof}

\subsection{Least squares loss}
We now detail the computations of the ultimate layer reversal for the case of the square loss, for which the minimization can be performed analytically. 
Denote by $V \in \reals^{d \times K}, c \in \reals^K$ the affine classifier in the ultimate layer and $W = W_{1:L}$ the inner weights with $d =d_L$ the dimension of the penultimate layer. The objective then reads
\begin{align*}
\min_{V, c, W} \frac{1}{n}\sum_{i=1}^n(y_i - F(W;x_i)^\top V - c)^2 + \lambda \|V\|_F^2,
\end{align*}
which can be compactly written as
\begin{align}\label{eq:least_squares}
\min_{W, V, c} f(W,V,c) \triangleq \left\{\frac{1}{n}\|Y - F(W) V - \ones_n c^\top\|_F^2 + \lambda \|V\|_F^2 \right\},
\end{align}
where $Y \in \reals^{n \times K}$ is the matrix of labels,
$F(W) \in \reals^{n \times d}$ is the output of the inner layers applied to the inputs $x_1, \ldots, x_n$, and $\lambda>0$ is a regularization parameter.

The derivation of the Ultimate Layer Reversal is then provided by the following proposition.
\begin{proposition}\label{prop:simp_least_square}
	The simplified objective $ \tilde f(W) = \min_{V,c} f(W,V,c)$ of $f$ in~\eqref{eq:least_squares} reads 
	\begin{align*}
	\tilde f(W) & = \frac{1}{n} \|\Pi Y \|_F^2 - \frac{1}{n}\Tr\Big(Y^\top \Pi F(W)\big( F(W)^\top\Pi F(W) + n\lambda  \id_{d}\big)^{-1} F(W)^\top \Pi Y\Big) \\ 
	& = \frac{1}{n}\Tr\Big( Y^\top\Pi\big(\id_{n} + (n\lambda)^{-1}\Pi F(W) F(W)^\top \Pi\big)^{-1} \Pi Y\Big),
	\end{align*}
	where $\Pi = I - \frac{1}{n} \mbone \mbone^\top $.
\end{proposition}
\begin{proof}
For fixed $W$, the optimal classifier parameters $V^*(W),c^*(W)$ can be found analytically. Minimization in $c$ amounts to centering the labels and the inputs to the ultimate layer, i.e., 
\begin{equation*}\label{eq:bias_min}
c^*(W) =  \frac{1}{n}(Y-F(W)V^*(W) )^T\mbone,
\end{equation*}
where $\mbone$ is a vector of ones.
Define the centering matrix $\Pi = \id - \frac{1}{n} \mbone \mbone^\top $, an orthogonal projector. Then minimization in $V$ gives 
\begin{equation*}\label{eq:weights_min}
V^*(W) = \big( F(W)^\top\Pi F(W) + n\lambda  \id_{d}\big)^{-1} F(W)^\top \Pi Y
\end{equation*}
and the resulting objective is 
\begin{align*}
\tilde f (W) & = \frac{1}{n} \|\Pi Y \|_F^2 - \frac{1}{n}\Tr\Big(Y^\top \Pi F(W)\big( F(W)^\top\Pi F(W) + n\lambda  \id_{d}\big)^{-1} F(W)^\top \Pi Y\Big) \\ 
& = \frac{1}{n}\Tr\Big( Y^\top\Pi\big(\id_{n} + (n\lambda)^{-1}\Pi F(W) F(W)^\top \Pi\big)^{-1} \Pi Y\Big),
\end{align*}
where the second line is obtained from the Sherman–Morrison–Woodbury formula. 
\end{proof}

%% file: sections/appendix10-additional_results_training.tex
 \section{Training methods details and results}
 \label{app:training_details_results}
 
 In this appendix we report additional experimental details and results. First, Table~\ref{tab:batch_sizes} provides the batch size we use for each ConvNet and CKN experiment. The batch size is the largest size that fits on the GPU. This batch size declines as the size of the architecture increases since larger architectures consume more memory.
 
  \begin{table}[t]
\centering
 \caption{\label{tab:batch_sizes} Batch sizes used for the LeNet-1, LeNet-5, and All-CNN-C CKNs and ConvNets}
\begin{tabular}{c|rrrrr}
 & \multicolumn{5}{c}{Number of filters/layer} \\
Architecture & 8 & 16 & 32 & 64 & 128 \\
\hline
LeNet-1 & 8192& 4096 & 2048 & 1024 & 512\\
LeNet-5 & 8192 & 4096 & 2048 & 1024 &512  \\
All-CNN-C &  2048 & 2048 &1024 & 512 & 256  \\
\bottomrule
\end{tabular}
\end{table}

 Next, we report the results of using our proposed Ultimate Layer Reversal method in terms of accuracy vs. time. We do so for the LeNet-5 CKN with 8 and 128 filters/layer and for the All-CNN-C CKN on CIFAR-10 with 8 and 128 filters/layer. We can see from the plots that our new Ultimate Layer Reversal method, URL-SGO, outperforms stochastic gradient optimization, SGO.
 
 \begin{figure*}[t!]
\centering
\begin{subfigure}{.48\textwidth}
  \centering
\includegraphics[width=1.05\linewidth]{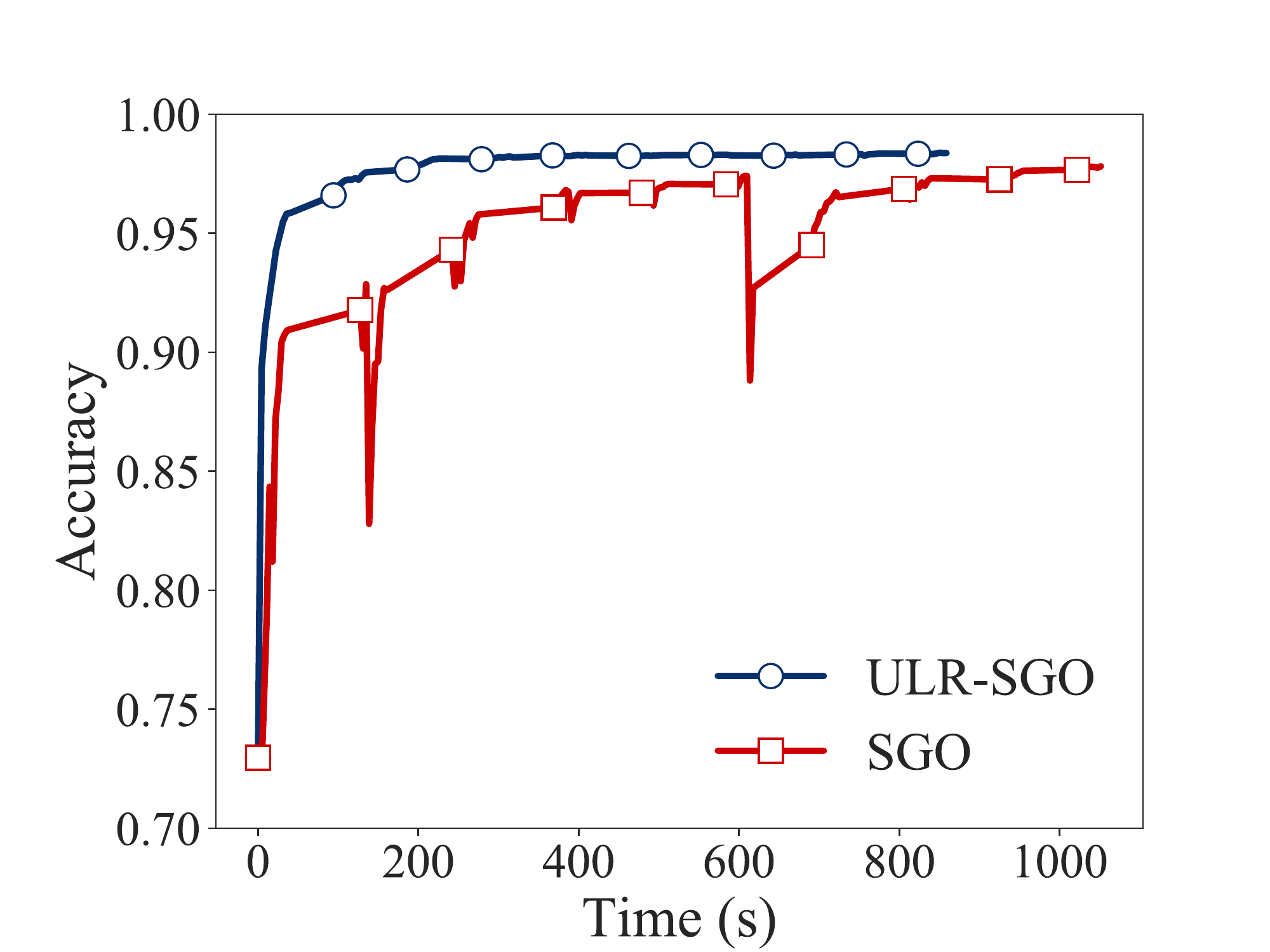}
 \caption{\label{fig:lenet5_8_ulr_sgo_time} LeNet-5 CKN on MNIST with 8 filters/layer}
\end{subfigure}\hfill
\begin{subfigure}{.48\textwidth}
  \centering
\includegraphics[width=1.05\linewidth]{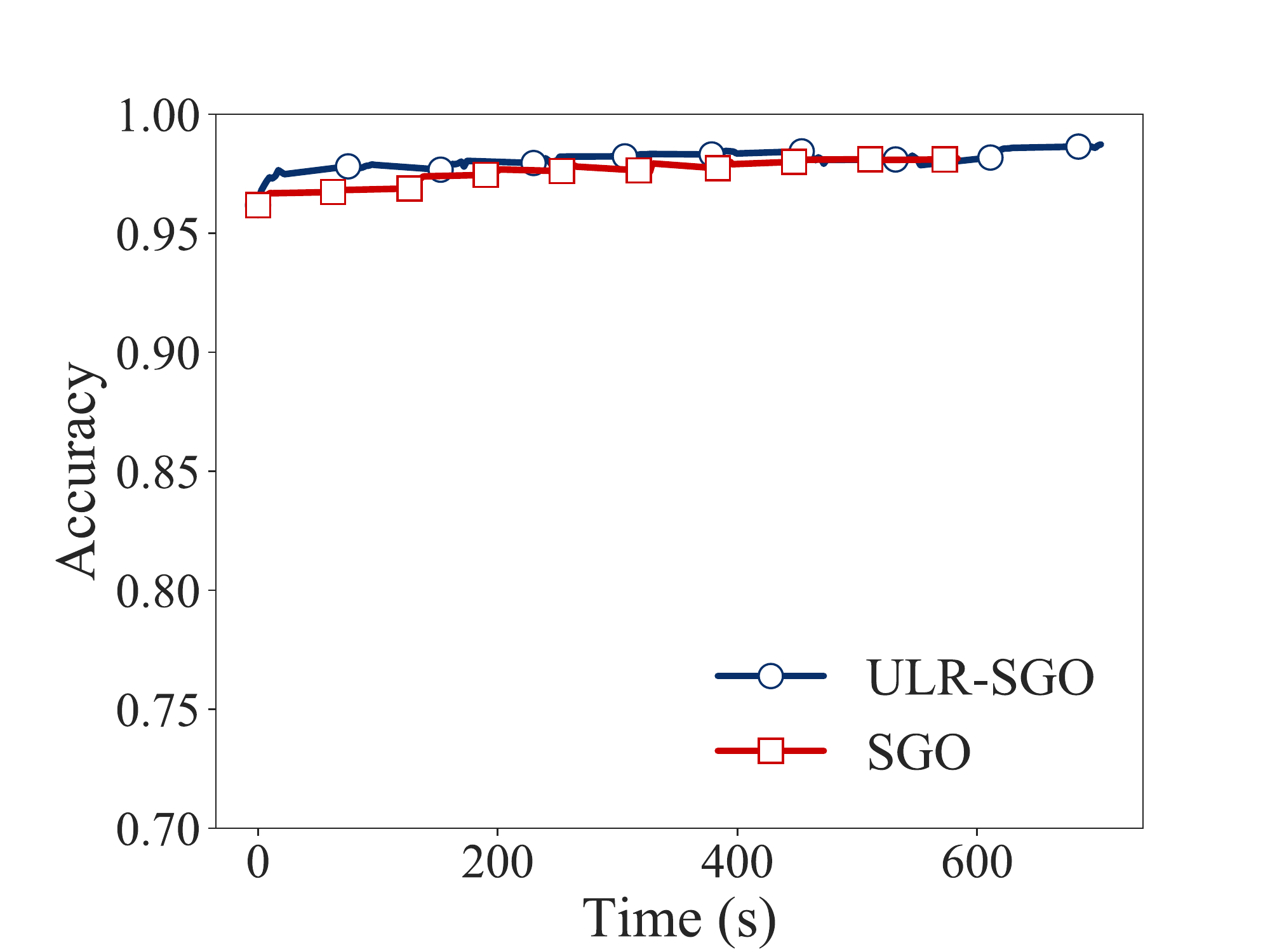}
 \caption{\label{fig:lenet5_128_ulr_sgo_time} LeNet-5 CKN on MNIST with 128 filters/layer}
\end{subfigure}\hfill
\begin{subfigure}{.48\textwidth}
  \centering
\includegraphics[width=1.05\linewidth]{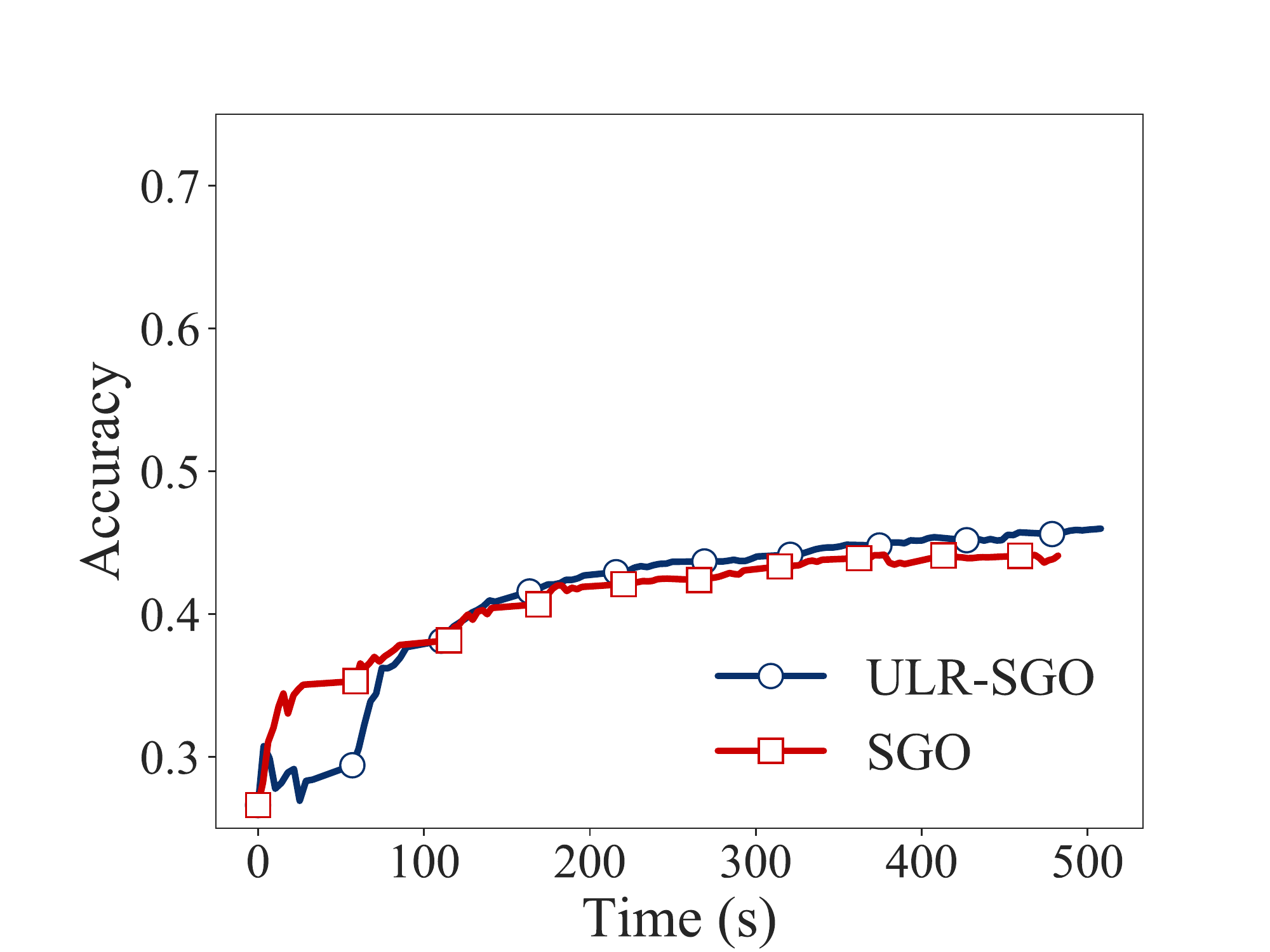}
 \caption{\label{fig:lenet5_8_ulr_sgo_time} All-CNN-C CKN on CIFAR-10 with 8 filters/layer}
\end{subfigure}\hfill
\begin{subfigure}{.48\textwidth}
  \centering
	\includegraphics[width=1.05\linewidth]{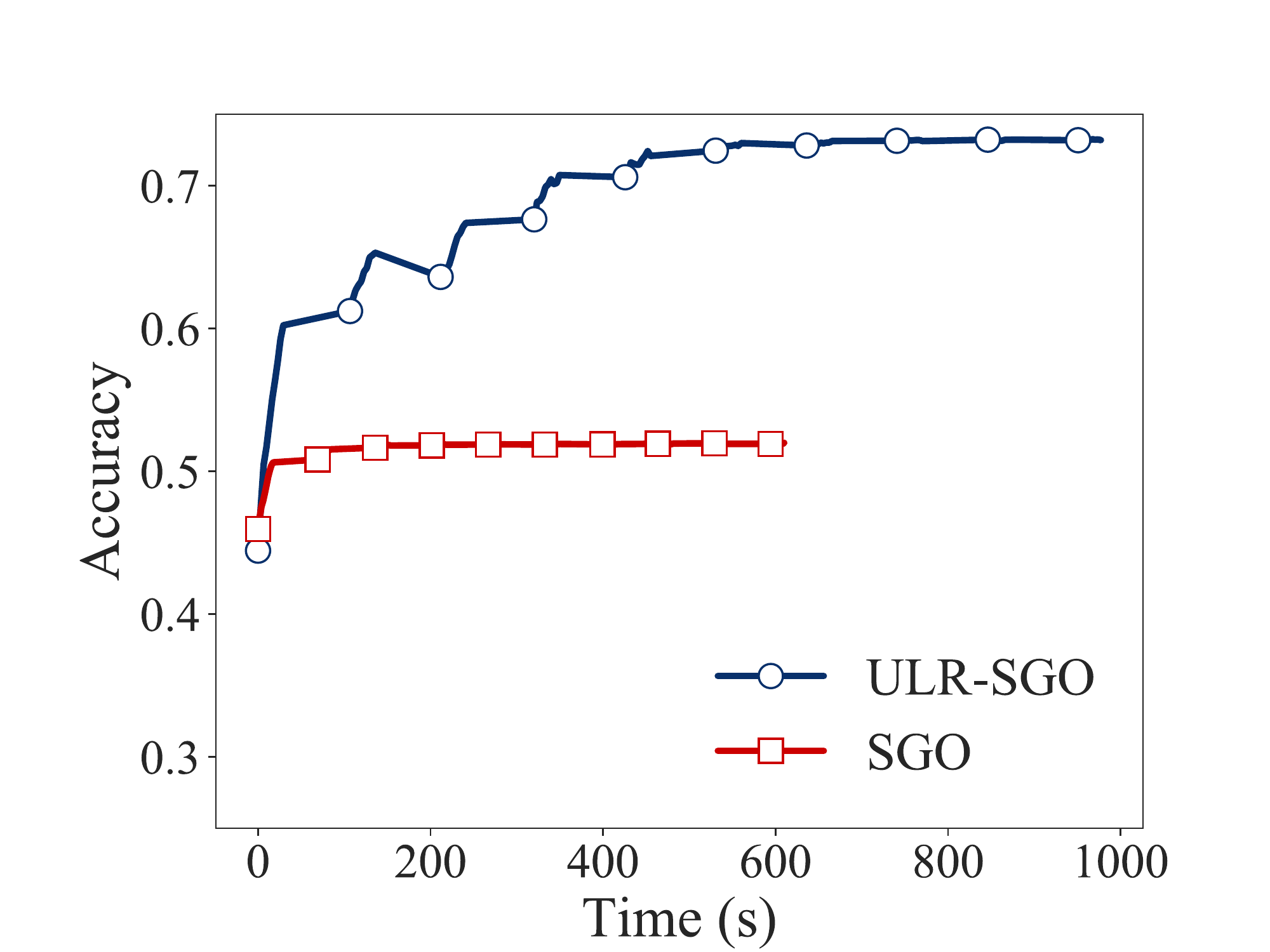}
 \caption{\label{fig:lenet5_128_ulr_sgo_time} All-CNN-C CKN on CIFAR-10 with 128 filters/layer}
\end{subfigure}\hfill
\caption{\label{fig:sgo_ulr_time} Performance of CKNs when using stochastic gradient optimization (SGO) vs. our Ultimate Layer Reversal method (ULR-SGO) in terms of accuracy vs. time.}
\end{figure*}

Finally, Figure~\ref{fig:acc_vs_nfilt_unsup} shows the performance of the unsupervised initialization of the CKNs relative to the trained CKNs and ConvNets. We can see that the unsupervised performance improves as the number of filters per layer increases, as we would expect. For the LeNets the performance varies between 67\% and 98\% of the performance of the supervised CKN. For All-CNN-C the results are more modest: On All-CNN-C the unsupervised performance is 44-59\% of that of the supervised CKN. This suggests that for more complex tasks it is more difficult to obtain a good unsupervised performance.

 \begin{figure*} %
\centering
\begin{subfigure}{.33\textwidth}
  \centering
\includegraphics[width=1.05\linewidth]{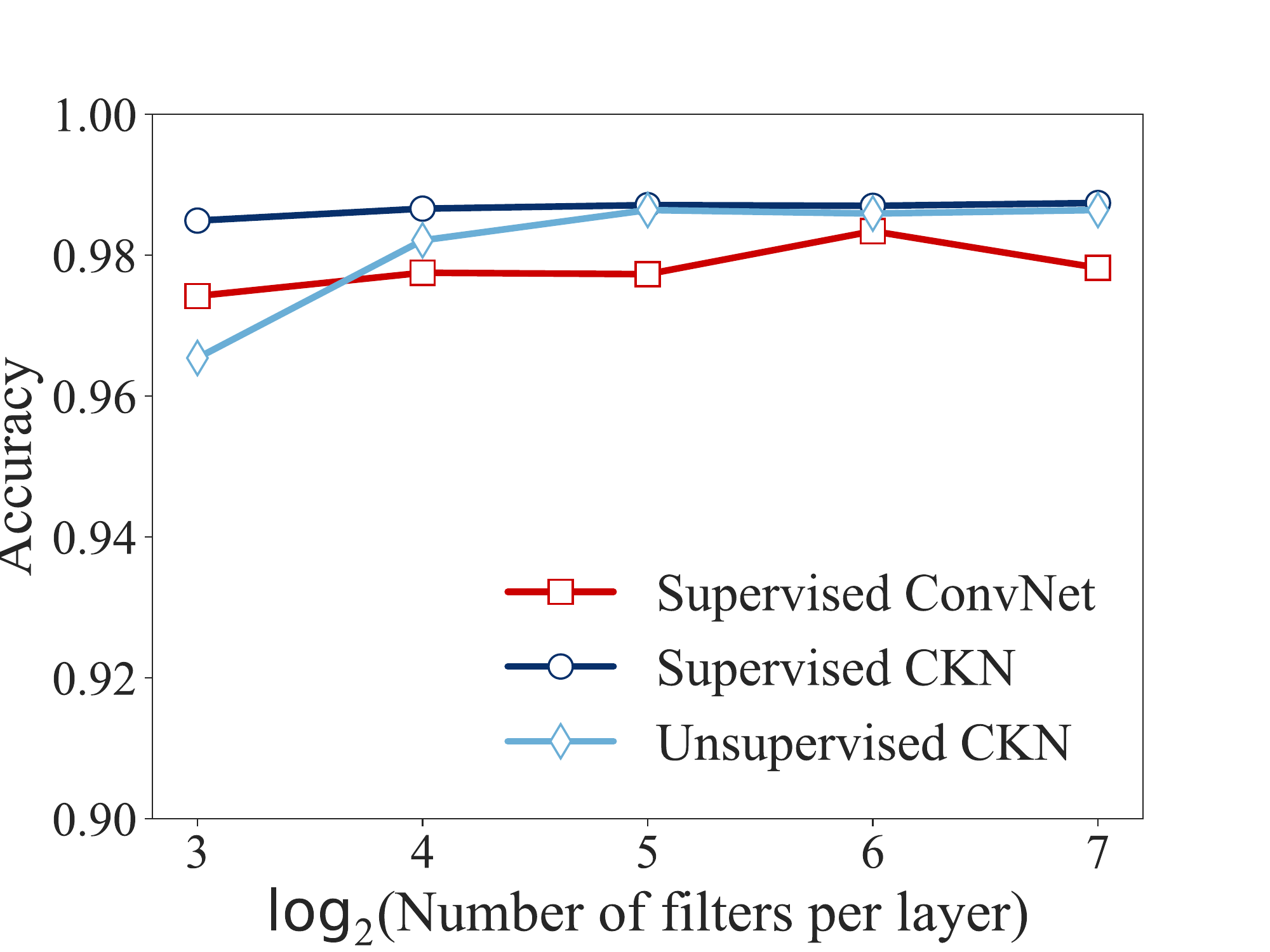}
 \caption{\label{fig:lenet5_unsup}LeNet-1 on MNIST}
\end{subfigure}\hfill
\begin{subfigure}{.33\textwidth}
  \centering
\includegraphics[width=1.05\linewidth]{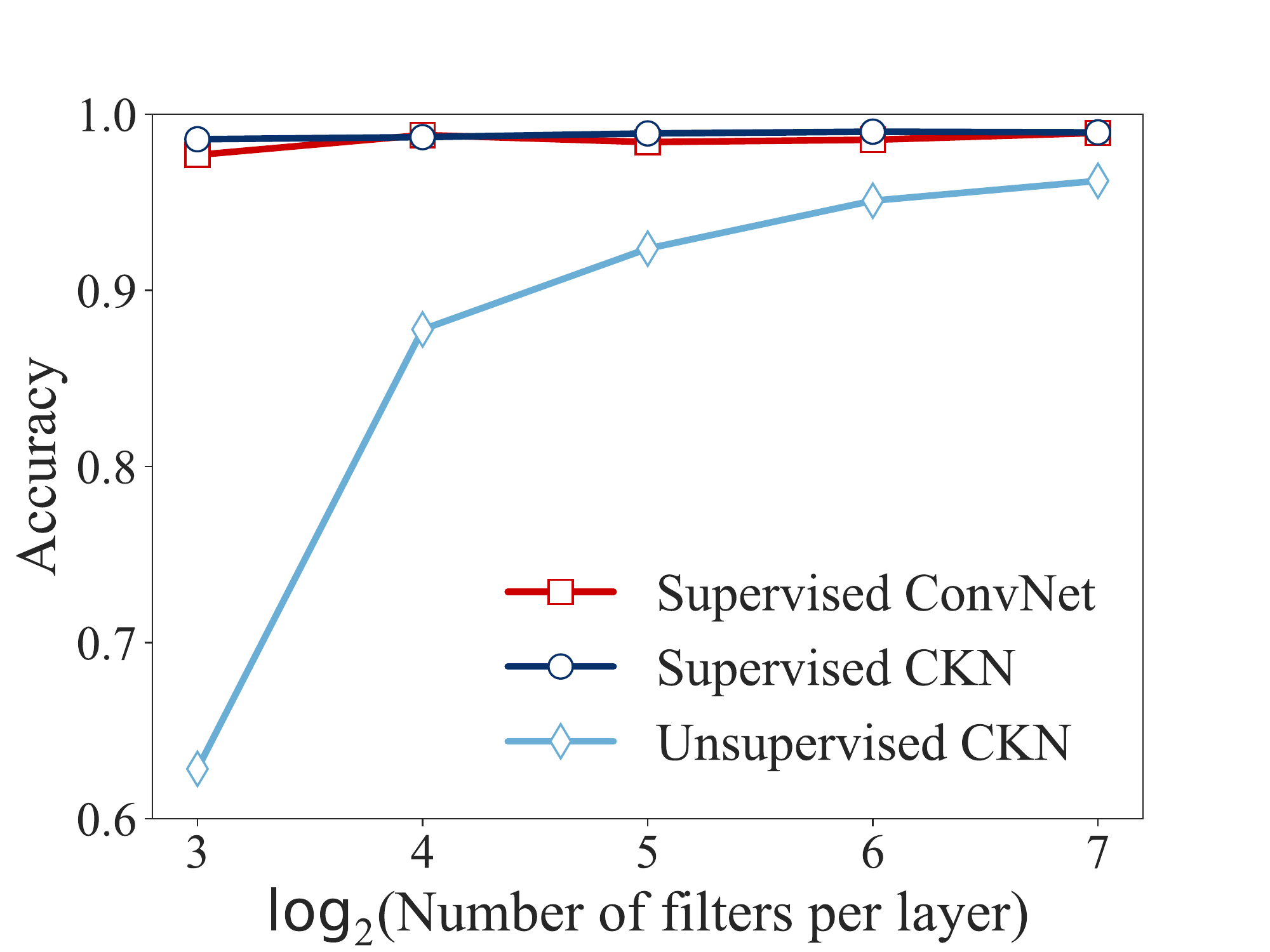}
 \caption{\label{fig:lenet5_unsup}LeNet-5 on MNIST}
\end{subfigure}\hfill
\begin{subfigure}{.33\textwidth}
  \centering
\includegraphics[width=1.05\linewidth]{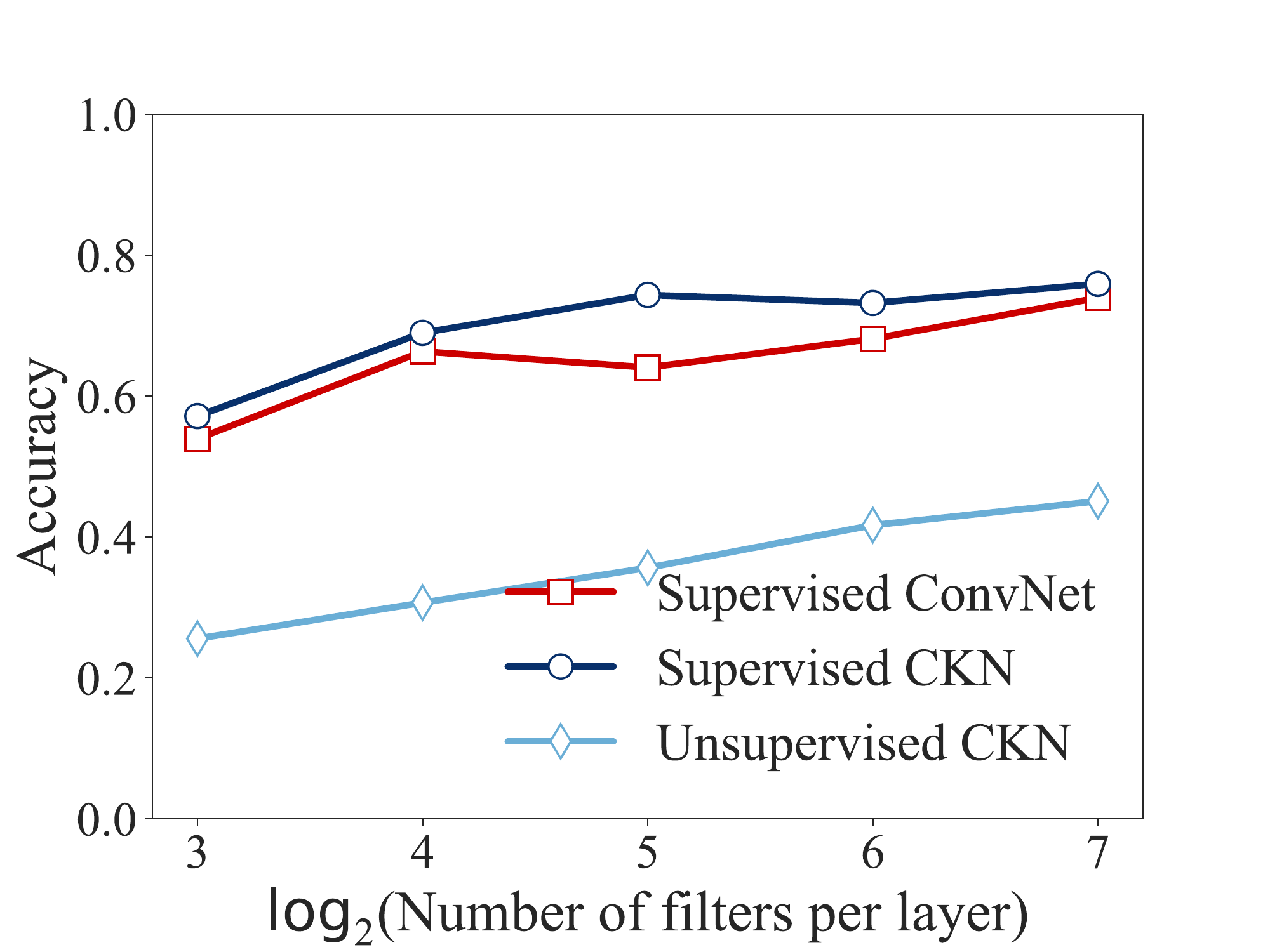}
 \caption{\label{fig:allcnn_unsup} All-CNN-C on CIFAR-10}
\end{subfigure}\hfill
\caption{\label{fig:acc_vs_nfilt_unsup} Performance of unsupervised and supervised CKNs and their ConvNet counterparts when varying the number of filters per layer. Note that the y-axes begin at different values.}
\end{figure*}